\documentclass[letterpaper]{article} 
\usepackage[preprint]{aaai2027}  
\usepackage[hyphens]{url}  
\usepackage{graphicx} 
\usepackage{natbib}  
\usepackage{caption} 
\usepackage{algorithm}
\usepackage{algorithmic}

\usepackage{newfloat}
\usepackage{listings}
\DeclareCaptionStyle{ruled}{labelfont=normalfont,labelsep=colon,strut=off} 
\floatstyle{ruled}
\newfloat{listing}{tb}{lst}{}
\floatname{listing}{Listing}

\usepackage{booktabs}

\usepackage{xr}

\makeatletter
\newcommand{\externaldocumentnocite}[1]{%
    \begingroup
    \let\bibcite\@gobbletwo
    \externaldocument{#1}%
    \endgroup
}
\makeatother

\ifdefined\APPENDIXONLY
    \externaldocumentnocite{main}
\else
    \externaldocumentnocite{supplement}
\fi

\usepackage{amsmath,amssymb,amsthm}

\newcommand{\RR}{\mathbb{R}}
\newcommand{\SSym}{\mathbb{S}}
\newcommand{\Aop}{\mathcal{A}}
\newcommand{\Mcal}{\mathcal{M}}
\newcommand{\Tcal}{\mathcal{T}}
\newcommand{\smin}{\sigma_{\min}}
\newcommand{\Rcal}{\mathcal{R}}
\newcommand{\ip}[2]{\left\langle #1,#2\right\rangle}
\newcommand{\norm}[1]{\left\lVert #1\right\rVert}
\newcommand{\abs}[1]{\left\lvert #1\right\rvert}
\newcommand{\op}{\mathrm{op}}

\DeclareMathOperator{\rank}{rank}
\DeclareMathOperator{\tr}{tr}
\DeclareMathOperator{\diag}{diag}
\DeclareMathOperator{\grad}{grad}
\DeclareMathOperator{\Hess}{Hess}
\DeclareMathOperator{\Proj}{Proj}

\theoremstyle{definition}
\newtheorem{definition}{Definition}
\newtheorem{assumption}{Assumption}

\theoremstyle{plain}
\newtheorem{theorem}{Theorem}
\newtheorem{proposition}{Proposition}
\newtheorem{lemma}{Lemma}
\newtheorem{corollary}{Corollary}

\theoremstyle{remark}
\newtheorem{remark}{Remark}

\title{Over-parameterization Is Neither Robust nor Fragile in Matrix Sensing: \\A Geometric Theory of Solution Robustness}

\author{
    Runze Shi\textsuperscript{\rm},
    Ziye Ma\textsuperscript{\rm}\corresponding
}
\affiliations{
    \textsuperscript{\rm}City University of Hong Kong\\
    runzeshi3-c@my.cityu.edu.hk, ziyema@cityu.edu.hk
}

\begin{document}

\maketitle

\ifdefined\APPENDIXONLY

\appendix
\setcounter{secnumdepth}{2}
\input{appendix}

\bibliography{aaai2027}

\else

\begin{abstract}
Over-parameterization is ubiquitous in modern machine learning, yet its effect on robustness remains poorly understood. We study this question in matrix sensing, a canonical nonconvex problem that can also be viewed, under an appropriate parameterization, as training a two-layer quadratic neural network. We develop a rank-aware robustness theory for the Burer–Monteiro factorization when the search rank \(k\) may exceed the target rank \(r_{\star}\), thereby capturing over-parameterization. Our analysis shows that robustness depends jointly on the degree of over-parameterization and the local geometry of the attained second-order stationary point: increasing \(k\) generally worsens sensitivity, but stationary points with the same search rank can exhibit substantially different robustness. Motivated by this observation, we propose a trajectory-aware algorithm that uses spectral balancing and randomized sketching to steer optimization toward less sensitive solutions. In matrix-sensing experiments, the method reduces paired sensitivity by \(18\)–\(29\%\) for isotropic operators and \(87\)–\(89\%\) for weak-tail operators, without degrading reconstruction accuracy. These results demonstrate that parameter count alone does not determine robustness and identify solution geometry and the optimization trajectory as important levers for improving it. 

\end{abstract}


\section{Introduction}

Matrix sensing seeks to recover an unknown low-rank matrix from a collection
of linear measurements, typically corrupted by noise. It is a canonical
instance of the broader class of low-rank inverse problems, which also
includes matrix completion \citep{candes2009exact}, robust principal component
analysis and low-rank-plus-sparse decomposition \citep{candes2011robust},
lifted formulations of phase retrieval and blind deconvolution
\citep{candes2013phaselift, ahmed2014blind}, and low-rank quantum state tomography
\citep{gross2010quantum}.
Matrix sensing provides a canonical
setting for studying nonconvex low-rank estimation: the observation model is
simple enough to permit precise analysis, while the rank constraint retains
the key non-convex landscape that proves be mathematical challenging in large-scale machine learning
\citep{recht2010guaranteed,bhojanapalli2016global,ge2017nospurious}. Specifically, we consider positive-semidefinite matrix sensing,
\begin{equation}
\label{eq:intro-objective}
\begin{aligned}
    f_w(X)
    &=
    \tfrac{1}{2}
    \norm{\Aop(XX^\top-M_\star)-w}_2^2,
    \ \ X\in\RR^{n\times k}, \\
    y_w &= \Aop(M_\star)+w
\end{aligned}
\end{equation}
where \(M_\star\succeq0\) has rank \(1\le r_\star\le k\).
For rank-one positive-semidefinite measurements \(A_i=a_i a_i^\top\), as in
semidefinite liftings of phase retrieval \citep{candes2013phaselift}, the
prediction satisfies
\(\ip{A_i}{XX^\top}=\sum_{j=1}^k(a_i^\top x_j)^2\).
Thus, this special case is equivalent to training a width-\(k\) two-layer
quadratic network with fixed output weights \cite{li2018algorithmic}. In the positive-semidefinite setting considered in \eqref{eq:intro-objective}, the Burer--Monteiro
factorization writes \(M=XX^\top\), replacing the full matrix and eliminating the rank constraint 
\citep{burer2003nonlinear}.  The factor width \(k\) determines the search rank
and may exceed the unknown target rank \(r_\star\), indicating a mathematically elegant but representative form of over-parametrization.

Over-parameterization is useful when the target rank is unknown, since choosing
\(k < r_\star\) cannot recover the true solution, and practitioners normally make $k$ much larger than the conjectured $r_\star$ to achieve a "safe-margin". When the sensing matrices are i.i.d. sampled from Gaussian distribution and enough measurements are made, it is generally known that
matrix-sensing problems have no spurious local minima and exhibit
directions of escape at saddle points
\citep{bhojanapalli2016global,ge2017nospurious,zhang2019sharp}. However, the requirements on the quantity and geometric uniformity of the sensing data for which such benign landscape exists is quite stringent. Critically, a recent line of work \cite{zhang2025improved,ma2023noisyoverparam,stoegerSoltanolkotabi2021smallInit} has shown that simply increasing search rank $k$ can loosen such requirements. Thus, in terms of optimization, increase in parametrization can serve as useful degrees of freedom rather than merely
redundant parameters.

However, the impact of over-parametrization on robustness is a much more complicated one. Existing theory first asks whether all local minima remain accurate after the measurements are corrupted. Under Gaussian noise and a \(2r\)-RIP constant below \(1/10\), \citet{bhojanapalli2016global} show that every local minimum lies within \(20\sigma_w\sqrt{\log(n)/m}\) of the ground truth. For general smooth objectives in the exactly parameterized regime, \citet{ma2023noisy} show that a clean RIP constant \(\delta<1/3\) places every noisy local minimum in a recovery ball whose radius is linear in the noise magnitude; for squared-loss matrix sensing, this radius scales as \(O(\norm{\Aop}\norm{w}_2/(1-3\delta))\). Their arbitrary-RIP result instead excludes local minima from a region around the truth. \citet{ma2023noisyoverparam} extend this class-wide analysis to noisy over-parameterized matrix sensing: under suitable RIP conditions, they bound the maximum recovery error of an arbitrary local minimum, while under weaker RIP they establish a close-or-far alternative. \emph{These results therefore certify that robustness can be somehow maintained or expected even in the over-parametrized regime.}

A second line of results asks whether optimization can reach this certified set. The strict-saddle guarantees in \citet{ma2023noisy,ma2023noisyoverparam} show that points outside the recovery region must have either a large gradient or a direction of negative curvature, allowing perturbed gradient descent to find an approximate local minimum in polynomial time. A third approach changes the loss and studies a particular optimization trajectory. For an \(\ell_1\)-loss under Sign-RIP, \citet{ma2023global} show that subgradient descent from a small initialization can recover the truth despite rank overestimation and Gaussian or gross corruptions, even when some global minimizers overfit the noise. This is an algorithm-specific robustness guarantee for a different, robust loss.

A natural next question remains unanswered: is robustness the same for all attained solutions, and if not, how can optimization be steered toward the more robust ones? Existing recovery bounds assign one worst-case radius to the entire set of local minima, while strict-saddle theory guarantees convergence to some member of that set. They can certify that every solution is sufficiently close to the truth, but they cannot compare the sensitivities of two solutions inside the same recovery region. In particular, they do not pair a noisy critical point \(M_w\) with its clean counterpart \(M_0\), quantify the solution-specific response \(\norm{M_w-M_0}_F/\norm{w}_2\), or explain why solutions obtained with the same sensing operator, noise level, and search rank can respond differently. This distinction is especially important under over-parameterization because this means that the degree of parametrization is not the only factor influencing robustness. We therefore develop a local, quantitative, and rank-aware theory that identifies the spectral and curvature geometry governing the robustness of an attained stationary point, and then uses those solution-level quantities to steer optimization toward more robust solutions.

This point-wise view also allows for a trajectory-level study of how robustness can be influenced during
optimization. Building on small-initialization dynamics and
randomized matrix sketching
\citep{stoegerSoltanolkotabi2021smallInit,tong2021scaledgd,
xu2023preconditioning,halko2011finding,woodruff2014sketching}, we develop a
trajectory-aware method that can actively guide trajectories toward more robust solutions. We evaluate the
procedure on isotropic and weak-tail sensing operators using both
reconstruction accuracy and paired perturbation sensitivity.

\paragraph{Contributions:}
\begin{itemize}

\item \textbf{A geometric theory of robustness.}
We develop a point-wise theory that explains how over-parameterization and the
local geometry of an attained solution jointly determine its sensitivity to
measurement noise. The theory covers both rank-deficient and full-rank
solutions and distinguishes solutions that have the same search rank but
different robustness.

\item \textbf{A robust algorithm for matrix sensing and beyond.}
We translate the theory into a trajectory-aware method that steers
optimization toward less sensitive solutions. Across the sensing operators we
test, the method improves paired robustness without sacrificing reconstruction error. In the appendix, we also demonstrate preliminary evidence that such approach can also help improve robustness in LoRA fine-tuning of large langauge models.


\end{itemize}

\section{Related Work}
\label{sec:related-work}
\paragraph{Burer--Monteiro factorization and fixed-rank geometry.}
The Burer--Monteiro parameterization represents a positive-semidefinite
matrix as \(M=XX^\top\), enforcing an upper rank bound while reducing
the ambient dimension \citep{burer2003nonlinear}.  Rank-deficient
second-order points of smooth factorized semidefinite programs admit
global optimality certificates
\citep{boumal2016non,boumal2020deterministic}.  Complementarily,
fixed-rank manifold optimization provides the tangent-space Hessians
and local perturbation tools needed to analyze rank-preserving motion
\citep{absil2008optimization,vandereycken2013lowrank,
boumal2023introduction}.  We use these two viewpoints for different
rank profiles: rank-deficient points are treated through a matrix-space
KKT condition, whereas full-rank points are analyzed through local
continuation on a fixed-rank manifold.

\paragraph{Robustness of Matrix Sensing.}
We first talk about studies of robustness for the classic matrix sensing setting~\eqref{eq:intro-objective} with $\ell_2$ loss. Most works in this regime only guarantees global recoverability under noise, but neglects discussion of robustness of non-global solutions. \citet{bhojanapalli2016global} shows that with Gaussian noise and a restrictive RIP constant threshold of $1/10$, every local minima lies within \(20\sigma_w\sqrt{\log(n)/m}\) of the ground truth, which is considered to be the one of the most classic results in this field. Later on, several works have iterated upon this idea, and proved that the RIP constant can be as large as $1/2$ for robust global recovery to still hold \cite{ma2023noisy,li2020nonconvex,zhu2021global}. Specifically, recent works have also extended this analysis to the over-parameterized regime, and has preliminary shown that global solutions can still be robust, but the story is more complicated. \citet{ma2023noisyoverparam} has shown that in the over-parameterized setting global solution can still be recovered to certain accuracy under presence of noise, but this accuracy can be highly fragile and depends on eigenvalues of critical points. Another line of work by \citet{zhang2022improved,zhang2025sharp} also corroborates this fact and states that over-parameterization is benign under general loss functions. However, the caveat with this line of work is they require the loss function to satisfy RIP property with noise corruptions, which is quite unlikely.

\paragraph{Improving Robustness of Low-rank Optimization}
A separate line of work considers how to improve robustness of low-rank matrix problems. One of the most commonly considered approach is by adopting a robust loss, like \(\ell_1\) loss. For a known rank, \citet{li2020nonconvex} establishes local exact
recovery and convergence guarantees for non-convex robust matrix
factorization.  With an over-parameterized rank, \citet{ding2021rankoverspecified}
show that a subgradient method with diminishing step sizes still achieves
exact recovery under sparse Gaussian-measurement corruptions.
\citet{ma2023global} further show that small initialization can prevent
the redundant factor directions from fitting the noise, yielding global
trajectory guarantees under both outlier and Gaussian noise models.
Recent work extends this direction to preconditioned asymmetric
factorizations \citep{giampouras2025opsa}.  At the same time,
\citet{ma2025localoptimality} show that, in over-parameterized robust
matrix recovery, the ground-truth factors may be strict saddle points
rather than local minima, emphasizing that algorithmic selection cannot
be explained by local optimality alone. Despite the success of adopting a robust loss in the matrix space, such practice cannot be applied in larger-scale machine learning problems due to its non-differentiability. Therefore, we still need a more careful, and practically relevant analysis of robustness in differentiable regression losses.


\paragraph{Optimization dynamics of Low-rank Problems.}
Global landscape guarantees do not determine which region an optimization algorithm selects. However, there are existing works that we can build upon to explain robustness better. For example, in over-parameterized matrix sensing,
small initialization induces an implicit spectral bias that first
aligns the iterates with the signal subspace and suppresses redundant
directions \citep{stoegerSoltanolkotabi2021smallInit}. Gram
preconditioning accelerates ill-conditioned low-rank estimation
\citep{tong2021scaledgd}, and damped variants extend this mechanism to
over-parameterized noisy matrix sensing \citep{xu2023preconditioning}.
Randomized sketches provide scalable approximations of dominant
subspaces and Rayleigh information
\citep{halko2011finding,woodruff2014sketching}.  We use these tools for
a different purpose: to monitor and modify the local robustness
geometry along the optimization trajectory.  Because exact RIP
certification is generally intractable
\citep{bandeira2013certifying}, our finite-probe RIP quantities are
treated as diagnostics rather than formal global certificates.


\section{Problem Setup and Notation}
\label{sec:problem-setting}

Fix positive integers \(n,m,k\) with \(k\le n\). Let \(\SSym^n\) denote the space of real symmetric matrices equipped with the Frobenius inner product. The sensing operator \(\Aop:\SSym^n\to\RR^m\) is defined by \((\Aop Z)_i=\ip{A_i}{Z}\), where \(A_i\in\SSym^n\), with adjoint \(\Aop^\ast z=\sum_{i=1}^m z_iA_i\). We observe \(y_w=\Aop(M_\star)+w\), where \(w\in\RR^m\) and \(M_\star\succeq0\) has rank \(1\le r_\star\le k\). The theory treats \(w\) deterministically; distributional noise assumptions are used only in the experiments. The case \(k>r_\star\) is over-parameterized.

Define the smooth matrix-space loss on \(\SSym^n\) by
\begin{equation}
\label{eq:Fw}
    F_w(M)=\frac12\norm{\Aop(M-M_\star)-w}_2^2 .
\end{equation}
The nonconvex factorized problem studied in this paper is
\begin{equation}
\label{eq:factorized-problem}
    \min_{X\in\RR^{n\times k}} f_w(X),
    \qquad f_w(X)=F_w(XX^\top).
\end{equation}
We write \(H=\Aop^\ast\Aop\), \(N_w=\Aop^\ast w\), and \(G_w(M)=\nabla F_w(M)=H(M-M_\star)-N_w\), where \(G_w(M)\) is the ambient Frobenius gradient. Then assume \(r_\star\le k\) and \(\Aop\) satisfies
\((\delta,2k)\)-RIP for some \(\delta\in[0,1)\), i.e.,
\(\left|\norm{\Aop(Z)}_2^2-\norm{Z}_F^2\right|
\le\delta\norm{Z}_F^2\) for all \(Z\in\SSym^n\) with
\(\rank(Z)\le2k\). We set \(s_{\rm RIP}=2k\), and, for \(M\succeq0\),
write its positive eigenvalues as
\(\lambda_1(M)\ge\cdots\ge\lambda_{\rank(M)}(M)>0\).

For \(1\le q\le\rank(M)\), let \(P_q(M)\) be a projector onto a
top-\(q\) eigenspace and set \(P_q^\perp=I-P_q\). The tangent space of
the rank-\(q\) PSD manifold at a matrix with range projector \(P\) is
\begin{equation}
\label{eq:tangent-space}
    \Tcal(P)=\{\eta\in\SSym^n:P^\perp\eta P^\perp=0\}.
\end{equation}
Every \(\eta\in\Tcal(P)\) has rank at most \(2q\). Let factor width \(k\) be the search rank; \(q=\rank(M)\) at deficient points, \(q=k\)
along continuations, and \(q<k\) only for diagnostics.

\begin{definition}[Tangent gain and normal penalty]
\label{def:gains-ratios}
For a rank-\(q\) projector \(P\), define
\begin{equation}
\label{eq:tangent-gains}
    \alpha_T(P)=
    \inf_{\eta\in\Tcal(P),\,\norm{\eta}_F=1}
    \norm{\Aop(\eta)}_2^2.
\end{equation}
For \(M\succeq0\) and \(1\le q\le\rank(M)\), define
\begin{equation}
\label{eq:normal-ratios}
    \rho_q^-(M;w)=
    \frac{\norm{[P_q^\perp G_w(M)P_q^\perp]_-}_{\op}}
    {\lambda_q(M)}.
\end{equation}
Here \([B]_-\succeq0\) is the negative spectral part of \(B\).Here \([B]_-\) is the negative spectral part of \(B\).  The exact Hessian penalty is \(\rho_q^-\).
\end{definition}

Under Assumption~\ref{ass:unified-rip}, every rank-\(q\) projector
with \(q\le k\) satisfies
\begin{equation}
\label{eq:tangent-rip}
    1-\delta\le\alpha_T(P)\le1+\delta.
\end{equation}

\section{Stability of Rank-$k$ Critical Points}
\label{sec:full-rank-stability}
As per existing over-parametrized matrix sensing theory, any critical point that is rank-deficient is already globally optimal, which is rare to come by~\cite{boumal2016non}. Therefore, the majority of critical points we encounter will be of full rank (rank-$k$). This section will be dedicated to studying how the robustness of rank-$k$ critical points differ from each other, and in the next section we will extend this discussion to rank-deficient ones. We first show how the normal block of the ambient gradient reduces tangent curvature in the factorized formulation. We then use the resulting curvature margin to study the surrounding space of Riemannian critical points projected on a rank-$k$ manifold.

\paragraph{Tangent curvature in the factorized formulation.}

Let $M=XX^\top$, where $X$ has full column rank, and suppose
$G_w(M)X=0$. For a factor direction $D$, let
$\eta=XD^\top+DX^\top\in\Tcal(P)$, where $P$ is the range projector of $M$.
The Hessian decomposes as
\begin{equation}
    \label{eq:hessian-decomposition-main}
    \nabla^2 f_w(X)[D,D]
    =\norm{\Aop(\eta)}_2^2+2\ip{G_w(M)}{DD^\top}.
\end{equation}
The first term is the sensing curvature along $\eta$. In the second,
$G_w(M)X=0$ removes the range and cross terms, so only
$[P^\perp G_w(M)P^\perp]_-$ reduces curvature. Normalized by the active
spectral floor $\lambda_k(M)$, this loss is $\rho_k^-(M;w)$ from
Definition~\ref{def:gains-ratios}. Hence,
\begin{equation}
    \label{eq:active-margin-main}
    \begin{aligned}
        \nabla^2 f_w(X)[D,D]
        &\ge
        \left(\alpha_T(P)-\rho_k^-(M;w)\right)\norm{\eta}_F^2
        \\
        &\ge
        \left(1-\delta-\rho_k^-(M;w)\right)\norm{\eta}_F^2 .
    \end{aligned}
\end{equation}
The first inequality combines the lower tangent gain $\alpha_T(P)$ with the
most negative Rayleigh quotient of the normal-normal block, normalized by
$\lambda_k(M)$; the second uses the $2k$-RIP lower bound
$\alpha_T(P)\ge1-\delta$. Thus, $\alpha_T(P)-\rho_k^-(M;w)$ lower-bounds
$\nabla^2 f_w(X)[D,D]/\norm{\eta}_F^2$ and is the tangent-curvature margin at
$M$; $1-\delta-\rho_k^-(M;w)$ is its conservative RIP-based lower bound.
Because $[P^\perp G_w(M)P^\perp]_-$ can differ across attained points,
rank-$k$ critical points with the same search rank and sensing operator can
differ in robustness. Details are given in Appendix~\ref{app:s4proofs}.

\paragraph{Riemannian critical points on the rank-$k$ manifold.}
Let
\[
    \Mcal_k=\{M\succeq0:\rank(M)=k\}
\]
be the embedded rank-$k$ PSD manifold. At a full-column-rank critical factor,
horizontal factor directions lift all tangent directions of $\Mcal_k$. For
$\eta\in T_M\Mcal_k$ and a horizontal factor direction $D$ satisfying
$\eta=XD^\top+DX^\top$, \citet[Theorem~1]{luo2025nonconvex} give
\[
    \ip{\eta}{\Hess_{\Mcal_k}F_w(M)[\eta]}
    =\nabla^2 f_w(X)[D,D].
\]
Thus,~\eqref{eq:active-margin-main} also lower-bounds the Riemannian Hessian on
$T_M\Mcal_k$.

For $s\in[0,1]$, write
\[
    F_s(M)=F_{sw}(M),
    \qquad
    G_s(M)=G_{sw}(M).
\]
Let $M_1=M_w\in\Mcal_k$ be a computed rank-$k$ Riemannian critical point of
$F_1=F_w$ satisfying
\begin{equation}
    \label{eq:endpoint-critical}
    \grad_{\Mcal_k}F_w(M_w)=0.
\end{equation}
For the embedded Frobenius metric,
$\grad_{\Mcal_k}F_s(M)=0$ is equivalent to $G_s(M)P=0$.

Define the RIP-based tangent-curvature margin at $M_w$ by
\begin{equation}
    \label{eq:endpoint-margin-main}
    \mu_w=(1-\delta)-\rho_k^-(M_w;w).
\end{equation}
If $\mu_w>0$, the Riemannian Hessian at $M_w$ is positive definite on
$T_{M_w}\Mcal_k$. The implicit-function theorem gives a $C^1$ curve for $s$ in
a neighborhood of $1$, and the small-noise condition in
Appendix~\ref{app:s4proofs} continues it over $[0,1]$ while keeping the
tangent-Hessian margin and rank-$k$ spectral floor uniformly positive:
\[
    \begin{aligned}
        s&\mapsto M_s\in\Mcal_k,\qquad 0\le s\le1,\\
        M_1&=M_w,\qquad \grad_{\Mcal_k}F_s(M_s)=0.
    \end{aligned}
\]
Set $M_0:=M_{s=0}$, the clean rank-$k$ Riemannian critical point paired with
$M_w$ by this curve. The Hessian equality, complete argument, and uniform
bounds are given in Appendix~\ref{app:s4proofs}.

\begin{theorem}[Pointwise rank-stable perturbation bound]
\label{thm:rank-stable-bound}
Suppose Assumption~\ref{ass:unified-rip} and the local stability conditions at
$M_w$ in Assumptions~\ref{ass:endpoint-margin}--\ref{ass:small-noise} hold.
Then the noisy rank-$k$ Riemannian critical point $M_w$ and the corresponding
clean rank-$k$ Riemannian critical point $M_0$ satisfy
\begin{equation}
    \label{eq:rank-stable-rip}
    \norm{M_w-M_0}_F
    \le
    \frac{2\sqrt{1+\delta}}
    {(1-\delta)-\rho_k^-(M_w;w)}
    \norm{w}_2 .
\end{equation}
\end{theorem}

Theorem~\ref{thm:rank-stable-bound} makes the comparison among rank-$k$
critical points explicit: a smaller negative normal component gives a larger
tangent-curvature margin and a tighter noise-linear perturbation bound, whereas
the bound deteriorates as $\rho_k^-(M_w;w)$ approaches $1-\delta$. Thus, the
search rank alone does not determine the robustness of an attained point. The
bound controls the perturbation between the two rank-$k$ Riemannian critical
points $M_w$ and $M_0$ connected by the continuation curve, rather than the
recovery error relative to $M_\star$. In particular, when $r_\star<k$, the
rank-$k$ critical point $M_0$ cannot equal the lower-rank target $M_\star$;
direct recovery guarantees are given in Theorem~\ref{thm:three-cases}(2) and
(3). The associated projector-rotation bound is given in
Appendix~\ref{app:s4proofs}. The result is intrinsic and does not depend on a
particular local chart; complementary coordinate expressions are deferred to
Appendix~\ref{ssec:Supporting-Chart-Diagnostics}.

\section{Rank Deficiency and the Three Robustness Cases}
\label{sec:rank-cases}
We now extend the discussion to both rank-deficient and full-rank
SOSPs in this section. First we explain how unused columns provide a matrix-space optimality
certificate, and then we will show how the slack spectral tail limits the negative spectral part of the normal block at a full-rank point. These two arguments yield
three rank-aware robustness routes centered at the computed noisy point:
direct recovery for rank-deficient SOSPs, critical points stability along a full-rank
branch, and direct recovery relative to the lower-rank target.

\subsection{Rank-Deficient SOSPs}
Rank deficiency makes the factorized second-order condition especially useful.
If \(X\) has an unused column direction, a factor perturbation along that
direction has zero first-order matrix variation. The second-order condition
then gives a direct condition on the ambient matrix gradient. Since the same
argument applies to every matrix direction, it forces

\[
    G_w(XX^\top)\succeq0.
\]

First-order criticality also gives \(G_w(M)X=0\), where \(M=XX^\top\), and
therefore \(\langle G_w(M),M\rangle=0\). For every \(Z\succeq0\), these two
conditions imply

\[
    \langle G_w(M),Z-M\rangle
    =\langle G_w(M),Z\rangle\ge0.
\]

So the rank-deficient SOSP is a global minimizer
in matrix space. In the noiseless case, its objective value is zero and the
unified \(2k\)-RIP gives \(M=M_\star\). With noise, the same optimality
condition and RIP give the noise-linear recovery bound

\[
    \norm{M-M_\star}_F
    \le
    \frac{\norm{w}_2}{\sqrt{1-\delta}}.
\]

Thus, the rank-deficient guarantee is stated directly relative to the target;
no clean rank-\(k\) critical point is needed. This is the PSD matrix-sensing
specialization of the standard rank-deficiency certificate for
Burer--Monteiro factorizations
\citep{burer2003nonlinear,boumal2016non,boumal2020deterministic}. The formal
statements and complete arguments are deferred to Appendix~\ref{s5proofs}.

\subsection{Full-Rank Points and the Spectral Tail}

For a full-rank SOSP, write
\(M=XX^\top=\sum_{i=1}^k\lambda_i u_iu_i^\top\), with
\(\lambda_1\ge\cdots\ge\lambda_k>0\), and
\(P_q=\sum_{i=1}^q u_iu_i^\top\), \(q\le k\). Second-order stationarity and
the unified \(2k\)-RIP give
\begin{align}
    \norm{[P_q^\perp G_w(M)P_q^\perp]_-}_{\op}
    &\le (1+\delta)\lambda_k(M),
    \label{eq:tail-normal-bound}\\
    \rho_q^-(M;w)
    &\le
    (1+\delta)\frac{\lambda_k(M)}{\lambda_q(M)}.
    \label{eq:rho-tail-ratio}
\end{align}
Thus, \(\lambda_k(M)/\lambda_q(M)\) bounds the normalized negative spectral
part of the normal block, but not the full tail energy
\(\sum_{i=q+1}^k\lambda_i(M)^2\).

For target recovery, first-order criticality gives
\(G_w(M_w)=P_w^\perp G_w(M_w)P_w^\perp\), where \(P_w=P_k(M_w)\). Define
\begin{equation}
    \label{eq:absolute-normal-leakage}
    b_w^-
    :=
    \norm{[P_w^\perp G_w(M_w)P_w^\perp]_-}_{\op}
    =
    \lambda_k(M_w)\rho_k^-(M_w;w).
\end{equation}
Another ingredient controls the portion of the lower-rank target that lies
outside \(\operatorname{range}(M_w)\). This target-subspace leakage is controlled by
the current recovery error. Combining that estimate with \(b_w^-\) and the
unified \(2k\)-RIP gives the direct full-rank recovery certificate below. The formal
spectral-tail statement, the target-subspace leakage estimate, and the complete
normal-leakage argument are deferred to Appendix~\ref{s5proofs}.

\subsection{Main Theoretical Result}

Here we present our main theorem, which elaborates three possible robust cases.
Given the discussion about rank-deficient solutions above, we only compare it to \(M_\star\). For a
full-rank critical point under noise, one can either view it as perturbed from a clean and nearby full rank solution, or perturbed from $M_\star$. We believe those three cases offer a rather complete characterization of possible scenarios.

\begin{theorem}[Three robustness certificates]
\label{thm:three-cases}
Let \(X_w\) be an SOSP of \(f_w\), and write \(M_w=X_wX_w^\top\). Under
Assumption~\ref{ass:unified-rip}, the robustness of \(X_w\) can be characterized as:
\begin{enumerate}
    \item \textbf{General Case.}
    Suppose \(\rank(X_w)=k\), Assumption~\ref{ass:small-noise} additionally holds, and the
    RIP-based curvature margin is positive:
    \begin{equation}
        \label{eq:case-full-full-margin}
        \mu_w=(1-\delta)-\rho_k^-(M_w;w)>0.
    \end{equation}
    Let \(s\mapsto M_s\in\mathcal M_k\) be the critical continuation supplied
    by Lemma~\ref{lem:small-noise-continuation}, with \(M_1=M_w\), and let
    \(M_0=M_{s=0}\) be its clean rank-\(k\) critical point. Then
    \begin{equation}
        \label{eq:case-full-full-pointwise}
        \norm{M_w-M_0}_F
        \le
        \frac{2\sqrt{1+\delta}}
        {(1-\delta)-\rho_k^-(M_w;w)}
        \norm{w}_2.
    \end{equation}
    For rank \(q<k\), Lemma~\ref{lem:spectral-tail} in
    Appendix~\ref{s5proofs} gives deterministic bounds for the best
    rank-\(q\) approximation; direct recovery relative to \(M_\star\) is
    handled separately in part~(3).
    
    \item \textbf{Direct Rank-deficient certificate.}
    If \(\rank(X_w)<k\), then \(M_w\) is a global noisy PSD minimizer and
    \begin{equation}
        \label{eq:case-rank-def-target}
        \norm{M_w-M_\star}_F
        \le
        \frac{\norm{w}_2}{\sqrt{1-\delta}}.
    \end{equation}

    \item \textbf{Direct Full-rank Certificate.}
    Suppose \(\rank(X_w)=k\) with \(P_w=P_k(M_w)\), and let \(b_w^-\) be the
    absolute negative normal leakage defined in
    \eqref{eq:absolute-normal-leakage}. Define
    \begin{equation}
        \label{eq:case-mu-leak}
        \mu_{\rm leak}
        =
        1-\delta
        -
        \frac{b_w^-}{\lambda_{r_\star}(M_\star)}.
    \end{equation}
    If \(\mu_{\rm leak}>0\), then
    \begin{equation}
        \label{eq:case-normal-leakage}
        \norm{M_w-M_\star}_F
        \le
        \frac{\sqrt{1+\delta}}{\mu_{\rm leak}}\norm{w}_2.
    \end{equation}
    This bound is derived using a different technical route than the general case, and can be useful if one is only interested in comparing to $M_\star$.
\end{enumerate}
\end{theorem}

\begin{remark}[Scope of the full-rank continuation certificate]
Under standard no-spurious-minimum RIP conditions
\citep{bhojanapalli2016global,ge2017nospurious,zhang2019sharp}, a clean
full-rank SOSP distinct from the lower-rank target \(M_\star\) may not exist. Since those RIP requirements are usually hard to satisfy, the general case offers a more practical bound under weaker data conditioning and also covers non-global solutions.
\end{remark}

\section{Anchor-Free Guidance Toward Robust Full-Rank Critical Points}
\label{sec:methods}
The preceding theory identifies the pointwise margin
\((1-\delta)-\rho_k^-\) as a key factor of
stability. We now turn this observation into an optimization
procedure and ask whether the trajectory can be steered toward a
less sensitive critical point. At a high level, our method acts on
the numerator and denominator of \(\rho_k^-\) during an early
guidance window, gradually removes this guidance, and finally
polishes the iterate on the original objective \(f_w\). Then we will evaluate the polished critical points under an end-to-end protocol, an
initialization-controlled protocol, and comparisons with robust
low-rank recovery baselines.

For \(M_t=X_tX_t^\top\), let \(P_t=P_k(M_t)\) and define
\begin{equation}
\label{eq:method-rho-k}
\begin{aligned}
\rho_{k,t}^-
&=
\frac{b_t^-}{\lambda_k(M_t)},
\qquad
G_t
=
\Aop^\ast\!\bigl(\Aop(M_t)-y_w\bigr),
\\
b_t^-
&=
\max\!\left\{
0,\,
-\lambda_{\min}
\!\left(P_t^\perp G_tP_t^\perp\right)
\right\}.
\end{aligned}
\end{equation}
Here \(b_t^-\) is the absolute negative normal leakage, while
\(\lambda_k(M_t)\) is the smallest nonzero eigenvalue \(\lambda_k(M_t)\). The method is
``anchor-free'' because it uses only the current factor, observed
residual, applications of \(\Aop\) and \(\Aop^\ast\), and random
probes. It does not use \(M_\star\), clean measurements, or a penalty
anchoring the trajectory to its initialization.

\begin{algorithm}[t]
\caption{Decayed Anchor-Free Balanced Guidance}
\label{alg:decayed-balanced-guidance}
\begin{algorithmic}[1]
\REQUIRE Measurements \(y_w\), sensing operator \(\Aop\), search rank
\(k\), horizon \(T\), and the optimization and guidance parameters
listed in Appendix~\ref{app:guidance-details}.
\STATE Spectrally initialize \(X_0\in\mathbb R^{n\times k}\).
\STATE Set \(T_g=\lfloor\theta T\rfloor\).
\FOR{\(t=0,\ldots,T-1\)}
    \STATE Compute the residual
    \(e_t=\mathcal A(X_tX_t^\top)-y_w\) and
    \(G_t=\mathcal A^\ast e_t\).
    \IF{\(t<T_g\)}
        \STATE Set \(c_t=(1-t/T_g)^3\).
        \STATE Compute the scale-normalized inverse-trace direction
        \(g_{{\rm inv},t}\) from \eqref{eq:app-invtrace-direction} as defined in Appendix~\ref{app:guidance-details}.
        \STATE Draw \(s\) normal probes and compute the sampled Rayleigh
        direction \(g_{-,t}\) from \eqref{eq:app-rayleigh-direction} using the matrix-free sketch in Appendix~\ref{app:guidance-details}.
    \ELSE
        \STATE Set \(c_t=0\), \(g_{{\rm inv},t}=0\), and \(g_{-,t}=0\).
    \ENDIF
    \STATE Form
    \[
        g_t
        =
        2G_tX_t
        +
        \alpha_0c_t g_{{\rm inv},t}
        +
        \beta_0c_t g_{-,t}.
    \]
    \STATE Set
    \(S_t=X_t^\top X_t\) and
    \(\xi_t=\xi\,\max\{\operatorname{tr}(S_t)/k,\epsilon_0\}\).
    \STATE Take the preconditioned step
    \[
        X_{t+1}
        =
        X_t-\eta_tg_t(S_t+\xi_tI_k)^{-1}.
    \]
    \IF{\(t<T_g\) and \(h\) divides \(t+1\)}
        \STATE Apply the normal-direction correction
        \eqref{eq:app-periodic-correction} in Appendix~\ref{app:guidance-details} using a fresh sketch.
    \ENDIF
\ENDFOR
\STATE Set all guidance weights to zero and polish on \(f_w\) alone.
\RETURN The first full-rank polished iterate satisfying
\eqref{eq:critical-filter}.
\end{algorithmic}
\end{algorithm}

\subsection{Certificate-Guided Optimization}
\label{sec:certificate-guidance}

Equation~\eqref{eq:method-rho-k} exposes two complementary controls.
Protecting \(\lambda_k(M_t)\) prevents the denominator of
\(\rho_{k,t}^-\) from collapsing, while reducing \(b_t^-\) suppresses
negative normal curvature. Since \(\rho_k^-=b^-/\lambda_k\), the bounds
\(\lambda_k(\widetilde M)\geq a\lambda_k(M)\) and
\(b^-(\widetilde M)\leq c b^-(M)\) immediately imply
\(\rho_k^-(\widetilde M;w)\leq
(c/a)\rho_k^-(M;w)\). Thus both quantities must be controlled.
Thus, improving one component alone may not raise the certified margin
if the other worsens comparably.

To protect the active floor, we use the \(k\times k\) Gram matrix
\(S_t=X_t^\top X_t\) and the inverse-trace penalty
\[
\Omega_{\rm inv}(X_t)
=
(\bar\lambda_t+\varepsilon_t)
\tr\!\left((S_t+\varepsilon_t I_k)^{-1}\right),
\qquad
\bar\lambda_t=\frac{\tr(S_t)}{k}.
\]
Its scale-normalized gradient \(g_{{\rm inv},t}\) acts most strongly
on weak factor directions and therefore targets the denominator
\(\lambda_k(M_t)\) without forming \(M_t\).

To control the numerator, we project Gaussian probes onto
\(\operatorname{range}(X_t)^\perp\) and evaluate their normal Rayleigh quotients
against \(G_t\). The direction \(g_{-,t}\) penalizes the detected
negative quotients and moves them toward zero, providing a
matrix-free approximation to reducing \(b_t^-\). Exact gradients,
probe construction, and the periodic normal-direction correction
are given in Appendix~\ref{app:guidance-details}.

During the guidance window, the two directions are combined with the
original gradient:
\begin{equation}
\label{eq:guided-direction}
\begin{aligned}
g_t
&=
2G_tX_t+\alpha_0c_tg_{{\rm inv},t}
+\beta_0c_tg_{-,t},
\\
c_t
&=
(1-t/T_g)^3,
\qquad
T_g=\lfloor\theta T\rfloor .
\end{aligned}
\end{equation}
We take the damped Gram-preconditioned step
\[
X_{t+1}
=
X_t-\eta_tg_t
\left(S_t+\xi_tI_k\right)^{-1},
\qquad
\xi_t=\xi\max\{\bar\lambda_t,\epsilon_0\}.
\]
After \(T_g\), all guidance terms are removed. We then optimize
\(f_w\) alone until
\begin{equation}
\label{eq:critical-filter}
\frac{\norm{\nabla f_w(X)}_F}
{\max\{1,\norm{X}_F\}}
\leq \varepsilon_{\rm crit}.
\end{equation}
Consequently, guidance affects trajectory and solution selection,
but every reported critical point is an approximate critical point of the
original objective rather than of a permanently regularized loss.

\subsection{Empirical Protocol and Results}
\label{sec:empirical-validation}
We evaluate two questions: whether
Algorithm~\ref{alg:decayed-balanced-guidance} reduces end-to-end
critical-point response without degrading reconstruction, and how
its robustness--reconstruction trade-off compares with existing
low-rank recovery methods under controlled initialization.

\paragraph{Setup.}
We generate rank-\(r_\star\) PSD targets and use five
over-parameterized configurations with \(n\) ranging from \(50\) to
\(150\), \(r_\star\) from \(2\) to \(5\), and \(k\) from \(5\) to
\(12\). Measurements use implicit bilinear sensing with \(m=6nk\).
We consider isotropic Gaussian and Rademacher operators, together
with a weak-tail operator that attenuates directions orthogonal to
the target subspace by a factor \(0.35\). The latter is a deliberately
anisotropic stress test, not a heavy-tailed noise model. We use five
seeds and three independent Gaussian measurement perturbations per
configuration--seed instance, each with norm
\(0.08\norm{\Aop(M_\star)}_2\). 
Appendix~\ref{app:guidance-details} gives the exact configurations, seeds, operator construction, and optimization details.

\paragraph{Protocols and metrics.}
The primary end-to-end protocol initializes the clean and noisy runs
from their respective observations,
\(\operatorname{SpecInit}(\Aop^\ast y_0)\) and
\(\operatorname{SpecInit}(\Aop^\ast y_{w_j})\). The conditional
shared observed-spectral protocol instead initializes both members
of each pair from the same noisy-data initializer
\(\operatorname{SpecInit}(\Aop^\ast y_{w_j})\), thereby isolating
trajectory selection from initialization sensitivity. We report results separately.

For protocol \(p\in\{\mathrm{E2E},\mathrm{Shared}\}\), method \(a\), and matched instance \(s\), define
\[
    \mathsf R_{p,s}^{a}=\operatorname*{median}_{j=1,2,3} \norm{M_{{\rm crit},p,s}^{a}(w_j)-M_{{\rm crit},p,s}^{a}(0)}_F/\norm{w_j}_2
\]
and \(\mathsf Q_{p,s}=\mathsf R_{p,s}^{\rm A1}/\mathsf R_{p,s}^{\rm Plain}\). We form \(\mathsf Q_{p,s}\) for each matched instance and report its median over the \(25\) instances for each operator. We also report \(\mathsf E_{{\rm rel},s}=\norm{M_{{\rm crit},s}^{\rm A1}-M_\star}_F/\norm{M_{{\rm crit},s}^{\rm Plain}-M_\star}_F\).

We report operator-level medians of \(\mathsf Q_{p,s}\) over the
\(25\) matched instances. Here, \(\mathsf Q<1\) and \(\mathsf E_{\rm rel}<1\) indicate lower
paired critical-point sensitivity and reconstruction error than
Plain, respectively.


\begin{table}[t]
\centering
\small
\setlength{\tabcolsep}{2.5pt}
\renewcommand{\arraystretch}{1.08}
\begin{tabular}{llrrrr}
\toprule
Protocol & Operator & Ratio & \(\Delta_{\mathsf R}\) &
\(\mathsf E_{\rm rel}\) & Joint \\
\midrule
E2E & Gauss.    & 0.788 & 21.2\% & 0.887 & 23/25 \\
E2E & Radem.    & 0.785 & 21.5\% & 0.863 & 25/25 \\
E2E & Weak-tail & 0.180 & 82.0\% & 0.625 & 25/25 \\
\midrule
Shared & Gauss.    & 0.781 & 21.9\% & 0.887 & 23/25 \\
Shared & Radem.    & 0.771 & 22.9\% & 0.863 & 25/25 \\
Shared & Weak-tail & 0.117 & 88.3\% & 0.625 & 25/25 \\
\bottomrule
\end{tabular}
\caption{Algorithm~\ref{alg:decayed-balanced-guidance} relative to
Plain. Ratio is the median instancewise response ratio
\(\mathsf Q_{p,s}\), and
\(\Delta_{\mathsf R}=1-\operatorname{median}_s\mathsf Q_{p,s}\).
Joint counts instances in which both response and reconstruction
improve. Each row contains \(25\) configuration--seed instances.}
\label{tab:protocol-summary}
\end{table}

\begin{figure}
    \centering    \includegraphics[width=\linewidth]{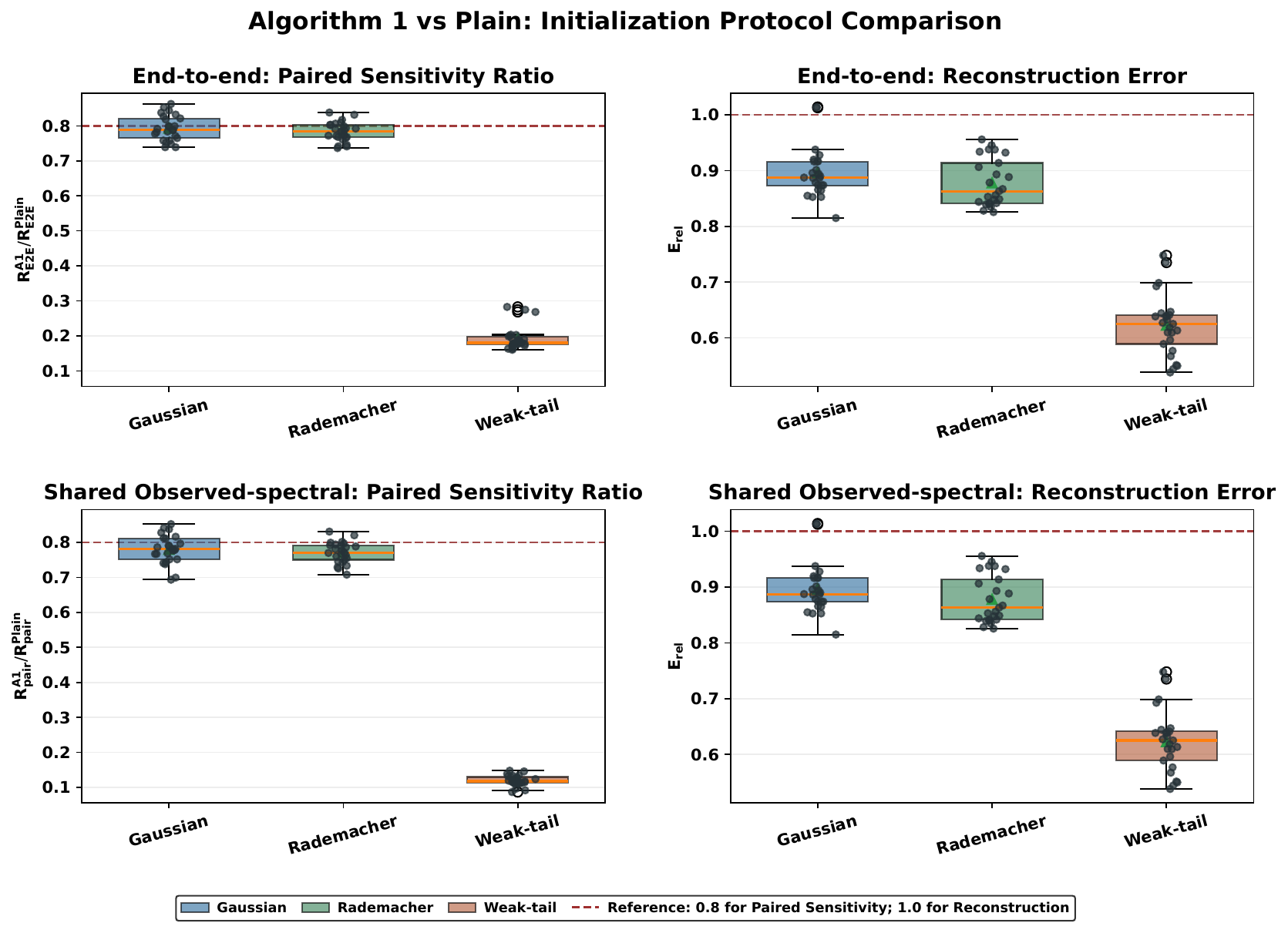}
    \caption{Distributions of the Algorithm~1/Plain Paired Sensitivity Ratio
(left) and reconstruction-error ratios (right) under E2E (top) and
Shared (bottom) initialization. Smaller values favor Algorithm~1;
the dashed lines mark \(20\%\) response improvement and reconstruction
parity. The right panels coincide because both protocols use the same
noisy-data runs.}
    \label{fg:ete-sos}
\end{figure}

\paragraph{Main results.}
Table~\ref{tab:protocol-summary} reports operator-level medians of
the paired critical-point sensitivity ratio \(\mathsf Q\) and
relative reconstruction error \(\mathsf E_{\rm rel}\), while
Figure~\ref{fg:ete-sos} shows their instance-level distributions.
Under E2E, the median \(\mathsf Q\) values are \(0.788\), \(0.785\),
and \(0.180\), and the corresponding median
\(\mathsf E_{\rm rel}\) values are \(0.887\), \(0.863\), and \(0.625\)
for Gaussian, Rademacher, and weak-tail sensing. Both metrics are
below one in \(73/75\) matched instances. All reported Algorithm~1
solutions pass the unguided criticality and full-rank checks.

Under Shared initialization, the median \(\mathsf Q\) values are
\(0.781\), \(0.771\), and \(0.117\), with the same
\(\mathsf E_{\rm rel}\) values and joint count. Because Shared
conditions out initialization sensitivity, it is treated as a
controlled rather than primary result. The weak-tail gains reflect
the anisotropic stress test, not an easier sensing problem.

\paragraph{Robust recovery and reconstruction baselines.}
L1-SubGM~\cite{ma2023global} and
ScaledSM~\cite{tong2021scaledsm} are robust matrix-recovery methods
using nonsmooth \(\ell_1\)-type losses for corrupted measurements.
GNMR~\cite{zilber2022gnmr} is included as a strong nonconvex
reconstruction baseline, testing whether lower paired sensitivity
merely results from improved recovery.

\begin{figure}
    \centering    \includegraphics[width=\linewidth]{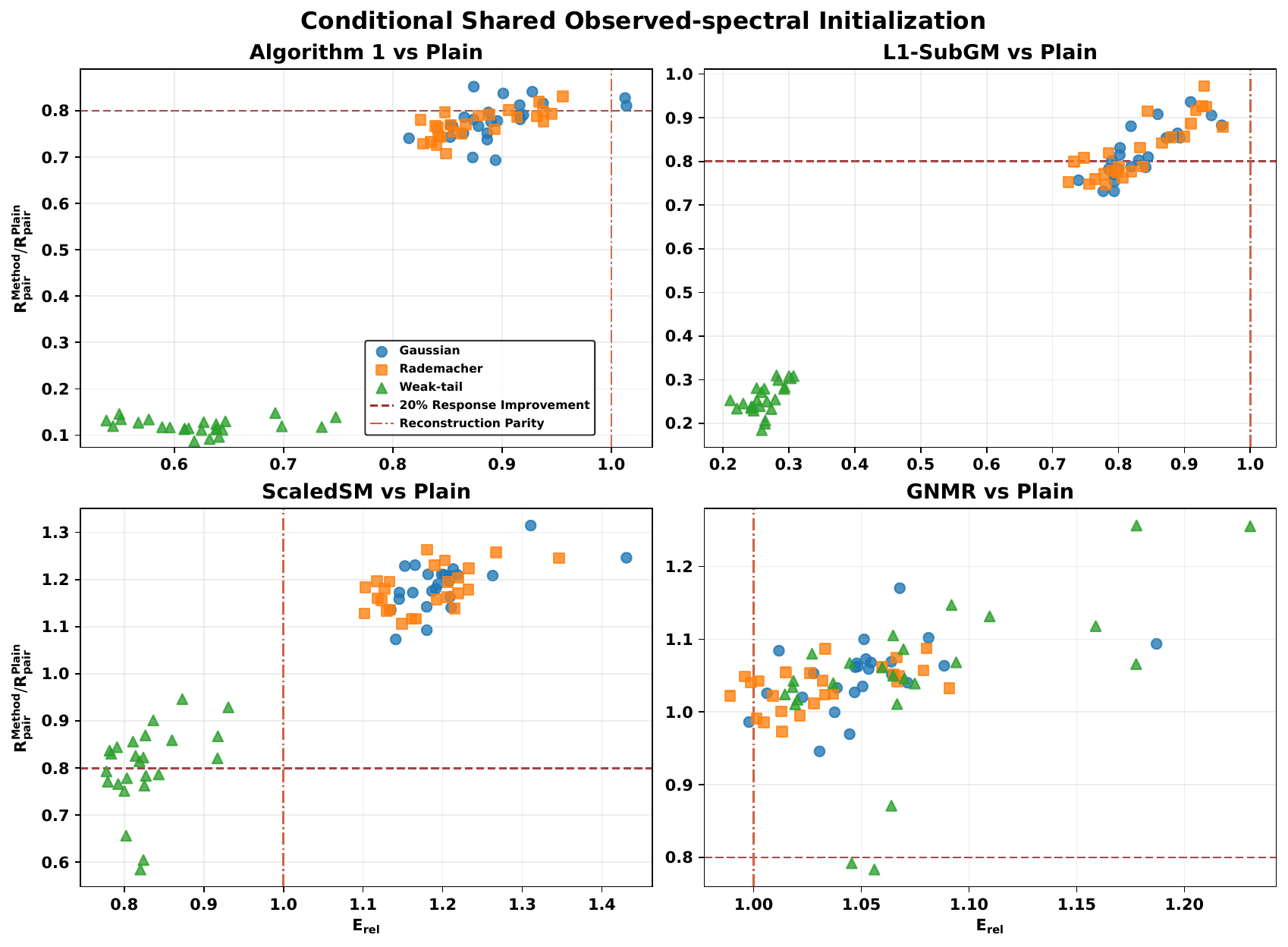}
    \caption{Conditional Shared comparison relative to Plain. Each point
is one configuration--seed instance. The horizontal line marks a
\(20\%\) reduction in paired sensitivity, and the vertical line marks
reconstruction parity; the lower-left region improves both metrics.}
    \label{fig:shared}
\end{figure}

Figure~\ref{fig:shared} compares all methods under Shared using the
same paired-sensitivity and reconstruction metrics. Algorithm~1 has
the smallest pooled median paired-sensitivity ratio (\(0.751\)),
compared with \(0.776\) for L1-SubGM, \(1.155\) for ScaledSM, and
\(1.049\) for GNMR, and improves both metrics in \(73/75\) instances.
L1-SubGM also gives favorable trade-offs, especially under weak-tail
sensing; ScaledSM improves mainly in that setting, while GNMR gives
no consistent sensitivity reduction. All displayed solutions are
polished on \(f_w\), and \(f_w\)-criticality claims refer to the
polished critical points.

\paragraph{Additional experiments.}
Appendix~\ref{app:guidance-details} reports further E2E,
shared-initialization, and large-scale matrix-sensing results with
consistent trends. It also presents a preliminary
Algorithm~1-inspired LoRA~\cite{hu2021lora} study that fine-tunes
RoBERTa-base~\cite{liu2019roberta} on
GLUE/SST-2~\cite{wang-etal-2018-glue}, improving corrupted-data
accuracy by approximately one percentage point.

\section{Conclusion}

We develop a rank-aware robustness theory for over-parameterized PSD matrix sensing. Rank-deficient SOSPs admit direct recovery guarantees, whereas full-rank robustness depends on tangent curvature, the active spectral floor, and negative normal leakage. These quantities motivate a transient, anchor-free method whose polished critical points improve paired robustness across Gaussian, Rademacher, and weak-tail operators without worsening reconstruction. Thus robustness is governed by solution geometry and solution selection, not factor width alone.

\bibliography{aaai2027}

\clearpage

\setcounter{secnumdepth}{2}
\section*{Appendix}

\appendix

\setcounter{figure}{2}
\setcounter{table}{1}
\setcounter{equation}{39}

\section{Algorithmic and Experimental Details}
\label{app:guidance-details}

This appendix provides the implementation and diagnostic details of
Algorithm~\ref{alg:decayed-balanced-guidance}, followed by the complete
experimental protocol and additional matrix-sensing results omitted from
Section~\ref{sec:empirical-validation}.

\subsection{Algorithmic and Diagnostic Details}
\label{app:algorithmic-details}

Algorithm~\ref{alg:decayed-balanced-guidance} is implemented using only
factor-space operations, applications of the sensing operator and its
adjoint, and inverses of \(k\times k\) matrices. It never materializes
\(M_t=X_tX_t^\top\) or the ambient normal block as a dense
\(n\times n\) matrix.

\paragraph{Matrix-free normal sketch.}
Let \(U_t\) be an orthonormal basis of
\(\operatorname{range}(X_t)\). Given
\(\Xi_t\in\mathbb R^{n\times s}\) with i.i.d.\ Gaussian entries, we form
\begin{equation}
\label{eq:app-normal-sketch}
Q_t
=
\operatorname{orth}
\bigl(\Xi_t-U_t(U_t^\top\Xi_t)\bigr).
\end{equation}
Thus, the projection onto
\(\operatorname{range}(X_t)^\perp\) uses only products with the thin
matrix \(U_t\). Writing the columns of \(Q_t\) as \(q_{t,j}\), define
\[
r_{t,j}
=
\left\langle G_t,q_{t,j}q_{t,j}^\top\right\rangle
=
\left\langle e_t,
\mathcal A(q_{t,j}q_{t,j}^\top)\right\rangle,
\]
where
\[
e_t=\mathcal A(X_tX_t^\top)-y_w,
\qquad
G_t=\mathcal A^\ast e_t.
\]
Negative values of \(r_{t,j}\) detect negative normal curvature. Let
\[
\mathcal J_t=\{j:r_{t,j}<0\}.
\]
The sampled normal-curvature penalty is
\[
\widehat\Phi_{-,t}
=
\begin{cases}
\displaystyle
\frac{1}{2|\mathcal J_t|}
\sum_{j\in\mathcal J_t}(-r_{t,j})^2,
& \mathcal J_t\neq\varnothing,\\[2mm]
0,
& \mathcal J_t=\varnothing.
\end{cases}
\]
The probes are held fixed while differentiating the current update. If
\[
H_{t,j}
=
\mathcal A^\ast\mathcal A(q_{t,j}q_{t,j}^\top),
\]
the resulting guidance direction is
\begin{equation}
\label{eq:app-rayleigh-direction}
g_{-,t}
=
-\frac{2}{|\mathcal J_t|}
\sum_{j\in\mathcal J_t}
(-r_{t,j})H_{t,j}X_t
\end{equation}
when \(\mathcal J_t\neq\varnothing\), and zero otherwise. The negative
sign moves the sampled negative Rayleigh quotients toward zero.

\paragraph{Inverse-trace direction.}
Let
\[
S_t=X_t^\top X_t,
\qquad
\bar\lambda_t=\frac{\operatorname{tr}(S_t)}{k},
\]
and set
\[
\varepsilon_t
=
\varepsilon_{\rm frac}
\max\{\bar\lambda_t,\epsilon_0\}.
\]
With
\[
R_t=(S_t+\varepsilon_tI_k)^{-1},
\qquad
\tau_t=\bar\lambda_t+\varepsilon_t,
\]
the scale-normalized inverse-trace penalty is
\[
\Omega_{\rm inv}(X_t)
=
\tau_t\operatorname{tr}(R_t).
\]
During the current update, \(\varepsilon_t\) is treated as fixed.
We use the \(\tau_t\)-normalized gradient direction
\begin{equation}
\label{eq:app-invtrace-direction}
\begin{aligned}
g_{{\rm inv},t}
&:=
\tau_t^{-1}
\nabla_X\Omega_{\rm inv}(X_t) \\
&=
2X_t
\left[
\frac{\operatorname{tr}(R_t)}
     {k\tau_t}I_k
-
R_t^2
\right].
\end{aligned}
\end{equation}
The \(R_t^2\) term acts most strongly on the weakest factor
directions, while the identity term controls the overall
factor scale.

\paragraph{Periodic normal correction.}
Every \(h\) iterations during the guidance window, the implementation
also applies a low-frequency correction. For a fresh set of normal
probes, define
\[
a_{t,j}
=
\frac{(-r_{t,j})_+}
     {\sum_{\ell=1}^{s}(-r_{t,\ell})_+}.
\]
If all sampled Rayleigh quotients are nonnegative, the correction is
skipped. Otherwise, set
\[
\widetilde v_t
=
\sum_{j=1}^{s}a_{t,j}q_{t,j},
\qquad
v_t
=
\frac{\widetilde v_t-U_t(U_t^\top\widetilde v_t)}
     {\|\widetilde v_t-U_t(U_t^\top\widetilde v_t)\|_2}.
\]
We approximate the weakest right factor direction \(u_t\in\mathbb R^k\)
using four inverse iterations with a damped version of \(S_t\). The
correction is
\begin{equation}
\label{eq:app-periodic-correction}
X_{t+1}
\leftarrow
X_{t+1}+\gamma_t v_tu_t^\top,
\end{equation}
where
\[
\gamma_t
=
\min\left\{
\gamma_0c_t
\sqrt{\max_j(-r_{t,j})_+},
\;
\gamma_{\max}\sqrt{\bar\lambda_t}
\right\}.
\]
The cap prevents the correction from dominating the base optimization
step.

\paragraph{Decay, preconditioning, and polishing.}
Let
\[
T_g=\lfloor\theta T\rfloor,
\qquad
c_t=
\left(1-\frac{t}{T_g}\right)^3,
\quad 0\leq t<T_g.
\]
During this window, the factor-space direction is
\[
d_t
=
2G_tX_t
+\alpha_0c_tg_{{\rm inv},t}
+\beta_0c_tg_{-,t}.
\]
We apply the damped Gram-preconditioned update
\begin{equation}
\label{eq:app-guided-update}
X_{t+1}
=
X_t-\eta_t d_t
\left(S_t+\xi_tI_k\right)^{-1},
\qquad
\xi_t
=
\xi\max\{\bar\lambda_t,\epsilon_0\}.
\end{equation}
The learning rate uses a \(25\)-iteration linear warm-up before reaching
its nominal value. For \(t\geq T_g\), all guidance terms and periodic
corrections are removed, and the method optimizes only
\[
f_w(X)
=
\frac12
\left\|
\mathcal A(XX^\top)-y_w
\right\|_2^2.
\]

After native optimization, the resulting solution is polished on
\(f_w\) for at most \(500\) additional steps. A reported critical point
must remain full rank and satisfy
\begin{equation}
\label{eq:app-criticality}
\frac{\|\nabla f_w(X)\|_F}
     {\max\{1,\|X\|_F\}}
\leq
\varepsilon_{\rm crit}.
\end{equation}
Therefore, guidance affects trajectory and critical-point selection,
whereas stationarity is always tested on the original unmodified
objective.

\paragraph{Computational cost.}
The inverse-trace direction requires forming \(S_t\), solving a
\(k\times k\) linear system, and multiplying an \(n\times k\) factor
by a \(k\times k\) matrix. Its additional cost is
\[
O(nk^2+k^3).
\]
The normal guidance uses \(s\) rank-one sensing evaluations and
adjoint-factor products together with an \(n\times s\) QR
factorization. Since \(s\) and \(k\) are small relative to \(n\), the
implementation stores only thin factor and sketch matrices and avoids
ambient eigendecompositions.

\paragraph{Finite-probe diagnostic.}
Exact RIP certification is computationally intractable in general.
We therefore use a finite-probe diagnostic only as an empirical screen.
For each probe, draw an orthonormal matrix
\(V_j\in\mathbb R^{n\times q}\), where \(q=\min\{2k,n\}\), and a
normalized Gaussian vector \(z_j\in\mathbb R^q\). Set
\[
Z_j=V_j\operatorname{diag}(z_j)V_j^\top,
\qquad
\|Z_j\|_F=1.
\]
Using \(N_{\rm RIP}=24\) probes, we compute
\[
\widehat\delta
=
\max_{1\leq j\leq N_{\rm RIP}}
\left|
\|\mathcal A(Z_j)\|_2^2-1
\right|,
\qquad
\widehat\mu
=
1-\widehat\delta-\rho_k^-.
\]
These quantities test only the sampled directions and are not certified
RIP constants or theorem-level robustness certificates.

\subsection{Matrix-Sensing Experiments}
\label{app:matrix-sensing-experiments}

\paragraph{Data and sensing operators.}
For each instance, we generate
\[
M_\star
=
U_\star\Lambda_\star U_\star^\top,
\qquad
\operatorname{rank}(M_\star)=r_\star,
\]
where \(U_\star\) is obtained by orthonormalizing a Gaussian matrix.
The nonzero eigenvalues of \(M_\star\) are geometrically spaced between
\(1.5\) and \(0.55\).

We use the implicit bilinear sensing operator
\[
\mathcal A_i(Z)
=
\frac{a_i^\top Zb_i}{\sqrt m},
\qquad
m=6nk.
\]
For Gaussian sensing, \(a_i\) and \(b_i\) have independent standard
normal entries. For Rademacher sensing, their entries are independent
random signs. For weak-tail sensing, Gaussian vectors are decomposed
into components inside and outside
\(\operatorname{range}(M_\star)\), and the orthogonal component is
scaled by \(0.35\). This produces an anisotropic stress test with
reduced visibility in slack directions. The target subspace is used
only to generate the synthetic operator and is not supplied to any
optimization method.

The formal grid is
\[
\begin{aligned}
(n,r_\star,k)\in\{&
(50,2,5),(75,3,6),(100,3,8),\\
&(125,4,10),(150,5,12)\}.
\end{aligned}
\]
We use seeds \(15,66,102,999,1325\), giving \(25\) matched
configuration--seed instances per operator and \(75\) in total.
Measurement perturbations satisfy
\[
y_{w_j}
=
\mathcal A(M_\star)+w_j,
\qquad
\|w_j\|_2
=
0.08\|\mathcal A(M_\star)\|_2,
\]
where \(j=1,2,3\) and each \(w_j\) is an independently sampled
Gaussian direction normalized to the prescribed magnitude.

\paragraph{Non-oracle protocols.}
Both formal protocols construct their initializations only from
observed measurements. Under E2E, the clean and perturbed runs use
separately constructed observed-data initializations. In particular,
Algorithm~1 and Plain use
\[
\operatorname{SpecInit}(\mathcal A^\ast y_0)
\quad\text{and}\quad
\operatorname{SpecInit}(\mathcal A^\ast y_{w_j}),
\]
respectively. Hence E2E includes sensitivity introduced by the
initialization map.

Under Shared observed-spectral initialization, both members of each
clean--perturbed pair use the same initialization constructed from
\(y_{w_j}\). This conditions out initialization differences and
isolates the subsequent optimization trajectory. E2E is the primary
protocol, while Shared is a controlled non-oracle comparison.

\paragraph{Evaluation metrics.}
Let \(M_{p,s}^{a}(w_j)\) and \(M_{p,s}^{a}(0)\) denote the polished
critical points produced by method \(a\), protocol \(p\), and matched
instance \(s\). Define
\[
d_{p,s,j}^{a}
=
\frac{
\|M_{p,s}^{a}(w_j)-M_{p,s}^{a}(0)\|_F
}{
\|w_j\|_2
}
\]
and
\[
\mathsf R_{p,s}^{a}
=
\operatorname{median}_{j=1,2,3}
d_{p,s,j}^{a}.
\]
The relative paired-sensitivity ratio is
\[
\mathsf Q_{p,s}^{a}
=
\frac{\mathsf R_{p,s}^{a}}
     {\mathsf R_{p,s}^{\rm Plain}}.
\]
For the primary noisy-data critical point \(M_s^a\), the relative
reconstruction error is
\[
\mathsf E_{{\rm rel},s}^{a}
=
\frac{\|M_s^a-M_\star\|_F}
     {\|M_s^{\rm Plain}-M_\star\|_F}.
\]
Values below one favor method \(a\). The primary noisy-data runs are
the same under E2E and Shared, so their
\(\mathsf E_{\rm rel}\) values coincide.

All formal runs are polished on \(f_w\) using tolerance
\(\varepsilon_{\rm crit}=0.025\). Every Algorithm~1/Plain matched row
passes the paired criticality check for all three perturbations under
both protocols. All reported primary solutions also satisfy
\(\lambda_k(XX^\top)>10^{-8}\).

\paragraph{Optimization settings.}
Algorithm~1 and Plain use learning rate \(0.035\) and at most \(220\)
native iterations. Algorithm~1 uses a \(25\%\) guidance window with
cubic decay, inverse-trace weight \(0.003\), Rayleigh-gradient weight
\(0.02\), correction weight \(0.08\), sketch size \(8\), correction
period \(10\), and Gram damping fraction \(0.05\).

L1-SubGM~\cite{ma2023global} uses step size \(0.04\), decay \(0.99\),
and initialization scale \(0.05\). ScaledSM~\cite{tong2021scaledsm}
uses scale parameter \(0.03\) and decay \(0.99\).
GNMR~\cite{zilber2022gnmr} uses \(40\) LSQR iterations per
Gauss--Newton step. Since these methods have different native
objectives, all their reported solutions are subsequently polished
on the common squared objective \(f_w\).

\begin{table}[t]
\centering
\small
\setlength{\tabcolsep}{3.5pt}
\renewcommand{\arraystretch}{1.05}
\begin{tabular}{@{}lr@{}}
\toprule
Parameter & Value \\
\midrule
Measurements & \(m=6nk\) \\
Formal seeds & \(15,66,102,999,1325\) \\
Native iterations & at most \(220\) \\
Algorithm~1/Plain learning rate & \(0.035\) \\
Relative perturbation norm & \(0.08\) \\
Paired perturbations & \(3\) \\
Guidance fraction & \(0.25\) \\
Guidance decay power & \(3\) \\
Normal sketch size & \(8\) \\
Normal correction period & \(10\) \\
Common polishing steps & at most \(500\) \\
Criticality tolerance & \(0.025\) \\
Finite-RIP probes & \(24\) \\
\bottomrule
\end{tabular}
\caption{Common settings for the formal matrix-sensing experiments.}
\label{tab:app-formal-settings}
\end{table}

\paragraph{Algorithm~1 versus Plain.}
Table~\ref{tab:app-configuration-results} reports configuration-level
medians over five seeds. The two \(\mathsf Q\) columns correspond to
E2E and Shared observed-spectral initialization. The reconstruction
ratio is common to both protocols because their primary noisy-data
runs are identical.

\begin{table}[t]
\centering
\small
\setlength{\tabcolsep}{3pt}
\renewcommand{\arraystretch}{1.03}
\begin{tabular}{@{}llrrr@{}}
\toprule
Operator & \((n,r_\star,k)\)
& \(\mathsf Q_{\rm E2E}\)
& \(\mathsf Q_{\rm Shared}\)
& \(\mathsf E_{\rm rel}\) \\
\midrule
Gauss. & \((50,2,5)\)   & .8323 & .8161 & .9375 \\
Gauss. & \((75,3,6)\)   & .8269 & .8123 & .9162 \\
Gauss. & \((100,3,8)\)  & .7883 & .7780 & .8874 \\
Gauss. & \((125,4,10)\) & .7651 & .7516 & .8864 \\
Gauss. & \((150,5,12)\) & .7483 & .7378 & .8739 \\
\midrule
Radem. & \((50,2,5)\)   & .8031 & .7991 & .9322 \\
Radem. & \((75,3,6)\)   & .8064 & .7898 & .9338 \\
Radem. & \((100,3,8)\)  & .7884 & .7708 & .8668 \\
Radem. & \((125,4,10)\) & .7658 & .7504 & .8487 \\
Radem. & \((150,5,12)\) & .7459 & .7333 & .8347 \\
\midrule
Weak-tail & \((50,2,5)\)   & .1722 & .1171 & .5960 \\
Weak-tail & \((75,3,6)\)   & .1768 & .1337 & .5512 \\
Weak-tail & \((100,3,8)\)  & .1827 & .1128 & .6250 \\
Weak-tail & \((125,4,10)\) & .1865 & .1119 & .6387 \\
Weak-tail & \((150,5,12)\) & .2683 & .1296 & .6468 \\
\bottomrule
\end{tabular}
\caption{Configuration-level medians for
Algorithm~\ref{alg:decayed-balanced-guidance} relative to Plain.
Each row aggregates five seeds. Values below one favor
Algorithm~\ref{alg:decayed-balanced-guidance}.}
\label{tab:app-configuration-results}
\end{table}

All configuration-level medians satisfy
\(\mathsf Q<1\) and \(\mathsf E_{\rm rel}<1\).
At the instance level, both inequalities hold in \(23/25\),
\(25/25\), and \(25/25\) cases for Gaussian, Rademacher, and
weak-tail sensing, respectively, under both protocols. Shared gives
uniformly smaller configuration-level \(\mathsf Q\) values because it
conditions out initialization sensitivity. Its weak-tail gains should
therefore be read as controlled evidence under the anisotropic stress
test, not as evidence that weak-tail sensing is easier.

\begin{figure}[t]
\centering
\includegraphics[width=\linewidth]
{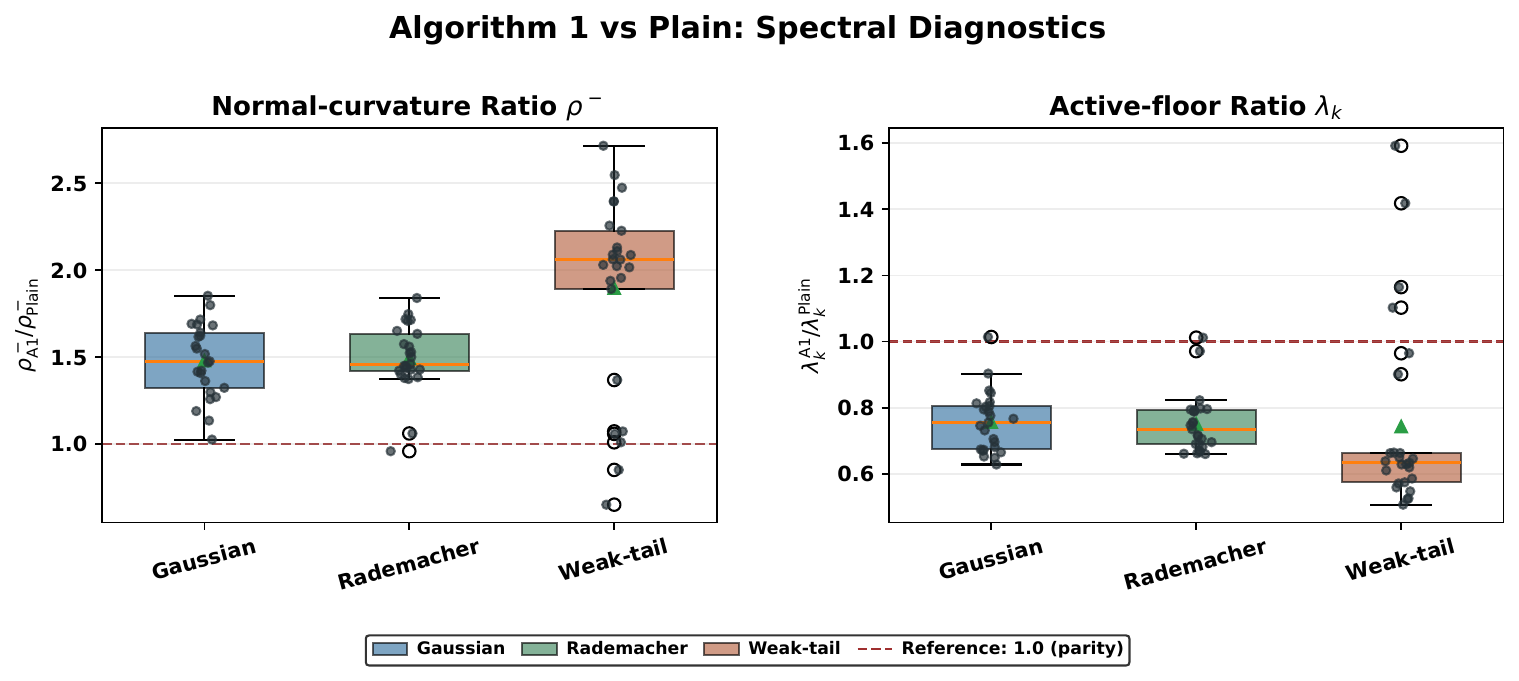}
\caption{Terminal spectral diagnostics for
Algorithm~\ref{alg:decayed-balanced-guidance} relative to Plain.
The panels show
\(\rho_{\rm A1}^{-}/\rho_{\rm Plain}^{-}\) and
\(\lambda_k^{\rm A1}/\lambda_k^{\rm Plain}\), respectively.
Each operator pools \(25\) matched instances. Smaller values favor
Algorithm~1 for \(\rho^-\), whereas larger values favor it for
\(\lambda_k\). E2E and Shared use the same primary noisy-data runs,
so these diagnostics coincide and are shown only once.}
\label{fig:common-rho-lambda}
\end{figure}

Figure~\ref{fig:common-rho-lambda} complements the direct paired
measurements with terminal solution geometry. The median
\(\rho^-\) ratios are \(1.476\), \(1.459\), and \(2.061\), while the
median active-floor ratios are \(0.755\), \(0.735\), and \(0.634\)
for Gaussian, Rademacher, and weak-tail sensing. Thus, the reduction
in \(\mathsf Q\) is not accompanied by uniform improvement in either
terminal diagnostic. These quantities describe local geometry at the
polished critical point; neither is intended to replace the directly
measured paired sensitivity.

\paragraph{External comparators.}
All external methods are evaluated relative to Plain using the same
metrics and common \(f_w\)-polishing. E2E recomputes the
observed-spectral initialization for each paired perturbation, whereas
Shared reuses a common observed-spectral initialization. Consequently,
only \(\mathsf Q\) changes between protocols; reconstruction,
\(\rho^-\), and \(\lambda_k\) are identical up to numerical precision.
Figures~\ref{fig:app-l1-subgm-diagnostics}--%
\ref{fig:app-gnmr-diagnostics} report the corresponding operator-wise
distributions. All displayed paired solutions pass the common
\(f_w\)-criticality check.

\begin{figure*}[p]
\centering\includegraphics[width=.68\textwidth]{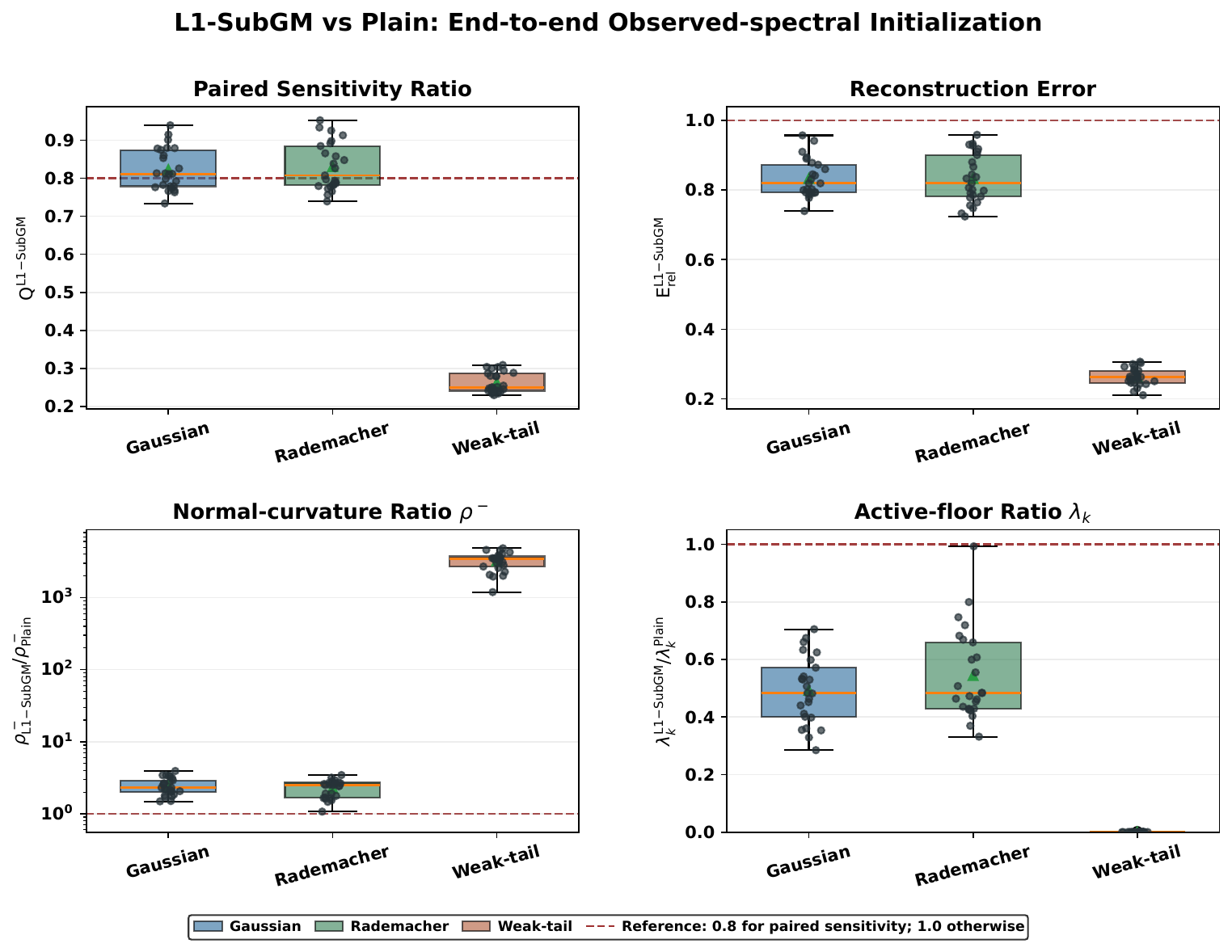}
\par\vspace{-1mm}\includegraphics[width=.68\textwidth]{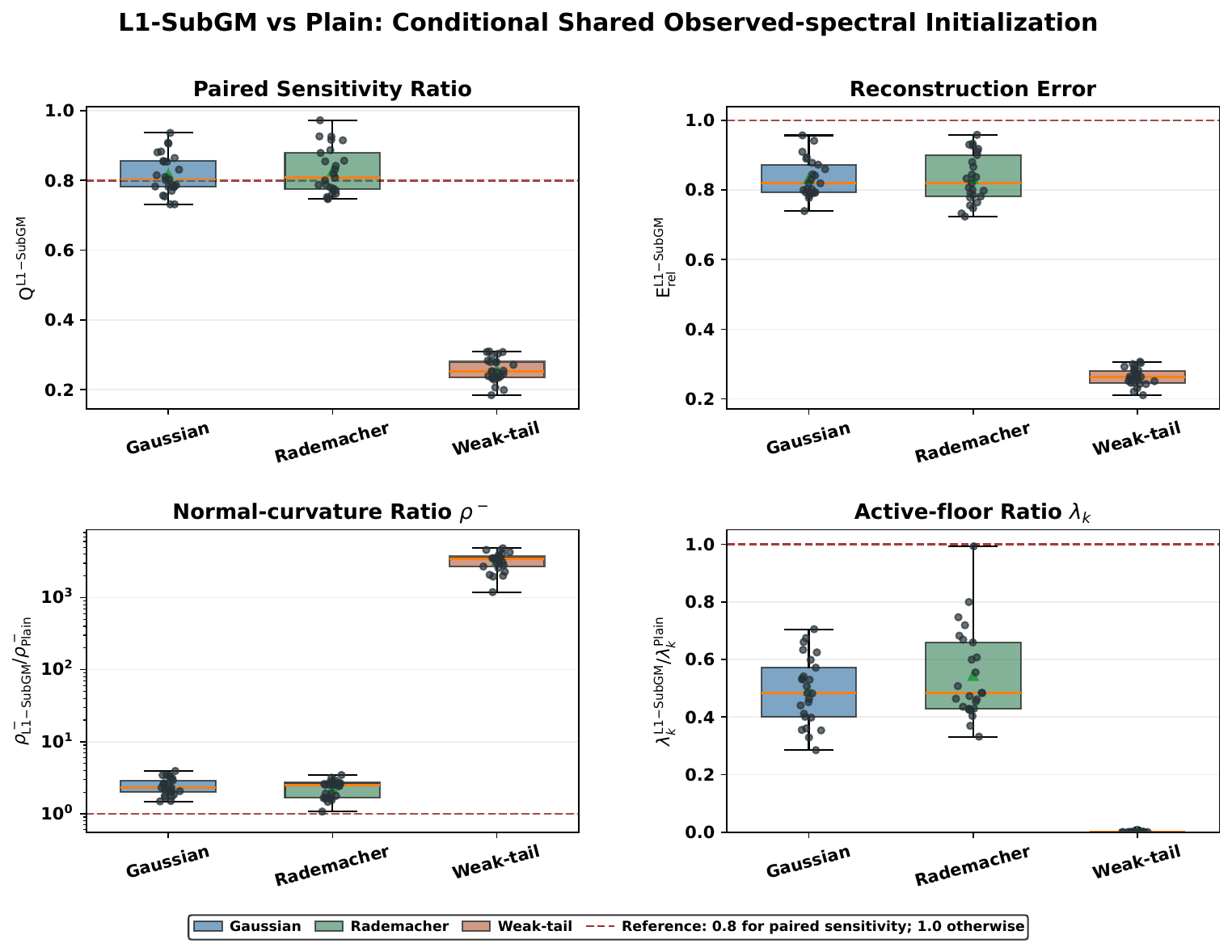}
\caption{L1-SubGM relative to Plain under E2E (top) and Shared
(bottom). From upper left to lower right, the panels show
\(\mathsf Q\), \(\mathsf E_{\rm rel}\), the \(\rho^-\) ratio, and the
\(\lambda_k\) ratio. The dashed references are \(\mathsf Q=0.8\) and
one for the other metrics. The \(\rho^-\) panel uses a logarithmic
scale because of the weak-tail range.}
\label{fig:app-l1-subgm-diagnostics}
\end{figure*}

L1-SubGM~\cite{ma2023global} has pooled median \(\mathsf Q\) values
\(0.779\) under E2E and \(0.776\) under Shared, together with pooled
\(\mathsf E_{\rm rel}=0.791\). Both metrics are below one in all
\(75\) instances, with the largest reductions under weak-tail
sensing. Its terminal spectral ratios nevertheless move in the
opposite directions, especially for weak-tail sensing, illustrating
that paired sensitivity is not determined by either diagnostic alone.

\begin{figure*}[p]
\centering
\includegraphics[width=.68\textwidth]
{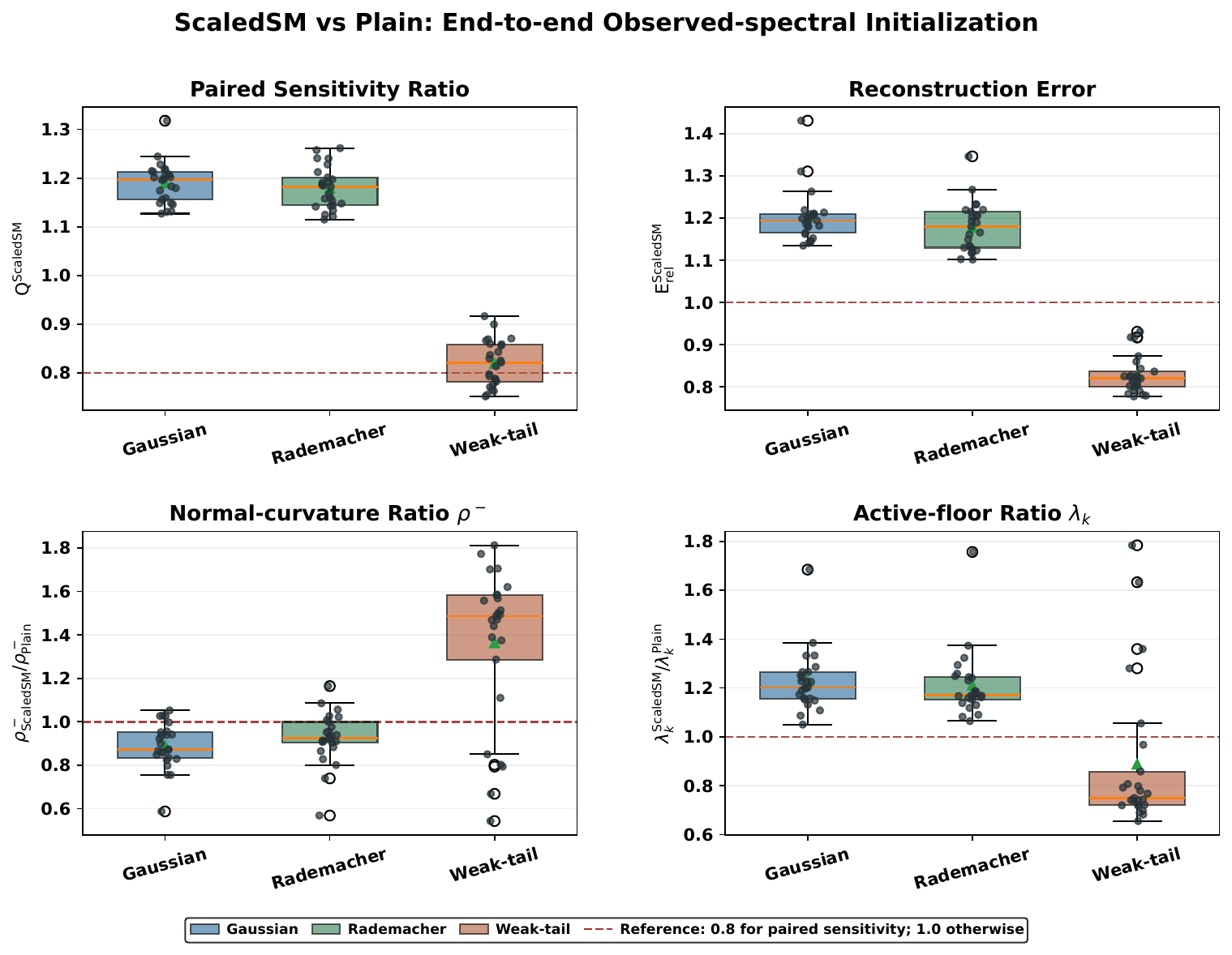}
\par\vspace{-1mm}
\includegraphics[width=.68\textwidth]{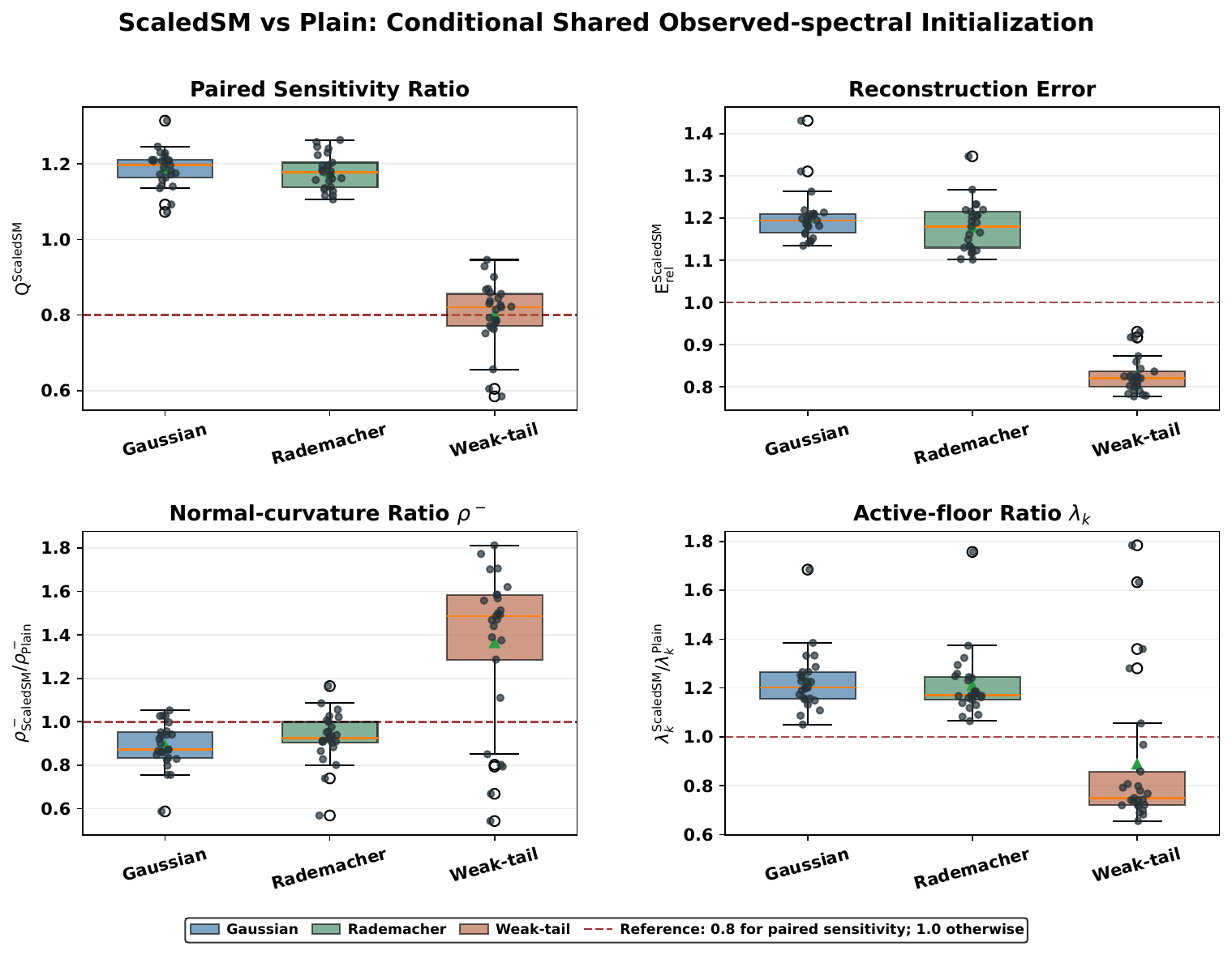}
\caption{ScaledSM relative to Plain under E2E (top) and Shared
(bottom), using the same four diagnostics and reference lines as
Figure~\ref{fig:app-l1-subgm-diagnostics}. Each operator contains
\(25\) matched configuration--seed instances.}
\label{fig:app-scaledsm-diagnostics}
\end{figure*}

ScaledSM~\cite{tong2021scaledsm} has pooled median \(\mathsf Q\)
values \(1.148\) and \(1.155\) under E2E and Shared, respectively,
and pooled \(\mathsf E_{\rm rel}=1.146\). It lowers both paired
sensitivity and reconstruction error in all \(25\) weak-tail
instances, but in none of the \(50\) Gaussian or Rademacher
instances. Its benefit is therefore specific to the anisotropic
stress test rather than uniform across sensing operators.

\begin{figure*}[p]
\centering
\includegraphics[width=.68\textwidth]{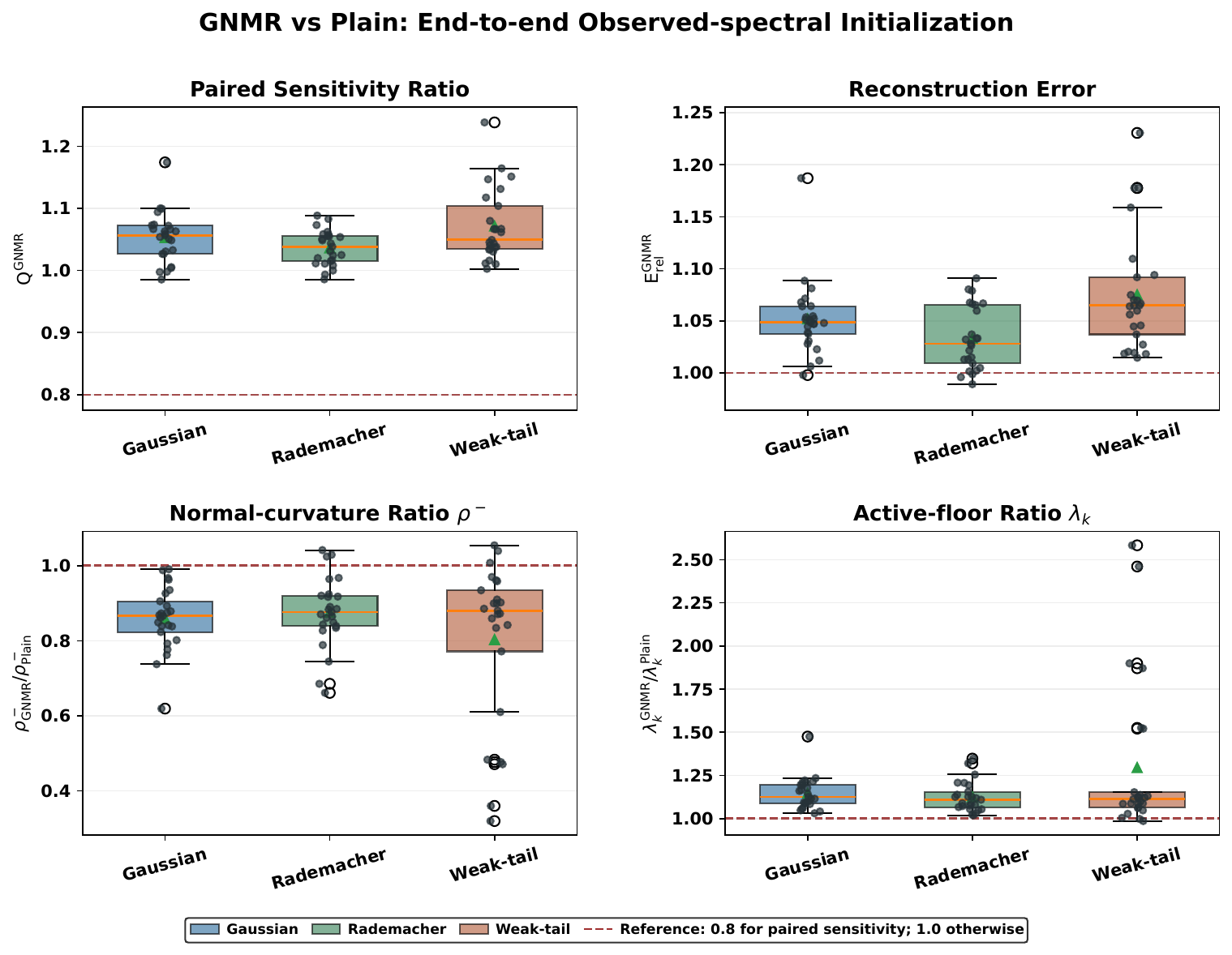}
\par\vspace{-1mm}
\includegraphics[width=.68\textwidth]
{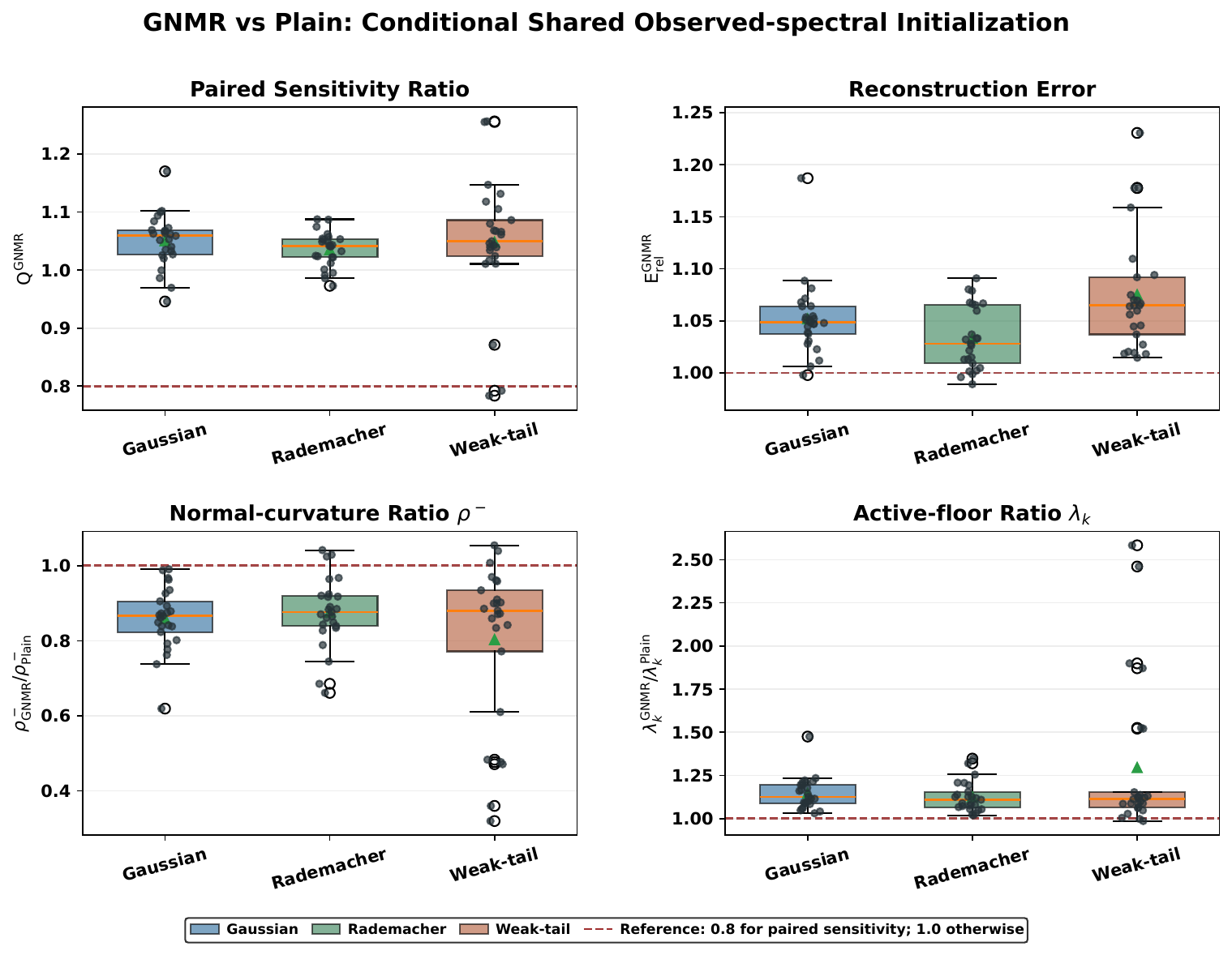}
\caption{GNMR relative to Plain under E2E (top) and Shared (bottom),
using the same four diagnostics and reference lines as
Figure~\ref{fig:app-l1-subgm-diagnostics}. Each operator contains
\(25\) matched configuration--seed instances.}
\label{fig:app-gnmr-diagnostics}
\end{figure*}

GNMR~\cite{zilber2022gnmr} has pooled median
\(\mathsf Q\simeq1.049\) under both protocols and pooled
\(\mathsf E_{\rm rel}=1.048\), with only one joint improvement among
the \(75\) instances. Although its median \(\rho^-\) ratios are below
one and its median active-floor ratios are above one, these favorable
terminal directions do not produce a consistent reduction in paired
sensitivity.

Taken together, the comparator results show that terminal
\(\rho^-\) and \(\lambda_k\) ratios do not by themselves rank methods
by paired sensitivity. The theoretical certificates use these
quantities jointly with the remaining local and operator-dependent
conditions; they do not assert cross-method monotonicity in either
terminal ratio alone. Accordingly, \(\mathsf Q\) remains the direct
experimental measure, while \(\rho^-\) and \(\lambda_k\) provide
complementary geometric diagnostics.

\paragraph{Large-scale experiments.}
We evaluate whether the paired-sensitivity improvements persist as the
ambient dimension increases. We use
\[
(n,r_\star,k)\in
\{(250,4,10),(500,5,12),(1000,5,12)\},
\]
with Gaussian, Rademacher, and weak-tail sensing, five seeds per
configuration, and \(m=6nk\) measurements. The relative measurement-noise
level and paired-perturbation norm are both \(0.08\), and each paired
critical-point sensitivity is aggregated over three perturbations.
Algorithm~\ref{alg:decayed-balanced-guidance} and Plain use at most
\(220\) native iterations, followed by the same \(f_w\)-polishing
procedure for at most \(500\) iterations. The criticality tolerance is
\(0.025\).

We evaluate both the primary E2E protocol and the controlled Shared
observed-spectral protocol. In E2E, the initialization changes with the
perturbed observations. Shared instead uses the same observed-spectral
initialization for the two objectives in each matched pair, thereby
conditioning out initialization sensitivity. The main noisy-data runs
are identical under the two protocols; hence their reconstruction and
terminal spectral diagnostics coincide, while their relative paired
sensitivities \(\mathsf Q\) can differ.

\begin{table*}[t]
\centering
\scriptsize
\setlength{\tabcolsep}{3.0pt}
\renewcommand{\arraystretch}{1.06}
\begin{tabular}{llrrrrrrrr}
\toprule
\(n\) & Operator &
\(\mathsf Q_{\rm E2E}\) &
\(\Delta_{\rm E2E}\) &
\(\mathsf Q_{\rm Sh}\) &
\(\Delta_{\rm Sh}\) &
\(\mathsf E_{\rm rel}\) &
\(\rho^-_{\rm Plain}\) &
\(\rho^-_{\rm A1}\) &
Joint \\
\midrule
250  & Gauss.    & 0.715 & 28.5\% & 0.710 & 29.0\% & 0.821 & 0.2624 & 0.4093 & 5/5 \\
250  & Radem.    & 0.712 & 28.8\% & 0.706 & 29.4\% & 0.824 & 0.2543 & 0.3954 & 5/5 \\
250  & Weak-tail & 0.307 & 69.3\% & 0.125 & 87.5\% & 0.668 & 0.0077 & 0.0084 & 5/5 \\
\midrule
500  & Gauss.    & 0.713 & 28.7\% & 0.703 & 29.7\% & 0.823 & 0.1826 & 0.2850 & 5/5 \\
500  & Radem.    & 0.722 & 27.8\% & 0.702 & 29.8\% & 0.834 & 0.1951 & 0.3058 & 5/5 \\
500  & Weak-tail & 0.385 & 61.5\% & 0.120 & 88.0\% & 0.586 & 0.0029 & 0.0061 & 5/5 \\
\midrule
1000 & Gauss.    & 0.709 & 29.1\% & 0.684 & 31.6\% & 0.835 & 0.1287 & 0.1950 & 5/5 \\
1000 & Radem.    & 0.720 & 28.0\% & 0.684 & 31.6\% & 0.840 & 0.1384 & 0.1934 & 5/5 \\
1000 & Weak-tail & 0.570 & 43.0\% & 0.157 & 84.3\% & 0.584 & 0.0031 & 0.0041 & 5/5 \\
\bottomrule
\end{tabular}
\caption{Formal large-scale comparison of
Algorithm~\ref{alg:decayed-balanced-guidance} with Plain over five
seeds per \((n,\mathrm{operator})\) cell. The quantities
\(\mathsf Q_{\rm E2E}\) and \(\mathsf Q_{\rm Sh}\) are the median
instancewise relative paired critical-point sensitivities under E2E
and Shared observed-spectral initialization. The corresponding
sensitivity reductions are
\(\Delta_{\rm E2E}=100(1-\mathsf Q_{\rm E2E})\) and
\(\Delta_{\rm Sh}=100(1-\mathsf Q_{\rm Sh})\).
The reconstruction ratio is common to both protocols because their
main noisy-data runs are identical. The two \(\rho^-\) columns report
the median terminal normal-curvature ratios of Plain and Algorithm~1;
these values are likewise common to the two protocols up to numerical
precision. Joint counts seeds for which both
\(\mathsf Q<1\) and \(\mathsf E_{\rm rel}<1\); all cells have \(5/5\)
under both protocols.}
\label{tab:large-scale-formal}
\end{table*}

\begin{figure*}[t]
\centering
\includegraphics[width=.78\textwidth]
{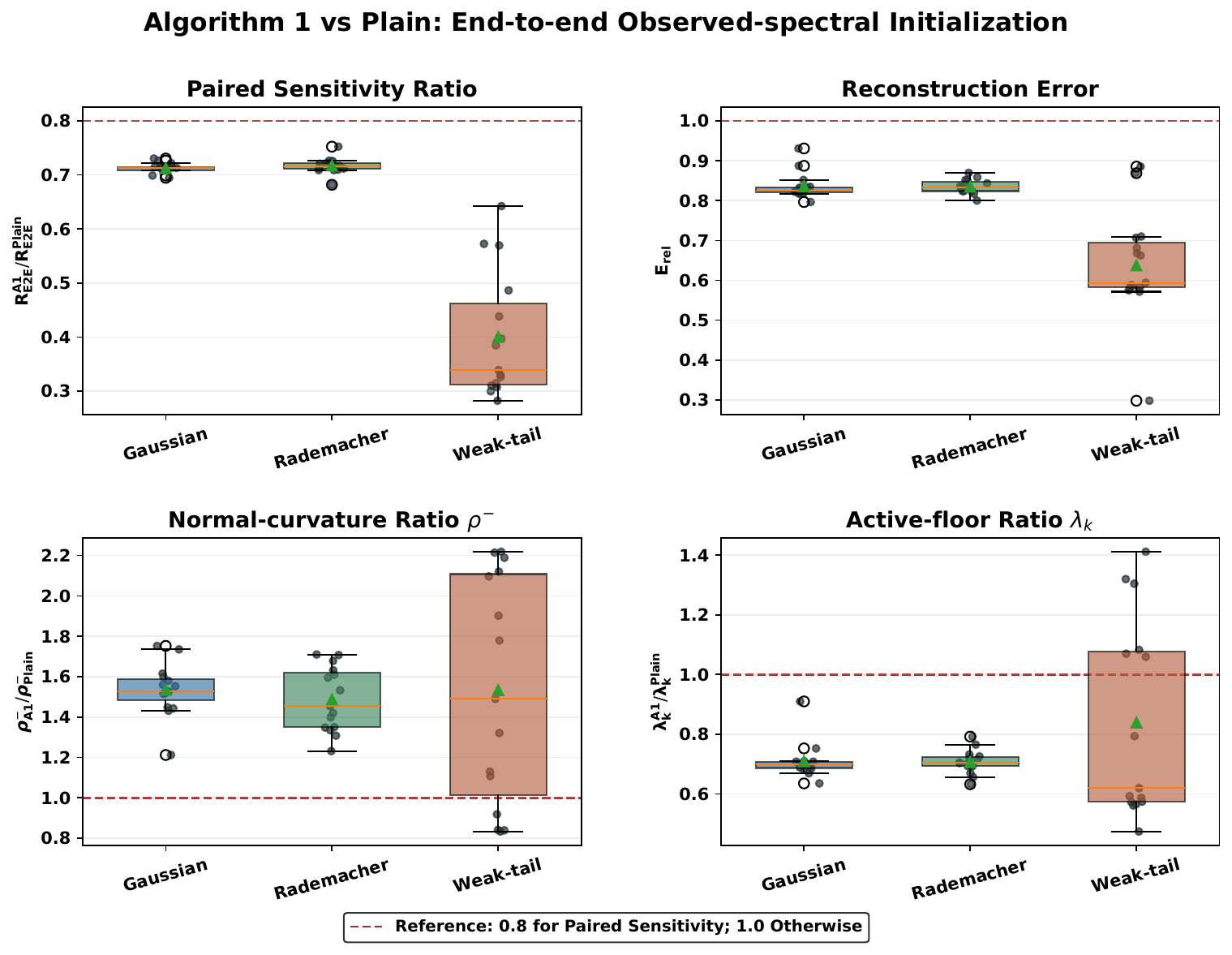}
\caption{Formal large-scale comparison of
Algorithm~\ref{alg:decayed-balanced-guidance} with Plain.
Each row contains the five seeds for one fixed
\((n,\mathrm{operator})\) cell.
The quantities \(\mathsf Q_{\rm E2E}\) and \(\mathsf Q_{\rm Sh}\)
are the median instancewise relative paired critical-point
sensitivities under E2E and Shared observed-spectral initialization.
Their corresponding sensitivity reductions are
\(\Delta_{\rm E2E}=100(1-\mathsf Q_{\rm E2E})\) and
\(\Delta_{\rm Sh}=100(1-\mathsf Q_{\rm Sh})\).
The reconstruction ratio \(\mathsf E_{\rm rel}\) and the terminal
\(\rho^-_{\rm Plain}\) and \(\rho^-_{\rm A1}\) values are common to
both protocols because their main noisy-data runs are identical.
Joint counts seeds for which both \(\mathsf Q<1\) and
\(\mathsf E_{\rm rel}<1\); every row has \(5/5\) under both protocols.
Pooling the three dimensions gives \(15/15\) joint improvements for
each operator and \(45/45\) per protocol.}
\label{fig:large-scale-e2e}
\end{figure*}

\begin{figure*}[t]
\centering
\includegraphics[width=.78\textwidth]
{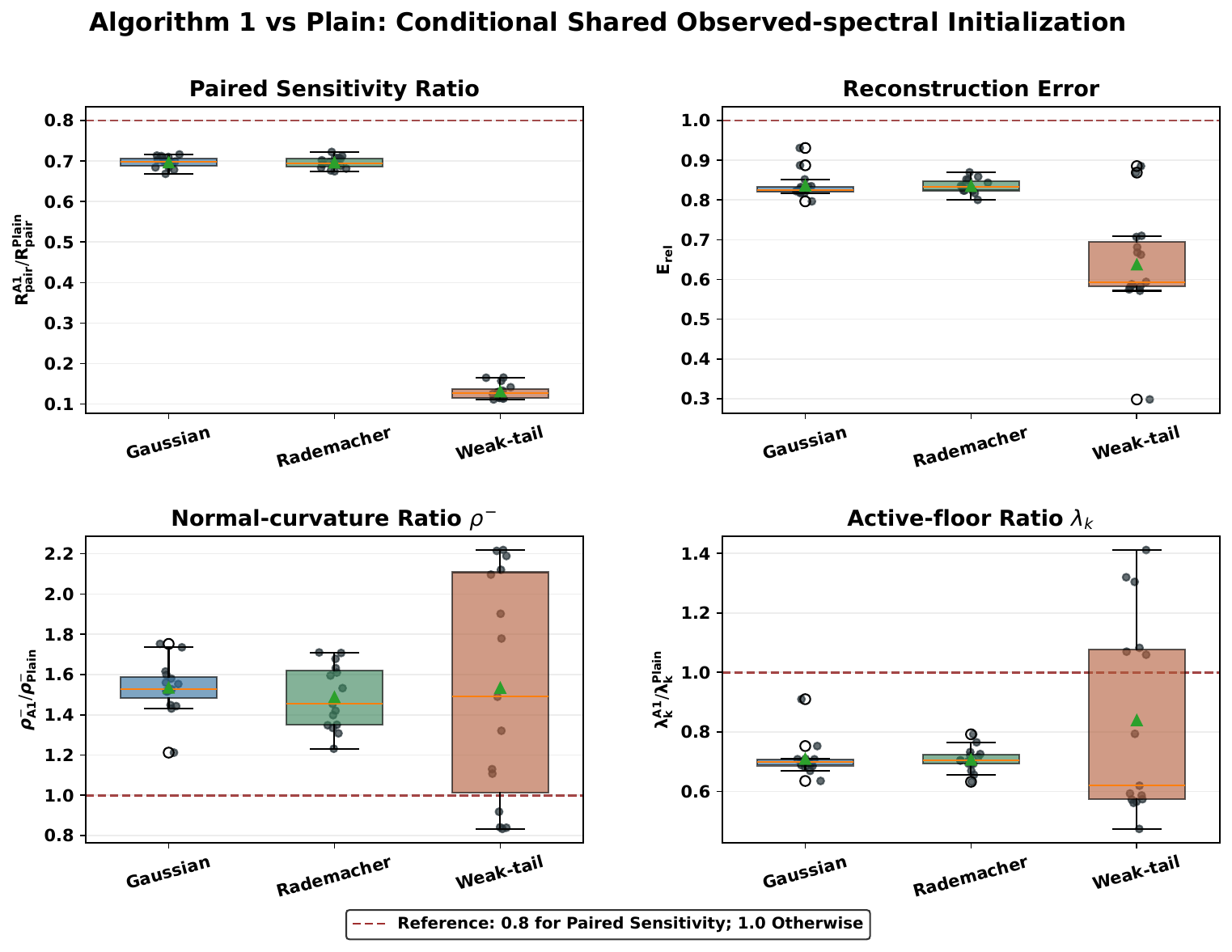}
\caption{Formal large-scale comparison under Shared observed-spectral
initialization. Only the paired-sensitivity panel changes relative to
Figure~\ref{fig:large-scale-e2e}; reconstruction and terminal spectral
diagnostics coincide because the two protocols use the same main
noisy-data solutions.}
\label{fig:large-scale-shared}
\end{figure*}

Table~\ref{tab:large-scale-formal} shows that the improvements persist
through \(n=1000\). Under E2E, the Gaussian and Rademacher median
\(\mathsf Q\) values remain between \(0.709\) and \(0.722\), giving
approximately \(28\%\)--\(29\%\) lower paired critical-point
sensitivity than Plain at every dimension. For weak-tail sensing,
\(\mathsf Q_{\rm E2E}\) increases from \(0.307\) at \(n=250\) to
\(0.570\) at \(n=1000\), but remains below one and therefore still
gives a \(43.0\%\) reduction at the largest scale.

Conditioning out initialization sensitivity produces smaller Shared
ratios. The Gaussian and Rademacher values decrease to approximately
\(0.684\) at \(n=1000\), while the weak-tail values remain between
\(0.120\) and \(0.157\). Shared is retained as a controlled diagnostic,
rather than replacing the primary E2E result.

Reconstruction also improves uniformly. The median
\(\mathsf E_{\rm rel}\) values lie between \(0.821\) and \(0.840\) for
the isotropic operators and between \(0.584\) and \(0.668\) for
weak-tail sensing. Consequently, both \(\mathsf Q<1\) and
\(\mathsf E_{\rm rel}<1\) hold in all \(45\) matched instances under
each protocol.

The lower panels of Figures~\ref{fig:large-scale-e2e} and
\ref{fig:large-scale-shared} provide complementary terminal spectral
diagnostics. The pooled median normal-curvature ratios are \(1.527\),
\(1.454\), and \(1.489\), while the corresponding active-floor ratios
are \(0.699\), \(0.704\), and \(0.619\) for Gaussian, Rademacher, and
weak-tail sensing. These terminal quantities do not vary monotonically
with the directly measured paired sensitivity and are therefore
reported as diagnostics rather than substituted for \(\mathsf Q\).

All main solutions remain full rank, with minimum
\(\lambda_k=0.0192\). The largest final relative gradient norm is
\(0.0249\), below the prescribed \(0.025\) tolerance, and every paired
solution used to compute \(\mathsf Q\) passes the same criticality
check.

\subsection{LoRA Fine-Tuning Experiments}
\label{app:lora-experiments}

\paragraph{From Matrix-Sensing Guidance to LoRA}

We evaluate whether the guidance principle developed in Algorithm~1 can
improve low-rank adaptation of a pretrained language model. Algorithm~1 is
motivated by the matrix-sensing robustness ratio
\[
    \rho_k^-=\frac{b^-}{\lambda_k(M)},
\]
which identifies two complementary mechanisms: protecting the active
spectral floor in the denominator and controlling leakage outside the active
subspace in the numerator. The LoRA experiment is a practical adaptation of
this principle rather than a direct application of the matrix-sensing
theorems, since LoRA updates are rectangular matrices and do not possess the
same positive-semidefinite manifold geometry.

For layer $\ell$, LoRA parameterizes the trainable update as
\[
    \Delta W_\ell
    =
    \frac{\alpha}{r}B_\ell A_\ell,
\]
where $r$ is the nominal LoRA rank. Let
$s_{\ell,1}\geq\cdots\geq s_{\ell,r}$ denote the singular values of
$\Delta W_\ell$. At the beginning of the guidance window, we freeze a
layer-specific active rank $q_\ell$ as the smallest rank explaining at least
$95\%$ of the spectral energy. Layers for which $q_\ell=r$ are excluded from
spectral guidance.

For the retained active directions, we use the scale-normalized inverse-trace
penalty
\[
\Omega_{\ell,\mathrm{floor}}
=
\left(
    \frac{1}{q_\ell}\sum_{i=1}^{q_\ell}s_{\ell,i}^{2}
    +\epsilon
\right)
\left(
    \frac{1}{q_\ell}\sum_{i=1}^{q_\ell}
    \frac{1}{s_{\ell,i}^{2}+\epsilon}
\right)-1.
\]
This term places stronger pressure on weak directions within the frozen
active spectrum. To avoid improving the active floor by simultaneously
inflating unused directions, we add the normalized tail penalty
\[
\Omega_{\ell,\mathrm{tail}}
=
\frac{
    \sum_{i=q_\ell+1}^{r}s_{\ell,i}^{2}
}{
    \operatorname{stopgrad}
    \left(\sum_{i=1}^{q_\ell}s_{\ell,i}^{2}\right)+\epsilon
}.
\]
The guided objective during the active window is
\[
\mathcal L_{\mathrm{guided}}
=
\mathcal L_{\mathrm{task}}
+
c_t
\left(
10^{-2}\,\overline{\Omega}_{\mathrm{floor}}
+
10^{-3}\,\overline{\Omega}_{\mathrm{tail}}
\right),
\]
where the bars denote averages over guided LoRA layers. Guidance begins after
$10\%$ of training, remains active for $20\%$ of the training trajectory, and
uses a short fade-in followed by cubic decay. Its contribution is clipped to
at most $10\%$ of the detached task loss.

Following the damped Gram preconditioning in Algorithm~1, scaled-active LoRA
also preconditions the factor gradients as
\begin{equation}
\begin{aligned}
\widetilde{\nabla}_{B_\ell}
&=
\nabla_{B_\ell}
\left(A_\ell A_\ell^\top+\xi I\right)^{-1},
\\
\widetilde{\nabla}_{A_\ell}
&=
\left(B_\ell^\top B_\ell+\xi I\right)^{-1}
\nabla_{A_\ell}.
\end{aligned}
\label{eq:scaled-lora-preconditioning}
\end{equation}
We use damping $\xi=0.1$, a damping floor of $10^{-4}$, and cap the
preconditioning gain at $10$. The normal-Rayleigh guidance term from
Algorithm~1 is disabled in this experiment, because a rectangular LoRA
update has no direct theorem-preserving analogue of the positive-semidefinite
normal projector. Consequently, scaled-active LoRA should be interpreted as
an active-floor and tail-guided method combined with damped LoRA
preconditioning.

The plain baseline uses the same initialization, data order, optimizer,
training seed, and objective perturbation, but applies neither spectral
guidance nor ScaledLoRA preconditioning. After the guidance window, the
spectral penalties are removed and the endpoint is polished on the original
training objective. In the reported legacy protocol, the scaled-active branch
retains its damped preconditioner during this polish.

\paragraph{Fixed Logit-Objective Perturbation}

To construct an analogue of matrix-sensing measurement noise, we perturb the
training objective rather than the validation inputs. For training example
$n$ and perturbation seed $s$, the perturbed objective is
\[
\mathcal L_{\varepsilon,s}(\theta)
=
\frac{1}{N}\sum_{n=1}^{N}
\operatorname{CE}
\left(
f_\theta(x_n)+\varepsilon u_{s,n},y_n
\right),
\]
where $u_{s,n}$ is generated deterministically from the pair consisting of
the perturbation seed and training-example identifier. It satisfies
\[
    \mathbf 1^\top u_{s,n}=0,
    \qquad
    \|u_{s,n}\|_2=1,
\]
and is fixed across all epochs. Labels are unchanged and no perturbation is
applied during validation. Therefore, noisy validation accuracy refers to the
clean-validation accuracy of an endpoint trained under the perturbed
objective; it is not accuracy under corrupted validation text.

The prespecified experiment uses $\varepsilon=0.25$ and three fixed
perturbation seeds, $10000$, $10001$, and $10002$. We additionally evaluate
$\varepsilon\in\{0.05,0.1,0.2,0.3,0.4,0.5\}$ without changing any other
training setting.

\paragraph{Experimental configuration.}
We fine-tune \texttt{roberta-base}
\citep{liu2019roberta} on the SST-2 task
\citep{socher2013recursive}, using the version distributed
in GLUE \citep{wang2018glue}. Following LoRA
\citep{hu2022lora}, adapters are inserted into all attention
query and value projections with rank \(r=8\), scaling
\(\alpha=16\), and zero dropout. The classification head is trainable in both
conditions. We use a maximum sequence length of $128$, training batch size
$32$, evaluation batch size $64$, and train for two epochs.

Both methods use SGD with learning rate $10^{-2}$, momentum $0.9$, zero
weight decay, and a learning-rate warm-up ratio of $0.06$. The legacy
protocol does not use a separate shared warm-up stage. We evaluate training
seeds $0$, $1$, and $2$. For each training seed, the plain and scaled-active
runs share their initialization and are evaluated using the same three fixed
objective-perturbation directions.

For aggregation, noisy validation accuracy is first averaged over the three
perturbation directions within each training seed. We then report the mean
and sample standard deviation over the three training seeds. The perturbation
directions are not treated as independent training repetitions.

\paragraph{Results}

At the prespecified perturbation magnitude $\varepsilon=0.25$, plain LoRA
achieves a clean validation accuracy of
$91.093\pm0.331\%$, whereas scaled-active LoRA achieves
$92.087\pm0.303\%$. This corresponds to a paired clean-accuracy improvement
of $0.994\pm0.066$ percentage points.

For endpoints trained under the fixed perturbed objective, plain LoRA
achieves $91.106\pm0.277\%$ validation accuracy and scaled-active LoRA
achieves $92.125\pm0.268\%$. The resulting paired noisy-accuracy improvement
is $1.019\pm0.022$ percentage points. Scaled-active outperforms plain LoRA
for each of the three training seeds.

Figure~\ref{fig:lora-epsilon-sensitivity} evaluates sensitivity to the
perturbation magnitude. Across the six tested magnitudes, the scaled-active
advantage remains between $0.943$ and $1.007$ percentage points. It is
positive for every training-seed and perturbation-magnitude combination.
These $18$ comparisons reuse the same three training seeds and perturbation
directions and are therefore not interpreted as $18$ independent
replications.

The two accuracy curves are approximately parallel. Thus, the experiment
supports a consistent improvement in absolute validation accuracy at the
perturbed training endpoint, but does not establish a smaller
clean-to-noisy accuracy degradation.

\begin{figure*}[t]
    \centering   \includegraphics[width=\textwidth]   {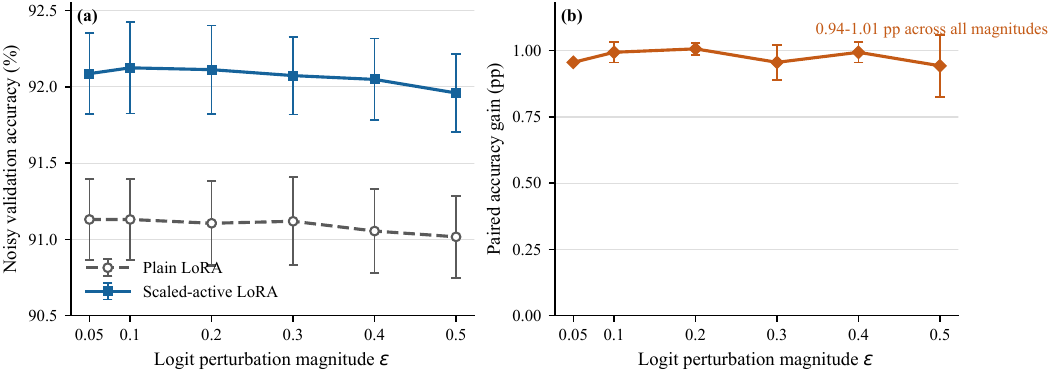}
    \caption{
    Robust validation performance under fixed unit-$\ell_2$
    logit-objective perturbations.
    Panel (a) reports noisy validation accuracy for plain and scaled-active
    LoRA. Panel (b) reports the paired scaled-active-minus-plain accuracy
    difference. Each point is the mean over three training seeds, after first
    averaging three fixed perturbation directions within each seed. Error bars
    denote sample standard deviations across training seeds.
    }
    \label{fig:lora-epsilon-sensitivity}
\end{figure*}

\section{Missing Details and Proofs}
\subsection{More Definitions}
\label{app:defs}

\begin{definition}[Second-order stationary point]
\label{def:SOSP}
A point \(X_w\in\RR^{n\times k}\) is a second-order stationary point
of \(f_w\) if
\[
\begin{gathered}
    \nabla f_w(X_w)=0,\\
    \nabla^2 f_w(X_w)[D,D]\ge0
    \quad\text{for every }D\in\RR^{n\times k}.
\end{gathered}
\]
We write \(M_w=X_wX_w^\top\).
\end{definition}

\begin{definition}[Restricted isometry]
\label{def:ri}
The operator \(\Aop\) satisfies \((\delta,s)\)-RIP if
\[
    (1-\delta)\norm{Z}_F^2
    \le \norm{\Aop(Z)}_2^2
    \le (1+\delta)\norm{Z}_F^2
\]
for every \(Z\in\SSym^n\) with \(\rank(Z)\le s\).
\end{definition}

\begin{assumption}[Unified \(2k\)-RIP]
\label{ass:unified-rip}
Unless stated otherwise, \(\Aop\) satisfies \((\delta,2k)\)-RIP for
a fixed \(\delta\in[0,1)\).  Since \(r_\star\le k\), this assumption
covers all rank-at-most-\(2k\) symmetric matrices used below, including
differences between rank-at-most-\(k\) PSD matrices.  We set
\(s_{\rm RIP}=2k\) throughout.

For \(M\succeq0\), let its nonzero eigenvalues be
\[
        \lambda_1(M)\ge\cdots\ge \lambda_{\rank(M)}(M)>0.
\]
When an effective active rank \(q\) is selected, \(P_q(M)\) denotes the
projector onto the top-\(q\) eigenspace and \(P_q^\perp=I-P_q\).  The tangent
space of the rank-\(q\) PSD manifold at a rank-\(q\) matrix with range
projector \(P\) is
\begin{equation}
    \label{eq:tangent-space}
        \Tcal(P)
        =
        \{\eta\in\SSym^n:\; P^\perp\eta P^\perp=0\}.
\end{equation}
Every \(\eta\in\Tcal(P)\) has rank at most \(2q\).
\end{assumption}

\subsection{Missing Details of Section~\ref{sec:full-rank-stability}}
\label{app:s4proofs}

\begin{lemma}[Factor Hessian identity]
\label{lem:hessian-identity}
Let \(M=XX^\top\).  For every \(D\in\RR^{n\times k}\),
\begin{equation}
    \label{eq:hessian-identity}
    \begin{aligned}
        \nabla^2 f_w(X)[D,D]
        &=
        \norm{\Aop(XD^\top+DX^\top)}_2^2 \\
        &\quad
        + 2\ip{G_w(M)}{DD^\top}.
    \end{aligned}
\end{equation}
Moreover, \(X\) is first-order critical iff
\begin{equation}
    \label{eq:foc-factor}
        G_w(M)X=0.
\end{equation}
\end{lemma}
\begin{proof}[Proof of Lemma~\ref{lem:hessian-identity}]

The first and second variations of \(XX^\top\) in direction \(D\) are
\[
       \eta = XD^\top+DX^\top,\qquad 2DD^\top .
\]
Applying the chain rule to \(F_w(XX^\top)\) gives the gradient
\(2G_w(M)X\) and the quadratic form \eqref{eq:hessian-identity}.
\end{proof}

\begin{remark}[Order of magnitude of chart constants]
\label{rem:chart-orders}
The raw chart is well conditioned when the active spectrum is bounded:
if \(0<c\le\lambda_k(M_w)\le\lambda_1(M_w)\le C\), then
\(\kappa_{\rm raw}(M_w)=O(1)\).  If
\(\lambda_k(M_w)=\varepsilon\ll1\) while \(\lambda_1(M_w)=O(1)\), then
\(\kappa_{\rm raw}(M_w)=\Theta(\varepsilon^{-1})\).  If
\(\lambda_1(M_w)=L\gg1\) and \(\lambda_k(M_w)=\Omega(1)\), then
\(\kappa_{\rm raw}(M_w)=\Theta(L)\).  
These scalings are coordinate geometry,
not RIP. RIP controls \(\Aop\) on tangent matrices, but it does not control the conditioning of a raw SVD/graph chart near the rank-deficient boundary.
\end{remark}

\begin{proposition}
\label{prop:active-margin}
Let \(X\in\RR^{n\times q}\) have rank \(q\), let \(M=XX^\top\), let
\(P\) be the range projector of \(M\), and assume \(G_w(M)X=0\).  For any
factor direction \(D\) producing the tangent matrix
\[
        \eta=XD^\top+DX^\top\in\Tcal(P),
\]
the factor Hessian obeys
\begin{equation}
        \nabla^2 f_w(X)[D,D]
        \ge
        \left(\alpha_T(P)-\rho_q^-(M;w)\right)\norm{\eta}_F^2 .
        \label{eq:active-margin}
\end{equation}
Consequently, under Assumption~\ref{ass:unified-rip},
\begin{equation}
        \nabla^2 f_w(X)[D,D]
        \ge
        \left(1-\delta-\rho_q^-(M;w)\right)\norm{\eta}_F^2 .
\end{equation}
\end{proposition}
\begin{proof}
Write \(M=U\Lambda U^\top\), where
\(\Lambda=\diag(\lambda_1,\ldots,\lambda_q)\) and
\(\lambda_q>0\).  Since \(XX^\top=M\) and \(X\) has full column rank, after a
right-orthogonal change of factor coordinates we may take
\(X=U\Lambda^{1/2}\).  Decompose a factor perturbation as
\[
        D=UA+U_\perp B .
\]
Then
\[
        \eta
        =
        U(\Lambda^{1/2}A^\top+A\Lambda^{1/2})U^\top
        +
        U\Lambda^{1/2}B^\top U_\perp^\top
        +
        U_\perp B\Lambda^{1/2}U^\top .
\]
The normal block of \(DD^\top\) is
\[
        P^\perp DD^\top P^\perp
        =
        U_\perp BB^\top U_\perp^\top .
\]
Since \(G_w(M)X=0\), the range and cross terms of \(G_w(M)\) vanish, so
\[
        \ip{G_w(M)}{DD^\top}
        =
        \ip{P^\perp G_w(M)P^\perp}{U_\perp BB^\top U_\perp^\top}.
\]
Let \(K=B\Lambda^{1/2}\).  The off-diagonal block of \(\eta\) is \(K\), hence
\[
\begin{aligned}
\norm{\eta}_F^2
&\ge
2\norm{K}_F^2,
\\
\tr(BB^\top)
&=
\tr(K\Lambda^{-1}K^\top)
\\
&\le
\frac{\norm{K}_F^2}{\lambda_q}
\le
\frac{\norm{\eta}_F^2}{2\lambda_q}.
\end{aligned}
\]
Therefore
\[
        2\ip{G_w(M)}{DD^\top}
        \ge
        -\frac{\norm{[P^\perp G_w(M)P^\perp]_-}_{\op}}{\lambda_q}
        \norm{\eta}_F^2 .
\]
The first term in \eqref{eq:hessian-identity} is at least
\(\alpha_T(P)\norm{\eta}_F^2\), proving \eqref{eq:active-margin}.  The RIP
specialization follows from \eqref{eq:tangent-rip}.
\end{proof}

\begin{assumption}[RIP-certified endpoint Hessian lower margin]
\label{ass:endpoint-margin}
Let \(M_w\in\Mcal_k\) satisfy \eqref{eq:endpoint-critical}, and set
\[
\begin{aligned}
    \mu_w
    &= (1-\delta)-\rho_k^-(M_w;w),
    \\
    \lambda_w
    &= \lambda_k(M_w),
    \qquad
    P_w
    = P_k(M_w),
    \\
    \gamma_w
    &= \norm{
        P_w^\perp G_w(M_w)P_w^\perp
    }_{\op}.
\end{aligned}
\]
Equivalently,
\begin{equation}
    \label{eq:rho-k-computable}
        \rho_k^-(M_w;w)
        =
        \frac{
        \max\{0,-\lambda_{\min}(P_w^\perp G_w(M_w)P_w^\perp)\}
        }{\lambda_k(M_w)} .
\end{equation}
We assume \(\mu_w>0\).
\end{assumption}

\begin{assumption}[Small noise relative to the endpoint scale]
\label{ass:small-noise}
With \(\mu_w,\lambda_w,\gamma_w\) from
Assumption~\ref{ass:endpoint-margin}, define the computable endpoint variation
scale
\begin{equation}
    \label{eq:endpoint-variation-scale}
        \mathfrak L_w
        =
        \frac{2(1+\delta+\rho_k^-(M_w;w))}{\lambda_w}
        +
        \frac{16\gamma_w}{\lambda_w^2},
\end{equation}
with the convention that \(\mu_w/(4\mathfrak L_w)=+\infty\) if
\(\mathfrak L_w=0\).  The noise obeys
\begin{equation}
    \label{eq:small-noise-condition}
        \frac{2\sqrt{1+\delta}}{\mu_w}\norm{w}_2
        \le
        \frac12 r_w^{\rm stab},
        \qquad
        r_w^{\rm stab}
        =
        \min\left\{
        \frac{\lambda_w}{4},
        \frac{\mu_w}{4\mathfrak L_w}
        \right\}.
\end{equation}
\end{assumption}

\paragraph{Why the margin cannot be dropped.}
If the endpoint Hessian has a flat eigen-direction, linear robustness can fail
even in one dimension.  Consider
\[
        \phi_0(x)=\frac14x^4,
        \qquad
        \phi_\varepsilon(x)=\frac14x^4-\varepsilon x .
\]
The unperturbed critical point is \(x_0=0\), but
\(\phi_0''(0)=0\).  The perturbed critical point solves
\[
        \phi_\varepsilon'(x)=x^3-\varepsilon=0,
        \qquad
        x_\varepsilon=\varepsilon^{1/3}.
\]
Thus \(\abs{x_\varepsilon-x_0}=\varepsilon^{1/3}\), which is not bounded by
\(C\varepsilon\) for any fixed \(C\) as \(\varepsilon\downarrow0\).  A positive
margin is exactly what rules out this flat-bottom behavior and lets the
criticality equation be inverted.

\begin{lemma}[Small-noise critical-point continuation]
\label{lem:small-noise-continuation}
Suppose Assumptions~\ref{ass:unified-rip}--\ref{ass:small-noise} hold.  Then
there exists a \(C^1\) curve
\[
        s\mapsto M_s\in\Mcal_k,\qquad 0\le s\le1,
\]
with \(M_1=M_w\), such that every \(M_s\) is a rank-\(k\) critical point of
\(F_s\):
\begin{equation}
    \label{eq:branch-criticality}
        \grad_{\Mcal_k}F_s(M_s)=0 .
\end{equation}
Moreover, along this branch,
\begin{align}
    \label{eq:endpoint-stable-hessian}
        \ip{\eta}{\Hess_{\Mcal_k}F_s(M_s)[\eta]}
        &\ge
        \frac{\mu_w}{2}\norm{\eta}_F^2,
        \qquad
        \eta\in T_{M_s}\Mcal_k,\\
    \label{eq:endpoint-stable-floor}
        \lambda_k(M_s)&\ge \frac{\lambda_w}{2}.
\end{align}
We call \(M_0:=M_{s=0}\) the paired clean endpoint.
\end{lemma}
\begin{proof}[Proof of Lemma ~\ref{lem:small-noise-continuation}]
Let \(\pi(X)=XX^\top\).  At every full-column-rank factor \(X\),
the differential
\[
    \mathrm{D}\pi_X[D]
    =
    XD^\top+DX^\top
\]
is surjective onto \(T_M\Mcal_k\), where \(M=\pi(X)\).  Its
kernel is the vertical space generated by the right-orthogonal
invariance:
\[
    \ker(\mathrm{D}\pi_X)
    =
    \{X\Omega:\Omega^\top=-\Omega\}.
\]
Thus, after removing these representation-induced vertical
directions, factor perturbations provide horizontal lifts of all
tangent directions on \(\Mcal_k\).  The factor--manifold Hessian
comparison of
\citet[Theorem~1 and Corollary~1]{luo2025nonconvex}
then transfers the curvature lower bound in
Proposition~\ref{prop:active-margin} to the Riemannian Hessian.
Assumption~\ref{ass:endpoint-margin} therefore implies that
\(\Hess_{\Mcal_k}F_w(M_w)\) is positive definite on
\(T_{M_w}\Mcal_k\).
The implicit-function theorem on the embedded manifold now yields
a local critical branch in the parameter \(s\)
\citep[Proposition~5.5.6]{absil2008optimization};
\citep[Proposition~5.45]{boumal2023introduction}.

 It remains to show that the branch cannot leave the endpoint neighborhood
before reaching \(s=0\).  Let \(P=P_k(M)\) and suppose
\(\norm{M-M_w}_F\le r\le\lambda_w/4\).  Weyl's inequality gives
\(\lambda_k(M)\ge3\lambda_w/4\), and Davis--Kahan gives
\(\norm{P-P_w}_{\op}\le 2r/\lambda_w\)
\citep{davisKahan1970rotation,bhatia1997matrix}.  Since
\(G_w(M_w)P_w=0\), the endpoint gradient is purely normal:
\[
        G_w(M_w)=P_w^\perp G_w(M_w)P_w^\perp .
\]
Using
\[
        G_s(M)
        =
        G_w(M_w)+H(M-M_w)+(1-s)N_w,
\]
RIP and Cauchy--Schwarz give, for every unit vector \(v\),
\[
\begin{aligned}
    \abs{v^\top H(M-M_w)v}
    &=
    \abs{\ip{\Aop(vv^\top)}
              {\Aop(M-M_w)}} \\
    &\le
    (1+\delta)\norm{M-M_w}_F .
\end{aligned}
\]
because \(vv^\top\) has rank one and \(M-M_w\) has rank at most \(2k\).  Also
\(\abs{v^\top N_wv}\le\sqrt{1+\delta}\norm{w}_2\).  Let
\[
        b_w=\norm{[P_w^\perp G_w(M_w)P_w^\perp]_-}_{\op}
        =
        \lambda_w\rho_k^-(M_w;w).
\]
Projector rotation changes the endpoint normal block by at most
\[
\begin{aligned}
    &\norm{P^\perp G_w(M_w)P^\perp
        -P_w^\perp G_w(M_w)P_w^\perp}_{\op}
    \\
    \le\;&
    2\gamma_w\norm{P-P_w}_{\op}
    \le
    \frac{4\gamma_w}{\lambda_w}\,r .
\end{aligned}
\]
Moreover \(\lambda_k(M)\ge\lambda_w-r\), so
\[
        \frac{\lambda_w}{\lambda_k(M)}
        \le
        1+\frac{2r}{\lambda_w},
        \qquad
        \frac{1}{\lambda_k(M)}
        \le
        \frac{2}{\lambda_w}
\]
for \(r\le\lambda_w/4\).  Combining the numerator perturbation with these
denominator bounds yields the local normal-ratio estimate
\begin{equation}
    \label{eq:local-rho-continuity}
        \rho_k^-(M;sw)
        \le
        \rho_k^-(M_w;w)
        +
        \mathfrak L_w r
        +
        \frac{2\sqrt{1+\delta}}{\lambda_w}(1-s)\norm{w}_2 .
\end{equation}
The constants in \(\mathfrak L_w\) are intentionally relaxed; the important
point is that they are endpoint quantities.

 By Assumption~\ref{ass:small-noise}, the two extra terms in
\eqref{eq:local-rho-continuity} are at most \(\mu_w/2\) whenever
\(r\le r_w^{\rm stab}\).  Hence, throughout this endpoint neighborhood,
\[
        (1-\delta)-\rho_k^-(M;sw)\ge\frac{\mu_w}{2},
\]
and Proposition~\ref{prop:active-margin} again gives
\eqref{eq:endpoint-stable-hessian}.

 On any interval where \eqref{eq:endpoint-stable-hessian} holds, differentiating
\eqref{eq:branch-criticality} gives
\[
        \Hess_{\Mcal_k}F_s(M_s)[\dot M_s]
        =
        \Proj_{T_{M_s}\Mcal_k}(\Aop^\ast w),
\]
and RIP gives
\[
        \norm{\Proj_{T_{M_s}\Mcal_k}(\Aop^\ast w)}_F
        \le
        \sqrt{1+\delta}\norm{w}_2 .
\]
Therefore
\begin{equation}
    \label{eq:branch-velocity-bound}
        \norm{\dot M_s}_F
        \le
        \frac{2\sqrt{1+\delta}}{\mu_w}\norm{w}_2 .
\end{equation}
Integrating backward from \(s=1\) gives, on any interval where the branch
exists inside the endpoint neighborhood,
\[
        \norm{M_s-M_w}_F
        \le
        \frac{2\sqrt{1+\delta}}{\mu_w}(1-s)\norm{w}_2
        \le \frac12 r_w^{\rm stab}.
\]
To formalize the continuation, let \(I\subset[0,1]\) be the maximal interval
containing \(1\) on which the implicit-function branch exists and remains in
\(\{M:\norm{M-M_w}_F<r_w^{\rm stab}\}\).  The local implicit-function theorem
makes \(I\) open relative to \([0,1]\).  The preceding estimate shows that any
point of \(I\) actually satisfies the stricter bound
\(\norm{M_s-M_w}_F\le r_w^{\rm stab}/2\); hence limits of branch points remain
inside the same neighborhood, the Hessian stays invertible by
\eqref{eq:endpoint-stable-hessian}, and the implicit-function theorem extends
the branch past the limit.  Thus \(I\) is also closed, so \(I=[0,1]\).
Finally, \eqref{eq:endpoint-stable-floor} follows from
\(r_w^{\rm stab}\le\lambda_w/4\).
\end{proof}

\begin{theorem}[Pointwise rank-stable perturbation bound]
\label{thm:rank-stable-bound}
Suppose Assumptions~\ref{ass:unified-rip}--\ref{ass:small-noise} hold, and
let \(M_0\) be the paired clean endpoint from
Lemma~\ref{lem:small-noise-continuation}.  Then
\begin{equation}
    \label{eq:rank-stable-rip}
        \norm{M_w-M_0}_F
        \le
        \frac{2\sqrt{1+\delta}}
        {(1-\delta)-\rho_k^-(M_w;w)}
        \norm{w}_2 .
\end{equation}
The associated rank-\(k\) spectral projector
\(P_s=P_k(M_s)\), with \(P_w=P_1\), rotates slowly:
\begin{equation}
\begin{aligned}
\norm{P_s-P_w}_{\op}
&\le {}
\\[-1mm]
&\quad
\frac{
    8\sqrt{1+\delta}\,(1-s)\norm{w}_2
}{
    \lambda_k(M_w)
    \bigl[(1-\delta)-\rho_k^-(M_w;w)\bigr]
}.
\end{aligned}
\label{eq:projector-rotation-bound}
\end{equation}
\end{theorem}

\begin{proof}[Proof of Theorem ~\ref{thm:rank-stable-bound}]
The rank-\(k\) manifold is equipped with the embedded Frobenius metric.  By
\eqref{eq:branch-criticality}, the Riemannian criticality equation is
\[
        \grad_{\Mcal_k}F_s(M_s)=0 .
\]
Differentiate this identity with respect to \(s\).  At a critical point, the
covariant derivative of the Riemannian gradient in the motion direction
\(\dot M_s\) is the Riemannian Hessian, while the explicit derivative of the
matrix gradient \(G_s(M)=H(M-M_\star)-sN_w\) with respect to \(s\) is
\(-N_w=-\Aop^\ast w\).  Hence

\[
\begin{aligned}
    0
    &=
    \frac{D}{ds}\grad_{\Mcal_k}F_s(M_s)
    \\
    &=
    \Hess_{\Mcal_k}F_s(M_s)[\dot M_s]
    -
    \Proj_{T_{M_s}\Mcal_k}(\Aop^\ast w).
\end{aligned}
\]

and therefore
\[
        \Hess_{\Mcal_k}F_s(M_s)[\dot M_s]
        =
        \Proj_{T_{M_s}\Mcal_k}(\Aop^\ast w).
\]
This is the intrinsic version of differentiating the pullback criticality
equation in local coordinates
\citep[Proposition~5.45]{boumal2023introduction};
\citep[Proposition~5.5.6]{absil2008optimization}.

The right-hand side is controlled by RIP.  Since every tangent matrix in
\(T_{M_s}\Mcal_k\) has rank at most \(2k\),
\[
\begin{aligned}
    \norm{
        \Proj_{T_{M_s}\Mcal_k}(\Aop^\ast w)
    }_F
    &=
    \sup_{
        \substack{
            \eta\in T_{M_s}\Mcal_k \\
            \norm{\eta}_F=1
        }
    }
    \ip{w}{\Aop(\eta)}
    \\
    &\le
    \sqrt{1+\delta}\norm{w}_2 .
\end{aligned}
\]
By the Hessian stability condition \eqref{eq:endpoint-stable-hessian},
\(\Hess_{\Mcal_k}F_s(M_s)\) is self-adjoint and positive definite on
\(T_{M_s}\Mcal_k\), with inverse norm at most \(2/\mu_w\).  Therefore
\[
        \norm{\dot M_s}_F
        \le
        \frac{2\sqrt{1+\delta}}{\mu_w}\norm{w}_2,
        \qquad
        \mu_w=(1-\delta)-\rho_k^-(M_w;w).
\]
Integrating from \(s=0\) to \(s=1\) proves \eqref{eq:rank-stable-rip}.

For the projector statement, the spectral gap between the \(k\)-th and
\((k+1)\)-st eigenvalues of \(M_s\) is \(\lambda_k(M_s)\), because
\(\rank(M_s)=k\).  The differential form of the Davis--Kahan/Kato spectral
projector bound gives the relaxed estimate
\citep{davisKahan1970rotation,bhatia1997matrix}
\[
        \norm{\dot P_s}_{\op}
        \le
        \frac{2\norm{\dot M_s}_{\op}}{\lambda_k(M_s)}
        \le
        \frac{8\sqrt{1+\delta}}{\lambda_w\mu_w}\norm{w}_2,
\]
where the last step uses \eqref{eq:endpoint-stable-floor}.  Integrating from
\(s\) to \(1\) gives \eqref{eq:projector-rotation-bound}.
\end{proof}

\subsubsection{Supporting Chart Diagnostics}
\label{ssec:Supporting-Chart-Diagnostics}

Theorem~\ref{thm:rank-stable-bound} is intrinsic: it controls
perturbations in the Frobenius geometry of \(\Mcal_k\) without
choosing local coordinates. We now express
the same argument in local coordinates centered at \(M_w\). The lemma below shows that its
conditioning is governed by the extreme active eigenvalues of
\(M_w\) and deteriorates as \(\lambda_k(M_w)\) approaches zero.

\begin{lemma}[Raw endpoint chart conditioning]
\label{lem:raw-chart-conditioning}
Let
\[
\begin{aligned}
    M_w
    &=
    U_w\Lambda_wU_w^\top,
    \\
    \Lambda_w
    &=
    \diag(\lambda_1,\ldots,\lambda_k),
    \\
    \lambda_1
    &\ge
    \cdots
    \ge
    \lambda_k
    >
    0 .
\end{aligned}
\]
and complete \(U_w\) to an orthogonal basis \([U_w\ U_{w,\perp}]\).  A raw
graph chart centered at \(M_w\) is
\[
\begin{aligned}
    \Phi(B,Z)
    &=
    Q(Z)BQ(Z)^\top,
    \\
    Q(Z)
    &=
    (U_w+U_{w,\perp}Z)
    (I+Z^\top Z)^{-1/2}.
\end{aligned}
\]
Here \(B\in\SSym_{++}^k\), \(Z\in\RR^{(n-k)\times k}\),
\(\dot B\in\SSym^k\), and \(\dot Z\in\RR^{(n-k)\times k}\).
At the base point \((B,Z)=(\Lambda_w,0)\), the differential is
\begin{equation}
\label{eq:raw-chart-differential}
\begin{aligned}
    D\Phi(\Lambda_w,0)[\dot B,\dot Z]
    &=
    U_w\dot B U_w^\top
    \\
    &\quad
    +
    U_{w,\perp}\dot Z\Lambda_wU_w^\top
    \\
    &\quad
    +
    U_w\Lambda_w\dot Z^\top U_{w,\perp}^\top .
\end{aligned}
\end{equation}
With the raw coordinate norm
\[
        \norm{(\dot B,\dot Z)}_{\rm raw}^2
        =
        \norm{\dot B}_F^2+\norm{\dot Z}_F^2,
\]
the singular values of \(D\Phi(\Lambda_w,0)\), viewed as a map into Frobenius
tangent matrices, lie between
\begin{equation}
\label{eq:raw-chart-singular-values}
\begin{aligned}
    \sigma_-
    &=
    \min\{1,\sqrt{2}\,\lambda_k(M_w)\},
    \\
    \sigma_+
    &=
    \max\{1,\sqrt{2}\,\lambda_1(M_w)\}.
\end{aligned}
\end{equation}
Thus the raw chart condition number is
\begin{equation}
    \label{eq:raw-chart-kappa}
        \kappa_{\rm raw}(M_w)
        =
        \frac{\sigma_+}{\sigma_-}
        =
        \frac{\max\{1,\sqrt2\,\lambda_1(M_w)\}}
        {\min\{1,\sqrt2\,\lambda_k(M_w)\}} .
\end{equation}
In raw coordinates, the same argument as
Theorem~\ref{thm:rank-stable-bound} gives
\[
        \norm{\dot M_s}_F
        \le
        2\kappa_{\rm raw}(M_w)^2
        \frac{\sqrt{1+\delta}}{\mu_w}\norm{w}_2
\]
at the base point.  In tangent-scaled coordinates \(K=Z\Lambda_w\), this
base-point distortion disappears.
\end{lemma}

\begin{proof}[Proof of Lemma~\ref{lem:raw-chart-conditioning}]
Let
\[
        W(Z)=U_w+U_{w,\perp}Z,
        \qquad
        R(Z)=I+Z^\top Z,
\]
so \(Q(Z)=W(Z)R(Z)^{-1/2}\).  For a perturbation \(\dot Z\),
\[
        DR(Z)[\dot Z]=\dot Z^\top Z+Z^\top\dot Z .
\]
At \(Z=0\), \(DR(0)[\dot Z]=0\), and hence the derivative of the
normalization factor \(R(Z)^{-1/2}\) vanishes.  Therefore
\[
        DQ(0)[\dot Z]=U_{w,\perp}\dot Z .
\]
Differentiating \(\Phi(B,Z)=Q(Z)BQ(Z)^\top\) gives
\[
        D\Phi(B,Z)[\dot B,\dot Z]
        =
        Q\dot BQ^\top+\dot Q BQ^\top+QB\dot Q^\top .
\]
Substituting \(Q(0)=U_w\), \(B=\Lambda_w\), and
\(\dot Q=U_{w,\perp}\dot Z\) proves
\eqref{eq:raw-chart-differential}.  The three blocks in
\eqref{eq:raw-chart-differential} are Frobenius-orthogonal, so
\[
        \norm{D\Phi(\Lambda_w,0)[\dot B,\dot Z]}_F^2
        =
        \norm{\dot B}_F^2
        +
        2\norm{\dot Z\Lambda_w}_F^2 .
\]
Since
\[
        \lambda_k(M_w)\norm{\dot Z}_F
        \le
        \norm{\dot Z\Lambda_w}_F
        \le
        \lambda_1(M_w)\norm{\dot Z}_F,
\]
the singular-value bounds \eqref{eq:raw-chart-singular-values} follow.  The
displayed coordinate velocity bound follows from the coordinate version of the
proof: the pulled-back noise gradient is bounded by
\(\sqrt{1+\delta}\sigma_+\norm{w}_2\), the pulled-back Hessian is lower bounded
by \((\mu_w/2)\sigma_-^2\), and matrix velocity conversion costs another
\(\sigma_+\).
\end{proof}

\subsection{Missing Details of Section~\ref{sec:rank-cases}}
\label{s5proofs}

\begin{proposition}[Unused columns force PSD matrix gradient]
\label{prop:rank-def-sosp}
Let \(X\in\RR^{n\times k}\) be a second-order point of \(f_w\).  If
\(\rank(X)<k\), then
\begin{equation}
    \label{eq:rank-def-gradient-psd}
        G_w(XX^\top)\succeq0 .
\end{equation}
\end{proposition}

\begin{proof}[Proof of Proposition~\ref{prop:rank-def-sosp}]
Choose a unit \(v\in\RR^k\) with \(Xv=0\).  For any \(u\in\RR^n\), set
\(D=uv^\top\).  Then \(XD^\top+DX^\top=0\) and \(DD^\top=uu^\top\).  The
second-order condition and Lemma~\ref{lem:hessian-identity} give
\[
        0\le \nabla^2 f_w(X)[D,D]=2u^\top G_w(XX^\top)u .
\]
Since \(u\) is arbitrary, \(G_w(XX^\top)\succeq0\).
\end{proof}

\begin{theorem}[Rank-deficient points are globally optimal in matrix space]
\label{thm:rank-def-global}
Let \(M=X X^\top\) with \(\rank(X)<k\), and suppose \(X\) is a second-order
point of \(f_w\).  Then \(M\) is a global minimizer of the convex PSD problem
\begin{equation}
    \label{eq:convex-psd-problem}
        \min_{Z\succeq0} F_w(Z).
\end{equation}
Consequently:
\begin{enumerate}
    \item If \(w=0\), then \(M=M_\star\).
    \item For general \(w\), if \(\rank(M)\le k\), then
    \begin{equation}
        \label{eq:rank-def-noisy-bound}
        \norm{M-M_\star}_F
        \le
        \frac{\norm{w}_2}{\sqrt{1-\delta}}.
    \end{equation}
\end{enumerate}
\end{theorem}
\begin{proof}[Proof of Theorem~\ref{thm:rank-def-global}]
By first-order criticality, \(G_w(M)X=0\), hence \(G_w(M)M=0\) and
\(\ip{G_w(M)}{M}=0\).  Proposition~\ref{prop:rank-def-sosp} gives
\(G_w(M)\succeq0\).  Therefore, for every \(Z\succeq0\),
\[
        \ip{G_w(M)}{Z-M}
        =
        \ip{G_w(M)}{Z}
        \ge0.
\]
This is the first-order optimality condition for the convex problem
\eqref{eq:convex-psd-problem}, so \(M\) is a global PSD minimizer.

If \(w=0\), \(M_\star\) is feasible and has objective value zero.  Hence
\(\Aop(M-M_\star)=0\).  Assumption~\ref{ass:unified-rip} implies
\(M=M_\star\), since \(\rank(M-M_\star)\le k+r_\star\le2k\).

For general \(w\), use the convex first-order inequality with \(Z=M_\star\).
Writing \(\Delta=M-M_\star\),
\[
        0
        \le
        \ip{G_w(M)}{M_\star-M}
        =
        \ip{\Aop\Delta-w}{-\Aop\Delta}.
\]
Thus \(\norm{\Aop\Delta}_2^2\le \ip{w}{\Aop\Delta}\le
\norm{w}_2\norm{\Aop\Delta}_2\), so
\(\norm{\Aop\Delta}_2\le\norm{w}_2\).  Since
\(\rank(\Delta)\le k+r_\star\le2k\), RIP gives
\(\sqrt{1-\delta}\norm{\Delta}_F\le\norm{\Aop\Delta}_2\).
\end{proof}

\begin{remark}[Relation to Boumal--Voroninski--Bandeira]
This result is the PSD matrix-sensing specialization of the standard
rank-deficiency certificate for Burer--Monteiro factorizations
\citep{burer2003nonlinear,boumal2016non,boumal2020deterministic};
the RIP structure further yields exact recovery in the noiseless case and
a noise-linear recovery bound in the noisy case.
\end{remark}

\subsection{Full-Rank Points and the Spectral Tail}
We now fully discuss the last part of our main theorem. The exact rank-deficiency certificate is no longer available at a full-rank
point, but the negative spectral part of the normal block can remain small
when the weakest slack direction is small relative to the active spectral
floor. Let

\[
    M=XX^\top=\sum_{i=1}^k\lambda_i u_iu_i^\top,
    \qquad
    \lambda_1\ge\cdots\ge\lambda_k>0,
\]

and choose an effective active rank \(q\le k\) with
\(P_q=\sum_{i=1}^q u_iu_i^\top\). A negative normal direction can be paired
with the weakest active direction. The second-order condition and the unified
\(2k\)-RIP then limit its negative spectral part by

\begin{equation}
    \label{eq:tail-normal-bound}
    \norm{[P_q^\perp G_w(M)P_q^\perp]_-}_{\op}
    \le
    (1+\delta)\lambda_k(M).
\end{equation}

After normalization by the selected active spectral floor, this becomes

\begin{equation}
    \label{eq:rho-tail-ratio}
    \rho_q^-(M;w)
    \le
    (1+\delta)\frac{\lambda_k(M)}{\lambda_q(M)}.
\end{equation}

\begin{proposition}[Negative normal curvature is limited by the thinnest slack direction]
\label{prop:full-rank-tail-normal}
Let \(X\in\RR^{n\times k}\) be a full-rank second-order point of \(f_w\), let
\[
        M=XX^\top=\sum_{i=1}^k\lambda_i u_iu_i^\top,
        \qquad
        \lambda_1\ge\cdots\ge\lambda_k>0,
\]
and choose an effective active rank \(q\le k\) with
\(P_q=\sum_{i=1}^q u_iu_i^\top\).  Under
Assumption~\ref{ass:unified-rip},
\begin{equation}
    \label{eq:tail-normal-bound}
        \norm{[P_q^\perp G_w(M)P_q^\perp]_-}_{\op}
        \le
        (1+\delta)\lambda_k(M).
\end{equation}
Consequently,
\begin{equation}
    \label{eq:rho-tail-ratio}
        \rho_q^-(M;w)
        \le
        (1+\delta)
        \frac{\lambda_k(M)}{\lambda_q(M)} .
\end{equation}
\end{proposition}
\begin{proof}[Proof of Proposition~\ref{prop:full-rank-tail-normal}]

The first-order condition \(G_w(M)X=0\) implies that \(G_w(M)\) vanishes on
\(\operatorname{range}(M)\).  Hence a negative eigenvector of
\(P_q^\perp G_w(M)P_q^\perp\) lies in the orthogonal complement of
\(\operatorname{range}(M)\).  Let \(u\) be such a unit vector and let
\(u_k\) be an eigenvector of \(M\) associated with \(\lambda_k(M)\).
Choose factor coordinates so that \(X e_k=\sqrt{\lambda_k}\,u_k\), and set
\[
        D=\sqrt{\lambda_k}\,u e_k^\top .
\]
Then
\[
\begin{aligned}
    XD^\top+DX^\top
    &=
    \lambda_k(u_ku^\top+uu_k^\top),
    \\
    \norm{XD^\top+DX^\top}_F^2
    &=
    2\lambda_k^2,
    \\
    DD^\top
    &=
    \lambda_k uu^\top .
\end{aligned}
\]
If
\[
        u^\top G_w(M)u=-\gamma<0,
\]
the second-order condition and Lemma~\ref{lem:hessian-identity} give
\[
        0
        \le
        \nabla^2f_w(X)[D,D]
        =
        \norm{\Aop(XD^\top+DX^\top)}_2^2
        -
        2\gamma\lambda_k .
\]
The matrix \(XD^\top+DX^\top\) has rank at most two, so RIP gives
\[
        \norm{\Aop(XD^\top+DX^\top)}_2^2
        \le
        (1+\delta)2\lambda_k^2 .
\]
Thus \(\gamma\le(1+\delta)\lambda_k\), proving
\eqref{eq:tail-normal-bound}.  Dividing by \(\lambda_q(M)\) gives
\eqref{eq:rho-tail-ratio}.
\end{proof}

For \(q<k\), the spectral ratio \(\lambda_k(M)/\lambda_q(M)\) is therefore an
effective-rank diagnostic for normalized negative normal curvature. It
controls the weakest slack direction, but not the total discarded tail energy.

\begin{corollary}[Effective-rank normal-curvature surrogate]
\label{cor:pointwise-spectral}
If \(M_w\) is produced by a full-column-rank second-order point of width \(k\),
then for any effective active rank \(q<k\),
\begin{equation}
    \label{eq:effective-rank-surrogate}
        \rho_q^-(M_w;w)
        \le
        (1+\delta)\frac{\lambda_k(M_w)}{\lambda_q(M_w)} .
\end{equation}
This is not the main strict rank-stable theorem, which uses \(q=k\) and the
directly computed \(\rho_k^-(M_w;w)\).  It is an effective-rank diagnostic:
when the spectral tail ratio \(\lambda_k(M_w)/\lambda_q(M_w)\) is small, the
normal curvature seen by the active rank-\(q\) block is small, but the discarded
tail must still be accounted for separately if one truncates to rank \(q\).
In particular, \(\lambda_k/\lambda_q\) is the thinnest-tail ratio; it does not
by itself control the whole tail energy.
\end{corollary}

\begin{lemma}[Deterministic spectral-tail bounds]
\label{lem:spectral-tail}
Let \(M\succeq0\) have rank at most \(k\), eigenvalues
\(\lambda_1\ge\cdots\ge\lambda_k\ge0\), and fix
\[
        1\le q<k,\qquad \lambda_q(M)>0.
\]
Let
\[
        M_q=\sum_{i=1}^q \lambda_i u_iu_i^\top
\]
be its best rank-\(q\) approximation.  Then
\begin{align}
    \label{eq:eckart-tail}
        \norm{M-M_q}_{\op}
        &=
        \lambda_{q+1}(M),\\
        \norm{M-M_q}_F^2
        &=
        \sum_{i>q}\lambda_i(M)^2 .
\end{align}
If \(d=k-q\), then
\begin{equation}
    \label{eq:tail-by-leading-slack}
        \frac{\norm{M-M_q}_F}{\lambda_q(M)}
        \le
        \sqrt d\,\frac{\lambda_{q+1}(M)}{\lambda_q(M)} .
\end{equation}
Also,
\begin{equation}
    \label{eq:lambda-k-tail}
        \frac{\lambda_k(M)}{\lambda_q(M)}
        \le
        \frac{\norm{M-M_q}_F}{\sqrt d\,\lambda_q(M)} .
\end{equation}
If, in addition, \(\bar M\succeq0\) has rank \(q\) and
\(\norm{M-\bar M}_{\op}<\lambda_q(\bar M)\), then Weyl's inequality gives
\begin{equation}
    \label{eq:weyl-ratio}
        \frac{\lambda_k(M)}{\lambda_q(M)}
        \le
        \frac{\norm{M-\bar M}_{\op}}
        {\lambda_q(\bar M)-\norm{M-\bar M}_{\op}} .
\end{equation}
\end{lemma}
\begin{proof}[Proof of Lemma~\ref{lem:spectral-tail}]

The identities \eqref{eq:eckart-tail} are the PSD form of the
Eckart--Young--Mirsky theorem \citep{eckartYoung1936lowerRank}.  Since the \(d\) tail eigenvalues
\(\lambda_{q+1},\ldots,\lambda_k\) are all at most \(\lambda_{q+1}\),
\[
        \norm{M-M_q}_F^2=\sum_{i>q}\lambda_i^2\le d\,\lambda_{q+1}^2,
\]
which proves \eqref{eq:tail-by-leading-slack}.  Since the same tail eigenvalues
are all at least \(\lambda_k\),
\[
        d\,\lambda_k^2\le \sum_{i>q}\lambda_i^2=\norm{M-M_q}_F^2,
\]
which proves \eqref{eq:lambda-k-tail}.  If \(M=\bar M+E\) and
\(\rank(\bar M)=q\), Weyl's inequality yields
\[
        \lambda_q(M)\ge \lambda_q(\bar M)-\norm{E}_{\op},
        \qquad
        \lambda_k(M)\le \norm{E}_{\op},
\]
because \(\lambda_k(\bar M)=0\) for \(k>q\).  This proves
\eqref{eq:weyl-ratio}.
    
\end{proof}

For direct recovery relative to the rank-\(r_\star\) target \(M_\star\), the
relevant quantity is the numerator of \(\rho_k^-(M_w;w)\), rather than its
normalization by the active spectral floor. At a full-rank first-order critical
point, \(G_w(M_w)X_w=0\) implies

\[
    G_w(M_w)=P_w^\perp G_w(M_w)P_w^\perp.
\]

We therefore define the absolute negative normal leakage by

\begin{equation}
    \label{eq:absolute-normal-leakage}
    b_w^-
    :=
    \norm{[P_w^\perp G_w(M_w)P_w^\perp]_-}_{\op}
    =
    \lambda_k(M_w)\rho_k^-(M_w;w).
\end{equation}

\begin{lemma}[Absolute normal leakage]
\label{lem:absolute-normal-leakage}
Let \(X_w\in\RR^{n\times k}\) be full column rank and first-order critical for
\(f_w\), let \(M_w=X_wX_w^\top\), and let \(P_w=P_k(M_w)\).  Then
\[
        G_w(M_w)=P_w^\perp G_w(M_w)P_w^\perp .
\]
Define the negative normal leakage amplitude
\begin{equation}
    \label{eq:absolute-normal-leakage0}
        b_w^-
        =
        \norm{[P_w^\perp G_w(M_w)P_w^\perp]_-}_{\op}.
\end{equation}
Equivalently,
\begin{equation}
    \label{eq:lambda-rho-leakage}
        b_w^-=\lambda_k(M_w)\rho_k^-(M_w;w).
\end{equation}
\end{lemma}
\begin{proof}[Proof of Lemma~\ref{lem:absolute-normal-leakage}]
First-order criticality gives \(G_w(M_w)X_w=0\).  Since \(X_w\) spans
\(\operatorname{range}(M_w)\), this is equivalent to \(G_w(M_w)P_w=0\).  By
symmetry also \(P_wG_w(M_w)=0\), so \(G_w(M_w)\) is supported on the normal
block.  Equation~\eqref{eq:lambda-rho-leakage} is the definition of
\(\rho_k^-\) in \eqref{eq:normal-ratios} with \(q=k\).
\end{proof}

Lemma~\ref{lem:absolute-normal-leakage} isolates the adverse part of the normal block at a
full-rank critical point: \(b_w^-\) measures its absolute spectral
amplitude, while \(\rho_k^-\) is the corresponding normalization by
the weakest active eigenvalue. To convert this local quantity into a recovery estimate, we must
also control how much of the target matrix lies outside
\(\operatorname{range}(M_w)\).The next lemma bounds this
target-subspace leakage directly by the reconstruction error.

\begin{lemma}[Target subspace leakage]
\label{lem:target-subspace-leakage}
Let \(M_\star\succeq0\) have rank \(1\le r_\star\le k\), let
\[
        \lambda_\star=\lambda_{r_\star}(M_\star)>0,
\]
and let \(M_w\succeq0\) have range projector \(P_w\).  Then
\begin{equation}
        \tr(P_w^\perp M_\star P_w^\perp)
        \le
        \frac{1}{\lambda_\star}
        \norm{M_w-M_\star}_F^2 .
        \label{eq:target-subspace-leakage}
\end{equation}
\end{lemma}
\begin{proof}[Proof of Lemma~\ref{lem:target-subspace-leakage}]

Write
\[
        M_\star=U_\star\Lambda_\star U_\star^\top
\]
and set \(E=M_w-M_\star\).  Since \(P_w^\perp M_w=0\),
\[
\begin{aligned}
        P_w^\perp U_\star\Lambda_\star^{1/2}
        &=
        P_w^\perp M_\star U_\star\Lambda_\star^{-1/2} \\
        &=
        -P_w^\perp E U_\star\Lambda_\star^{-1/2}.
\end{aligned}
\]
Therefore,
\[
\begin{aligned}
        \tr(P_w^\perp M_\star P_w^\perp)
        &=
        \norm{P_w^\perp U_\star\Lambda_\star^{1/2}}_F^2 \\
        &=
        \norm{P_w^\perp E U_\star\Lambda_\star^{-1/2}}_F^2 \\
        &\le
        \norm{\Lambda_\star^{-1/2}}_{\op}^2
        \norm{E}_F^2 \\
        &=
        \frac{1}{\lambda_\star}
        \norm{M_w-M_\star}_F^2,
\end{aligned}
\]
which proves \eqref{eq:target-subspace-leakage}.

\end{proof}

Lemmas~\ref{lem:absolute-normal-leakage} and~\ref{lem:target-subspace-leakage} now connect the normal geometry to recovery. Writing
\(E=M_w-M_\star\) and
\(\lambda_\star=\lambda_{r_\star}(M_\star)\), first-order criticality
gives
\[
\ip{G_w(M_w)}{E}
=
-\ip{G_w(M_w)}{M_\star}
\leq
\frac{b_w^-}{\lambda_\star}\norm{E}_F^2 .
\]
Meanwhile, RIP lower-bounds the data-fit contribution by
\((1-\delta)\norm{E}_F^2\), up to the measurement-noise term. Hence
the normal leakage can be absorbed whenever
\((1-\delta)-b_w^-/\lambda_\star>0\). The following lemma formalizes
the resulting recovery bound.

\begin{lemma}[Normal leakage recovery bound]
\label{lem:normal-leakage-recovery}
Assume the unified \(2k\)-RIP in Assumption~\ref{ass:unified-rip}, and let
\(M_\star\succeq0\) have rank \(1\le r_\star\le k\).  Let \(X_w\) be a
full-column-rank first-order critical point of \(f_w\), and let
\(M_w=X_wX_w^\top\).  Define \(b_w^-\) by
\eqref{eq:absolute-normal-leakage}, set
\begin{equation}
    \label{eq:mu-leak-definition}
        \mu_{\rm leak}
        =
        1-\delta
        -
        \frac{b_w^-}{\lambda_{r_\star}(M_\star)},
\end{equation}
and assume \(\mu_{\rm leak}>0\).  Then
\begin{equation}
    \label{eq:normal-leakage-recovery}
        \norm{M_w-M_\star}_F
        \le
        \frac{\sqrt{1+\delta}}{\mu_{\rm leak}}
        \norm{w}_2 .
\end{equation}
\end{lemma}
\begin{proof}[Proof of Lemma~\ref{lem:normal-leakage-recovery}]

Let \(E=M_w-M_\star\), \(G=G_w(M_w)\), and
\(\lambda_\star=\lambda_{r_\star}(M_\star)\).  By first-order criticality,
\(GX_w=0\), and hence \(GM_w=0\).  Therefore,
\[
        \ip{G}{E}
        =
        -\ip{G}{M_\star}.
\]
Since \(G\) is supported on \(P_w^\perp\),
Lemmas~\ref{lem:absolute-normal-leakage}
and~\ref{lem:target-subspace-leakage} give
\[
\begin{aligned}
        -\ip{G}{M_\star}
        &\le
        b_w^-\,\tr(P_w^\perp M_\star P_w^\perp) \\
        &\le
        \frac{b_w^-}{\lambda_\star}
        \norm{E}_F^2 .
\end{aligned}
\]
On the other hand,
\[
        \ip{G}{E}
        =
        \norm{\Aop(E)}_2^2-\ip{w}{\Aop(E)}.
\]
Since \(\rank(E)\le k+r_\star\le2k\), RIP and Cauchy--Schwarz imply
\[
        (1-\delta)\norm{E}_F^2
        \le
        \sqrt{1+\delta}\norm{w}_2\norm{E}_F
        +
        \frac{b_w^-}{\lambda_\star}\norm{E}_F^2 .
\]
Rearranging and using \eqref{eq:mu-leak-definition} gives
\[
        \mu_{\rm leak}\norm{E}_F^2
        \le
        \sqrt{1+\delta}\norm{w}_2\norm{E}_F,
\]
which proves \eqref{eq:normal-leakage-recovery}.

\end{proof}

Another ingredient controls the portion of the lower-rank target that lies
outside \(\operatorname{range}(M_w)\). This target-subspace leakage is controlled by
the current recovery error. Combining that estimate with \(b_w^-\) and the
unified \(2k\)-RIP gives the direct full-rank recovery certificate below. 

\begin{proof}[Proof of Theorem~\ref{thm:three-cases}]

Part (1) is Theorem~\ref{thm:rank-def-global}.  Part (2) is
Theorem~\ref{thm:rank-stable-bound}, whose strict full-rank case uses active
rank \(k\).  Part (3) is Lemma~\ref{lem:absolute-normal-leakage} followed by
Lemma~\ref{lem:normal-leakage-recovery}.

\end{proof}

\section{Trajectory Case Study: Full-Rank Shadowing from Small Initialization}
\label{sec:trajectory}

This section studies robustness along the whole GD path rather than at one
critical endpoint.  Therefore the clean reference is not a single matrix
\(M_0\).  It is the time-indexed noiseless trajectory
\[
        M_t^0=X_t^0X_t^{0\top},
\]
run from the same initialization as the noisy trajectory
\[
        M_t^w=X_t^wX_t^{w\top}.
\]
The trajectory robustness question is
\[
        \Delta_t=\norm{M_t^w-M_t^0}_F .
\]
Throughout this section we focus on the full-rank-to-full-rank regime with
\(k\le n\): for the time window under discussion, both \(X_t^0\) and \(X_t^w\)
remain full column rank, so \(M_t^0\) and \(M_t^w\) lie on the rank-\(k\) PSD
manifold.  The role of the normal-curvature ratio is now dynamic.  A small
\(\rho_{k,t}^-\), defined in \eqref{eq:trajectory-rho-k}, together with the
local Hessian check below, indicates that the rank-\(k\) tangent dynamics
around \(M_t^w\) are well conditioned, so measurement noise cannot separate the
noisy and noiseless trajectories too much during the next steps.

\subsection{Clean and Noisy GD Paths}

Let \(N_0=0\) and \(N_w=\Aop^\ast w\).  Starting from the same small random
initialization \(X_0^0=X_0^w=\alpha G\), define
\begin{equation}
\label{eq:paired-gd-update}
\begin{aligned}
    X_{t+1}^\zeta
    &=
    \left(
        I+\eta[H(M_\star-M_t^\zeta)+N_\zeta]
    \right)X_t^\zeta,
    \\
    M_t^\zeta
    &=
    X_t^\zeta (X_t^\zeta)^\top,
    \qquad
    \zeta\in\{0,w\}.
\end{aligned}
\end{equation}
The clean recursion is the one studied by
\citet{stoegerSoltanolkotabi2021smallInit}, up to normalization constants.  In
their notation the target is \(M_\star=X_\star X_\star^\top\), the factor
condition number is
\[
        \kappa_X=\frac{\norm{X_\star}_{\op}}{\smin(X_\star)}
        =
        \sqrt{\frac{\lambda_1(M_\star)}{\lambda_{r_\star}(M_\star)}} ,
\]
the search width is \(k>r_\star\), \(X_0=\alpha G\) with entries of \(G\)
distributed as \(N(0,1/\sqrt{k})\), and the step size satisfies
\(\eta\lesssim \kappa_X^{-4}\norm{X_\star}_{\op}^{-2}\).
Equivalently, if \(\kappa_M=\lambda_1(M_\star)/\lambda_{r_\star}(M_\star)\),
then the step size is of order
\(\kappa_M^{-2}\lambda_1(M_\star)^{-1}\).  Their Theorem 3.3 assumes
\(k>r_\star\), splits into the regimes \(k\ge2r_\star\) and
\(r_\star<k<2r_\star\), and requires rank-\((2r_\star+1)\) RIP with constant of
order at most \(1/(\kappa_X^4\sqrt{r_\star})\), together with explicit small
upper bounds on \(\alpha\).  It proves a finite-time early-stopped proximity
bound for the clean iterates.  In the \(k\ge2r_\star\) case, the normalized
error after the prescribed number of iterations is bounded by a polynomial
factor times
\((\alpha/\norm{X_\star}_{\op})^{21/16}\).  The exactly parameterized case
\(k=r_\star\) is treated separately in their Theorem 3.4.

The clean result is used here as a reference path.  The new object is the
shadowing error \(\Delta_t\) between \eqref{eq:paired-gd-update} with and
without \(w\).  No critical-point comparison is needed.

\subsection{Rank-$k$ Rho and a Shadowing Bound}

For a full-rank noisy iterate \(M_t^w\), let \(P_t^w=P_k(M_t^w)\).  In this
section
\[
        G_w(M)=H(M-M_\star)-N_w
\]
is the noisy matrix gradient.  Define
\begin{equation}
\label{eq:trajectory-rho-k}
\begin{aligned}
    \rho_{k,t}^-
    &=
    \frac{
        \norm{
            \bigl[
                P_t^{w,\perp}
                G_w(M_t^w)
                P_t^{w,\perp}
            \bigr]_-
        }_{\op}
    }{
        \lambda_k(M_t^w)
    },
    \\
    \mu_t^\rho
    &=
    (1-\delta)-\rho_{k,t}^- .
\end{aligned}
\end{equation}
This is the same full-rank denominator as in the endpoint perturbation theorem,
but now it is evaluated along the noisy trajectory.  The following deterministic
shadowing lemma is the trajectory analogue of the endpoint bound.  Because
trajectory points are not critical points, \(\rho_{k,t}^-\) should be read as a
pointwise lower-margin estimate for the local linearized rank-\(k\) map, not as
an automatic theorem about every GD step.

To make the bridge explicit, let
\[
        \Psi_t^0(M)=\Rcal_M\!\left(-\eta\,\grad_{\Mcal_k}F_0(M)\right)
\]
be the noiseless rank-\(k\) manifold step in a local chart, with any smooth
first-order retraction \(\Rcal\).  A sufficient check for the contraction used
below is for every
\(\eta_M\in T_{M_t^w}\Mcal_k\), 

\begin{equation}
\label{eq:trajectory-local-hessian-check}
\begin{aligned}
    \ip{\eta_M}{
        \Hess_{\Mcal_k}F_0(M_t^w)[\eta_M]
    }
    &\ge
    \mu_t^\rho\norm{\eta_M}_F^2,
    \\
    \eta L_t
    &\le
    1 .
\end{aligned}
\end{equation}
where \(L_t\) is a local smoothness bound for the charted gradient map.  For the
local quadratic model, and for the charted map after absorbing higher-order
retraction terms into the local check, this gives
\(\norm{D\Psi_t^0}\le 1-\eta\mu_t^\rho\).  Thus the Hessian Rayleigh quotient is
the actual trajectory certificate.  The ratio \eqref{eq:trajectory-rho-k} is the
rank-\(k\) normal-block diagnostic that explains when the check should become
favorable; at endpoint-critical or tangent-critical points it reduces to the
same normal-curvature margin used in Proposition~\ref{prop:active-margin}.

\begin{lemma}[Full-rank trajectory shadowing controlled by \(\rho\)]
\label{lem:trajectory-shadowing}
Fix a time horizon \(T\).  Suppose that the paired clean and noisy trajectories
remain in a common full-rank \(k\)-dimensional chart for \(0\le t\le T\).  Assume
that the local Hessian/smoothness check \eqref{eq:trajectory-local-hessian-check}
holds, so the noiseless chart map satisfies
\[
        \norm{D\Psi_t^0}\le 1-\eta\mu_t^\rho .
\]
Then the noisy and noiseless matrix iterates obey
\begin{equation}
    \label{eq:trajectory-shadow-recursion}
        \Delta_{t+1}
        \le
        (1-\eta\mu_t^\rho)\Delta_t
        +
        \eta\sqrt{1+\delta}\norm{w}_2 ,
\end{equation}
and therefore
\begin{equation}
    \label{eq:trajectory-shadow-bound}
        \Delta_t
        \le
        \sqrt{1+\delta}\norm{w}_2
        \sum_{j=0}^{t-1}
        \eta
        \prod_{\ell=j+1}^{t-1}
        (1-\eta\mu_\ell^\rho).
\end{equation}
If \(\mu_t^\rho\ge\underline\mu>0\) on the whole window and
\(\eta\underline\mu\le1\), then
\begin{equation}
    \label{eq:trajectory-uniform-shadow-bound}
        \Delta_t
        \le
        \frac{\sqrt{1+\delta}}{\underline\mu}
        \left(1-(1-\eta\underline\mu)^t\right)\norm{w}_2
        \le
        \frac{\sqrt{1+\delta}}{\underline\mu}\norm{w}_2 .
\end{equation}
\end{lemma}

\begin{proof}
The only first-order difference between the two local rank-\(k\) maps at the
same matrix is the tangent projection of \(N_w=\Aop^\ast w\).  RIP gives
\[
        \norm{\Proj_{T_{M_t^w}\Mcal_k}N_w}_F
        \le
        \sqrt{1+\delta}\norm{w}_2,
\]
because tangent matrices have rank at most \(2k\).  By assumption, the computed
rank-\(k\) margin \(\mu_t^\rho=(1-\delta)-\rho_{k,t}^-\) certifies the
one-step tangent contraction.  Combining this contraction of the clean map with
the additive noise perturbation gives
\eqref{eq:trajectory-shadow-recursion}.  Iterating the scalar recursion gives
\eqref{eq:trajectory-shadow-bound}, and the geometric sum gives
\eqref{eq:trajectory-uniform-shadow-bound}.
\end{proof}

Lemma~\ref{lem:trajectory-shadowing} is the clean trajectory philosophy.  The
object \(M_t^0\) is the clean iterate at the same time \(t\), not the clean
critical endpoint used earlier.  The time-varying stability margin is
\(\mu_t^\rho=(1-\delta)-\rho_{k,t}^-\), where the ratio \(\rho_{k,t}^-\)
contains the active floor \(\lambda_k(M_t^w)\) in its denominator.  Early times
with a poor rank-\(k\) margin contribute large terms to the sum in
\eqref{eq:trajectory-shadow-bound}; times with a strong margin damp all
previous noise.

\subsection{St\"oger--Soltanolkotabi Variables and a \(\rho\)-Proxy}

The remaining question is how the dynamics make the \(\rho\)-margin better or
worse.  The clean analysis of \citet{stoegerSoltanolkotabi2021smallInit}
introduces a signal/slack split that is useful for interpreting this.  Their
``noise'' term is an internal over-parameterized slack factor, not observation
noise \(w\).  Let
\[
\begin{alignedat}{2}
    M_\star
    &=
    V_\star\Lambda_\star V_\star^\top,
    &\qquad
    P_\star
    &=
    V_\star V_\star^\top,
    \\
    \lambda_\star
    &=
    \lambda_{r_\star}(M_\star),
    &\qquad
    \kappa_X
    &=
    \frac{\norm{X_\star}_{\op}}
         {\smin(X_\star)} .
\end{alignedat}
\]
For a clean iterate \(X_t^0\), take the compact SVD
\begin{equation}
    \label{eq:stoeger-svd}
        V_\star^\top X_t^0=\widetilde V_t\Sigma_tW_t^\top,
        \qquad W_t\in\RR^{k\times r_\star},\qquad W_t^\top W_t=I,
\end{equation}
and complete \(W_t\) to \([W_t\ W_{t,\perp}]\).  Following their Section 5.1,
write
\begin{equation}
\label{eq:stoeger-signal-noise}
\begin{aligned}
    S_t
    &=
    X_t^0W_t,
    \qquad
    N_t
    =
    X_t^0W_{t,\perp},
    \\
    X_t^0
    &=
    S_tW_t^\top
    +
    N_tW_{t,\perp}^\top .
\end{aligned}
\end{equation}
By \citet[Lemma~5.1]{stoegerSoltanolkotabi2021smallInit}, the slack factor has
column span orthogonal to \(V_\star\), and, whenever \(V_\star^\top X_t^0\) has
full row rank, \(S_t\) has rank \(r_\star\) while \(N_t\) has rank at most
\(k-r_\star\).  Let \(V_{S_t}\) denote an orthonormal basis for
\(\operatorname{col}(S_t)\).  Define
\begin{equation}
\label{eq:stoeger-variables}
\begin{aligned}
s_t
&=\smin(S_t),
\quad
\nu_t=\norm{N_t}_{\op},
\quad
\phi_t=\frac{s_t^2}{\lambda_\star},
\\
\xi_t
&=\norm{V_\star^\perp V_{S_t}}_{\op},
\quad
\chi_t=\frac{\nu_t^2}{s_t^2},
\\
\varepsilon_t
&=
\frac{
    \norm{V_\star^\top(M_t^0-M_\star)}_F
}{
    \lambda_\star
}.
\end{aligned}
\end{equation}
The angle \(\xi_t\) is the clean signal-factor angle; in
\citet{stoegerSoltanolkotabi2021smallInit}, Lemma~9.1 grows the projected
signal singular value, Lemma~9.2 controls the internal slack factor, and
Lemma~9.3 gives the direct angle recurrence.  The auxiliary
\(\varepsilon_t\) is used below only to quote their Phase~III refinement; it is
not a separate term in the four-metric dashboard.  With sensing distortion one also
uses the top \(r_\star\)-dimensional spectral space \(P_H\) of \(H(M_\star)\)
and the upper proxy
\[
        \theta_t=\norm{V_{S_t}^\perp P_H}_{\op}
        +
        \norm{P_H^\perp P_\star}_{\op},
        \qquad
        \xi_t\le\theta_t,
\]
where the second term is controlled by Davis--Kahan when
\(\norm{H(M_\star)-M_\star}_{\op}\ll\lambda_\star\).

The strict full-rank robustness ratio uses the smallest rank-\(k\) eigenvalue of
the noisy iterate.  Normalize this floor by the signal scale:
\[
        \psi_t^w=\frac{\lambda_k(M_t^w)}{\lambda_\star}.
\]
The normal block is easiest to read through one subspace metric.  For the
mathematical bound we use the target alignment
\begin{equation}
    \label{eq:trajectory-misalignment}
        m_t^w=\norm{P_t^{w,\perp}P_\star}_{\op}.
\end{equation}
This measures how much of the true signal space is still visible from the
noisy normal space.  It is small when the noisy rank-\(k\) range contains the
target signal space.  If \(P_\star\) is unavailable and the spectral surrogate
\(P_H\) is used instead, then
\[
        m_t^w
        \le
        \norm{P_t^{w,\perp}P_H}_{\op}
        +
        \norm{P_H^\perp P_\star}_{\op},
\]
and the second term is controlled by Davis--Kahan under the spectral closeness
condition stated above.
The normalized external noise level is
\begin{equation}
    \label{eq:trajectory-noise-level}
        \eta_w=\frac{\sqrt{1+\delta}\norm{w}_2}{\lambda_\star}.
\end{equation}
On a bootstrap window where Lemma~\ref{lem:trajectory-shadowing} keeps
\(\Delta_t=O(\norm{w}_2)\), the clean/noisy path discrepancy is folded into the
same relaxed noise scale.  The signal-floor/slack part is represented by the
measured floor \(\psi_t^w\) and the slack tail
\begin{equation}
    \label{eq:trajectory-slack-tail}
        \tau_t
        =
        \sqrt{k-r_\star}\,\chi_t\phi_t
        =
        \sqrt{k-r_\star}\,\frac{\nu_t^2}{\lambda_\star}.
\end{equation}
Thus \(\tau_t\) is not another signal-growth variable; it is the Frobenius-sized
energy of the \(k-r_\star\) internal slack directions, normalized by the target
spectral floor.

The following lemma is the promised mathematical bound.  It has two fixed
constants, not two new trajectory metrics.  The constant \(c_{\rm sh}\) says
that the noisy and clean paths are already in the small-noise shadowing window,
and \(c_{\rm sig}\) is the clean signal-block radius.  For example,
\(c_{\rm sig}=5\sqrt{r_\star}\kappa_X^2\) is valid whenever
\(\norm{S_t}_{\op}\le2\norm{X_\star}_{\op}\), because then
\[
        \norm{S_tS_t^\top-M_\star}_F
        \le
        5\sqrt{r_\star}\,\kappa_X^2\lambda_\star .
\]

\begin{lemma}[Four-metric upper bound for the trajectory \(\rho\)]
\label{lem:trajectory-rho-four-metric}
Assume \(k>r_\star\).  Fix a full-rank time \(t\) in the trajectory window.
Suppose
\[
\begin{aligned}
\Delta_t
&=\norm{M_t^w-M_t^0}_F
\le c_{\rm sh}\norm{w}_2,
\\
\norm{S_tS_t^\top-M_\star}_F
&\le c_{\rm sig}\lambda_\star .
\end{aligned}
\]
Then
\begin{equation}
\label{eq:trajectory-rho-hard-bound}
\begin{aligned}
\rho_{k,t}^-
&\le
\frac{
    \kappa_X^2(m_t^w)^2
    +(1+\delta)\tau_t
    +\delta c_{\rm sig}
}{
    \psi_t^w
}
\\
&\quad+
\frac{
    \left(1+c_{\rm sh}\sqrt{1+\delta}\right)\eta_w
}{
    \psi_t^w
}.
\end{aligned}
\end{equation}
Equivalently, the margin in Lemma~\ref{lem:trajectory-shadowing} is positive
whenever the numerator in \eqref{eq:trajectory-rho-hard-bound} is smaller than
\((1-\delta)\psi_t^w\).  If the surrogate \(P_H\) is used, replace
\((m_t^w)^2\) by
\[
        \left(
        \norm{P_t^{w,\perp}P_H}_{\op}
        +
        \norm{P_H^\perp P_\star}_{\op}
        \right)^2 .
\]
\end{lemma}

\begin{proof}
Let \(P=P_t^w\) and
\[
        B=P^\perp G_w(M_t^w)P^\perp .
\]
Since
\(\norm{[B]_-}_{\op}
=\sup_{\norm{v}=1,\ v\in\operatorname{range}(P^\perp)}
[-v^\top Bv]_+\), it is enough to bound the negative Rayleigh quotient on a
unit normal vector \(v\).  Using
\(M_t^0=S_tS_t^\top+N_tN_t^\top\),
\[
\begin{aligned}
-v^\top G_w(M_t^w)v
&=
-v^\top H(M_t^w-M_t^0)v
\\
&\quad
-v^\top H(S_tS_t^\top-M_\star)v
\\
&\quad
-v^\top H(N_tN_t^\top)v
\\
&\quad
+v^\top N_wv .
\end{aligned}
\]
The shadowing and observation-noise terms satisfy
\[
\begin{aligned}
\abs{v^\top H(M_t^w-M_t^0)v}
&\le (1+\delta)\Delta_t \\
&\le c_{\rm sh}\sqrt{1+\delta}\,
      \lambda_\star\eta_w,
\\
\abs{v^\top N_wv}
&\le \lambda_\star\eta_w .
\end{aligned}
\]
For the active signal block, restricted polarization gives
\[
\begin{aligned}
v^\top H(S_tS_t^\top-M_\star)v
&\ge v^\top(S_tS_t^\top-M_\star)v \\
&\quad-\delta\norm{S_tS_t^\top-M_\star}_F .
\end{aligned}
\]
This polarization step uses only the unified \(2k\)-RIP, because
\(\rank(S_tS_t^\top-M_\star)\le2r_\star\) and the trajectory case has
\(k>r_\star\), so \(2r_\star+1\le2k\).
Therefore
\[
\begin{aligned}
\left[
-v^\top H(S_tS_t^\top-M_\star)v
\right]_+
&\le
\left[
v^\top(M_\star-S_tS_t^\top)v
\right]_+
\\
&\quad+
\delta c_{\rm sig}\lambda_\star
\\
&\le
v^\top M_\star v
+\delta c_{\rm sig}\lambda_\star
\\
&\le
\lambda_1(M_\star)\norm{P_\star v}^2
\\
&\quad+
\delta c_{\rm sig}\lambda_\star
\\
&\le
\lambda_\star\kappa_X^2(m_t^w)^2
\\
&\quad+
\delta c_{\rm sig}\lambda_\star .
\end{aligned}
\]
Finally,
\[
\begin{aligned}
\abs{v^\top H(N_tN_t^\top)v}
&\le (1+\delta)\norm{N_tN_t^\top}_F \\
&\le (1+\delta)\sqrt{k-r_\star}\nu_t^2 \\
&= (1+\delta)\tau_t\lambda_\star .
\end{aligned}
\]
Combining these four estimates and dividing by
\(\lambda_k(M_t^w)=\psi_t^w\lambda_\star\) gives
\eqref{eq:trajectory-rho-hard-bound}.
\end{proof}

Combining Lemmas~\ref{lem:trajectory-shadowing}
and~\ref{lem:trajectory-rho-four-metric} gives the explicit trajectory
perturbation bound.  Define the four-metric upper envelope
\begin{equation}
\label{eq:trajectory-rho-envelope}
\begin{aligned}
\bar\rho_{k,t}
&=
\frac{\kappa_X^2(m_t^w)^2+(1+\delta)\tau_t
      +\delta c_{\rm sig}}{\psi_t^w}
\\
&\quad+
\frac{\left(1+c_{\rm sh}\sqrt{1+\delta}\right)\eta_w}
     {\psi_t^w},
\\
\bar\mu_t
&=(1-\delta)-\bar\rho_{k,t}.
\end{aligned}
\end{equation}
On any time window where the hypotheses of
Lemma~\ref{lem:trajectory-rho-four-metric} hold at each time and the local
Hessian check in \eqref{eq:trajectory-local-hessian-check} holds, if
\(\bar\mu_t>0\), then \(\mu_t^\rho\ge\bar\mu_t\) and
\begin{equation}
    \label{eq:trajectory-four-metric-shadow-bound}
        \Delta_t
        \le
        \sqrt{1+\delta}\norm{w}_2
        \sum_{j=0}^{t-1}
        \eta
        \prod_{\ell=j+1}^{t-1}
        (1-\eta\bar\mu_\ell).
\end{equation}
In particular, if \(\bar\mu_t\ge\underline\mu>0\) on the window, then
\begin{equation}
    \label{eq:trajectory-four-metric-uniform-bound}
        \norm{M_t^w-M_t^0}_F
        \le
        \frac{\sqrt{1+\delta}}{\underline\mu}\norm{w}_2 .
\end{equation}
This is the actual robustness statement in the trajectory section: alignment,
slack, noise, and the active floor determine \(\bar\rho_{k,t}\), and the gap
\(\bar\mu_t=(1-\delta)-\bar\rho_{k,t}\) determines the perturbation size.
The bootstrap form is non-circular.  Fix a desired uniform margin
\(\underline\mu>0\) and set
\[
        c_{\rm sh}=\frac{\sqrt{1+\delta}}{\underline\mu}.
\]
If, with this predetermined value of \(c_{\rm sh}\), the envelope
\eqref{eq:trajectory-rho-envelope} satisfies
\[
        \bar\rho_{k,t}\le 1-\delta-\underline\mu
        \qquad\text{for every }0\le t\le T,
\]
then \eqref{eq:trajectory-four-metric-uniform-bound} gives
\[
        \Delta_t\le c_{\rm sh}\norm{w}_2
        \qquad\text{for every }0\le t\le T.
\]
Thus the shadowing condition used inside
Lemma~\ref{lem:trajectory-rho-four-metric} is proved by induction on the same
window, rather than assumed after the fact.

The exact ratio \(\rho_{k,t}^-\) remains directly computable from
\eqref{eq:trajectory-rho-k}.  Lemma~\ref{lem:trajectory-rho-four-metric} gives
the relaxed four-metric upper bound
\begin{equation}
    \label{eq:trajectory-rho-proxy}
    \widehat\rho_{k,t}^{\rm traj}
    :=
    \frac{
        \kappa_X^2(m_t^w)^2
        +
        \tau_t
        +
        \eta_w
    }{\psi_t^w}.
\end{equation}
This is the readable version of \eqref{eq:trajectory-rho-hard-bound} after
folding the fixed RIP-bias term \(\delta c_{\rm sig}\) and the fixed shadowing
constant \(c_{\rm sh}\) into constants.  The measured certificate remains
\(\rho_{k,t}^-\), but now the section has an explicit upper bound:
\[
\begin{aligned}
\rho_{k,t}^-
&\le
\frac{\kappa_X^2(m_t^w)^2+(1+\delta)\tau_t
      +\delta c_{\rm sig}}{\psi_t^w}
\\
&\quad+
\frac{\left(1+c_{\rm sh}\sqrt{1+\delta}\right)\eta_w}
     {\psi_t^w}
\\
&\approx \widehat\rho_{k,t}^{\rm traj}.
\end{aligned}
\]
The denominator \(\psi_t^w\) is the active full-rank floor.  The middle term is
the slack tail created by the extra \(k-r_\star\) columns.  Here ``active''
means active in the rank-\(k\) chart, not necessarily signal-aligned with
\(M_\star\); when \(k>r_\star\), the floor may be carried by the thinnest
retained slack direction.  The remaining noise term \(\eta_w\) absorbs
the small perturbation effects we are deliberately not separating in this
relaxed dashboard.  Even excellent signal variables do not certify strict
full-rank robustness if \(\lambda_k(M_t^w)\) is still microscopic.

To see \eqref{eq:trajectory-rho-proxy}, set \(P=P_t^w\).  Then
\begin{equation}
\label{eq:trajectory-normal-decomposition}
\begin{aligned}
P^\perp G_w(M_t^w)P^\perp
&=
P^\perp H(M_t^w-M_t^0)P^\perp
\\
&\quad+
P^\perp H(M_t^0-M_\star)P^\perp
\\
&\quad-
P^\perp N_wP^\perp .
\end{aligned}
\end{equation}
For a unit vector \(v\in\operatorname{range}(P_t^{w,\perp})\), Cauchy--Schwarz
and RIP give \(\abs{v^\top N_wv}\lesssim \lambda_\star\eta_w\).  The clean
residual is controlled by the single misalignment metric \(m_t^w\): larger
\(m_t^w\) means the noisy normal space still sees signal directions and can
therefore contain more negative normal curvature.  The slack block has
Frobenius norm at most
\[
        \sqrt{k-r_\star}\nu_t^2
        =
        \tau_t\lambda_\star .
\]
Dividing the corresponding rigorous estimate by
\(\lambda_k(M_t^w)=\psi_t^w\lambda_\star\) is precisely
Lemma~\ref{lem:trajectory-rho-four-metric}.  The simplified display
\eqref{eq:trajectory-rho-proxy} only suppresses the fixed constants from
\eqref{eq:trajectory-rho-hard-bound}; the exact robustness certificate is still
the measured \(\rho_{k,t}^-\) together with the Hessian check
\eqref{eq:trajectory-local-hessian-check}.

\subsection{Clean Phase Dynamics and the Evolution of the \(\rho\)-Proxy}

The robustness quantity in this section is always the strict full-rank margin
\(\rho_{k,t}^-\) from \eqref{eq:trajectory-rho-k}.  The proxy
\eqref{eq:trajectory-rho-proxy} only explains its moving parts.  Once the
measured ratio gives a positive margin and the local Hessian check passes,
Lemma~\ref{lem:trajectory-shadowing} converts that margin into the shadowing
bound \eqref{eq:trajectory-shadow-bound}.  The phase story below is therefore a
story about the dashboard in \eqref{eq:trajectory-rho-proxy}: how the single
misalignment metric \(m_t^w\), the slack tail \(\tau_t\), the noise level
\(\eta_w\), and the measured rank-\(k\) floor \(\psi_t^w\) evolve.

\paragraph{Phase I: spectral alignment.}
The first phase is controlled by the power-method comparison in
\citet[Lemmas~8.1--8.2]{stoegerSoltanolkotabi2021smallInit}.  They compare GD
to the linearized iterate
\[
        \widetilde X_t=(I+\eta H(M_\star))^tX_0
\]
and show that \(X_t^0\) remains close to \(\widetilde X_t\) while the iterate is
small.  Their Lemma~8.3 gives SVD perturbation bounds for the approximate power
iterate \(Z_tX_0+E_t\); Lemma~8.4 translates top-left-singular-space alignment
into properties of \(X_t^0W_t\) and \(X_t^0W_{t,\perp}\); Lemma~8.5 combines
these two steps, under spectral closeness
\(H(M_\star)=M_\star+\widetilde E\), small
\(\norm{\widetilde E}_{\op}/\lambda_\star\), and the corresponding small
power-method error parameter, into lower bounds on \(\smin(X_t^0W_t)\), upper
bounds on \(\norm{X_t^0W_{t,\perp}}\), and angle bounds for
\(\norm{V_\star^\perp V_{S_t}}_{\op}\).  Their Theorem~8.8 supplies the
Gaussian smallest-singular-value lower bound for an \(r_\star\times k\)
Gaussian matrix with entries normalized by \(1/\sqrt{k}\):
\[
        \smin(G)
        \ge
        \varepsilon\,\frac{\sqrt{k}-\sqrt{r_\star-1}}{\sqrt{k}}
\]
with probability at least
\(1-(C\varepsilon)^{k-r_\star+1}-e^{-ck}\).  This is the probabilistic source of
improved initial overlap under over-parameterization.

In the dashboard \eqref{eq:trajectory-rho-proxy}, Phase I primarily reduces the
clean component of \(m_t^w\).  More precisely, Lemma~8.5 controls the clean
subspace angle through \(\xi_t\) after the approximate power-method phase,
while the small-noise/bootstrap condition keeps the noisy component at the
\(\eta_w\) scale.  In this relaxed view, Phase I is best viewed as subspace
preparation: the trajectory has learned where the signal is, but the full-rank
chart may still be fragile.

\paragraph{Phase II: saddle avoidance and signal-floor growth.}
After alignment, the clean trajectory obeys the signal-floor growth inequality
\citep[Lemma~9.1]{stoegerSoltanolkotabi2021smallInit}
\[
\begin{aligned}
\smin(V_\star^\top X_{t+1}^0)
&\ge
\left(1+\frac14\eta\lambda_\star\right)
\smin(V_\star^\top X_t^0)
\\
&\quad-
\eta\,\smin(V_\star^\top X_t^0)^3 .
\end{aligned}
\]
under their step-size, angle, norm, and RIP-error conditions.  While this
projected signal floor is below signal scale, and while
\(\xi_t=\norm{V_\star^\perp V_{S_t}}_{\op}\) remains small, it is comparable to
\(\smin(S_t)\) up to the principal-angle factor
\(\sqrt{1-\xi_t^2}\).  Hence a useful reading is
\begin{equation}
    \label{eq:phi-growth-rule}
        s_{t+1}\gtrsim s_t(1+c\eta\lambda_\star),
        \qquad
        \phi_{t+1}\gtrsim \phi_t(1+c\eta\lambda_\star)^2 .
\end{equation}
The same analysis controls the slack factor
\(\nu_t=\norm{N_t}_{\op}\) and the signal angle
\citep[Lemmas~9.2--9.3]{stoegerSoltanolkotabi2021smallInit}.
In the proof of Theorem~9.6, starting from a post-spectral time \(t_0\) with
initial signal scale \(\gamma\), Phase II is summarized by inductive bounds of
the form
\[
\begin{aligned}
\smin(V_\star^\top X_t^0)
&\gtrsim
(1+c\eta\lambda_\star)^{t-t_0}\gamma,
\\
\nu_t
&\lesssim
(1+c_{\rm slack}\eta\lambda_\star)^{t-t_0}\gamma,
\\
\xi_t
&\lesssim
\kappa_X^{-2}.
\end{aligned}
\]
Therefore the clean signal floor grows while the slack ratio
\(\chi_t=\nu_t^2/s_t^2\) and the angle remain controlled.  In the dashboard
\eqref{eq:trajectory-rho-proxy}, this is the favorable middle regime:
\(m_t^w\) is small because the surrogate signal space is aligned, while
\(\tau_t=\sqrt{k-r_\star}\chi_t\phi_t\) is harmless only if the product
\(\chi_t\phi_t\) remains small relative to \(\psi_t^w\).  The noise level
\(\eta_w\) is the fixed floor.  The measured
denominator \(\psi_t^w\) can become meaningful if the smallest rank-\(k\)
direction is no longer microscopic.  The caveat is the slack tail
\[
        (1+\delta)\tau_t
        =
        (1+\delta)\sqrt{k-r_\star}\,\frac{\nu_t^2}{\lambda_\star}.
\]
Lemma~9.2 controls slack relative to the signal factor, but the proxy cares
about \(\tau_t/\psi_t^w\).  Thus slack stays harmless only when its normalized
energy remains small relative to the measured rank-\(k\) floor.  For strict
full-rank robustness, the decisive test is still whether the
computed
\[
        \rho_{k,t}^-=
        \frac{\norm{[P_t^{w,\perp}G_w(M_t^w)P_t^{w,\perp}]_-}_{\op}}
        {\lambda_k(M_t^w)}
\]
falls below \(1-\delta\).  Phase II is the period in which this can happen:
the numerator becomes more signal-aligned, while the denominator becomes
acceptable if the smallest retained rank-\(k\) direction, possibly a slack
direction relative to \(M_\star\), has grown enough to make the rank-\(k\)
chart well conditioned.

\paragraph{Phase III: local refinement and slack risk.}
Once \(s_t\) reaches signal scale, the active reconstruction error contracts
\citep[Lemma~9.5]{stoegerSoltanolkotabi2021smallInit}:
\[
\begin{aligned}
\norm{V_\star^\top(M_\star-M_{t+1}^0)}_F
&\le
(1-c\eta\lambda_\star)
\norm{V_\star^\top(M_\star-M_t^0)}_F
\\
&\quad+
C\eta\lambda_\star\norm{N_tN_t^\top}_F .
\end{aligned}
\]
Their Theorem~9.6 then combines Phase II and Phase III.  Its post-spectral
hypotheses include \(\min\{k,n\}>r_\star\), bounded iterate norm, small slack,
alignment
\(\norm{V_\star^\perp V_{S_{t_0}}}_{\op}\lesssim\kappa_X^{-2}\), an initial
signal scale
\[
        \gamma\le
        \frac{c\,\smin(X_\star)}{\min\{k,n\}\kappa_X^2},
\]
and a rank-\((2r_\star+1)\) RIP constant of order at most
\(1/(\kappa_X^4\sqrt{r_\star})\).  Under these hypotheses, the additional
iteration count is
\[
        O\!\left(
        \frac{1}{\eta\lambda_\star}
        \log\!\left(
        \max\left\{1,\frac{\kappa_X r_\star}{\min\{k,n\}-r_\star}\right\}
        \frac{\norm{X_\star}}{\gamma}
        \right)
        \right)
\]
and the clean error is bounded by a polynomial factor times
\((\gamma/\norm{X_\star})^{21/16}\).  In the dashboard
\eqref{eq:trajectory-rho-proxy}, Phase III no longer introduces a separate
active-error metric; the four-metric view keeps \(\eta_w\) as the fixed
external noise floor and tracks the remaining risk through \(\tau_t\).  This is
good for clean active reconstruction, but it is not by itself a strict
full-rank robustness certificate.  The slack tail can still dominate if the
extra columns fail to decay relative to \(\psi_t^w\).  Then
\(\tau_t/\psi_t^w\) can flatten or increase, and the strict full-rank
\(\rho\)-margin stops improving.

The trajectory-aware stopping rule is therefore simple: continue while the
measured \(\rho_{k,t}^-\) decreases or the shadowing bound
\eqref{eq:trajectory-shadow-bound} tightens; stop, damp, or truncate once the
full-rank \(\rho\)-margin stops improving.

\fi

\end{document}